\documentclass[10pt,twocolumn]{article}

\usepackage[letterpaper,margin=0.75in]{geometry}
\usepackage{times}
\usepackage[utf8]{inputenc}
\usepackage[T1]{fontenc}
\usepackage[hidelinks,hypertexnames=false]{hyperref}
\usepackage{xurl}
\usepackage{booktabs}
\usepackage{amsfonts}
\usepackage{amsmath}
\usepackage{amsthm}
\usepackage{algorithm}
\usepackage{algpseudocode}
\usepackage{enumitem}
\newtheorem{proposition}{Proposition}
\newtheorem{corollary}{Corollary}
\usepackage{nicefrac}
\usepackage{microtype}
\usepackage{graphicx}
\usepackage{placeins}
\usepackage{xcolor}
\usepackage{natbib}

\newcommand{\figonewidth}{\textwidth}
\newcommand{\figfourwidth}{\textwidth}

\hypersetup{
  pdftitle={The Premise Is the Problem: Exchangeability Failure in Self-Monitored Test-Time Adaptation},
  pdfauthor={Weijia Han, Lisha Qu, Zhenda Li, and Liying Liang},
  pdfsubject={Anytime-valid gates can lose their guarantee when self-monitored adaptation breaks exchangeability.},
  pdfkeywords={test-time adaptation, online time-series forecasting, anytime-valid inference, conformal martingales, exchangeability, change detection, feedback loops, robust forecasting}
}

\title{The Premise Is the Problem:\\Exchangeability Failure in Self-Monitored Test-Time Adaptation}
\author{
  \begin{tabular}{cc}
    \begin{tabular}[t]{c}
      Weijia Han\\
      University of Washington\\
      Seattle, USA\\
      \texttt{wh73@uw.edu}
    \end{tabular}
    &
    \begin{tabular}[t]{c}
      Lisha Qu\\
      University of Washington\\
      Seattle, USA\\
      \texttt{lishaq@uw.edu}
    \end{tabular}
    \\[1.25em]
    \begin{tabular}[t]{c}
      Zhenda Li\\
      Independent Researcher
    \end{tabular}
    &
    \begin{tabular}[t]{c}
      Liying Liang\\
      Independent Researcher
    \end{tabular}
  \end{tabular}
}
\date{}

\begin{document}

\maketitle

\begin{abstract}
Modern forecasting models are often updated after deployment so they can respond to changing data.
These updates can also make predictions worse, so practical systems need a reliable monitor that
can detect harmful changes and trigger protection. A natural design is to monitor the same
prediction errors that guide the updates. This paper asks whether the statistical guarantee behind
such a monitor remains valid when monitoring and adaptation use the same feedback. We study this
question in multi-step time-series forecasting. We show that overlapping targets and dependence in
forecast errors can break a key assumption required by the guarantee. The monitor may then raise
alarms even when no harmful change has occurred, and its response can further damage prediction
quality. We also find that adaptation can hide sustained changes from its own monitor, while the
original frozen model retains a clearer signal. These results expose a basic failure mode in
self-monitored adaptation. They show why reliable deployment requires checking the monitor's
assumptions, comparing adaptation with the frozen model under realistic feedback, and limiting the
effect of every protective response.
 \end{abstract}

\section{Introduction}
\label{sec:intro}

\begin{figure*}[!t]
\centering
\includegraphics[width=\figonewidth]{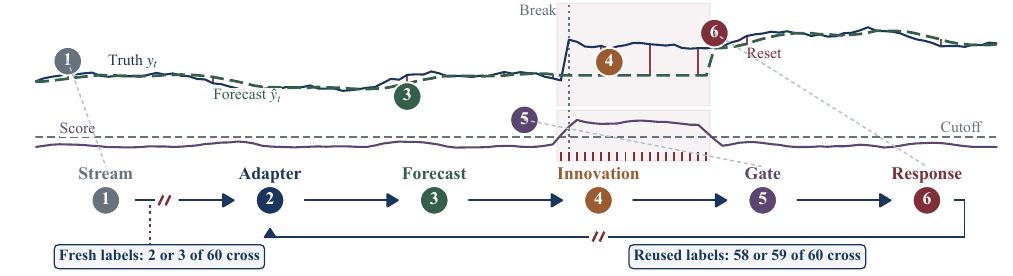}
\caption{Self-monitored adaptation and the label-reuse mechanism; the stream is illustrative and the callouts
summarize Section~\ref{sec:e6}. Consecutive 96-step labels share 95 of their 96 entries, so the oracle
score is a moving sum. On a conditional null with the scorer held at the truth, reuse changes the
classical no-reset wealth process from 2 or 3 of 60 crossings to 58 or 59 of 60; the Ville guarantee
applies only when fresh labels are used. Appendix~\ref{app:spec} specifies every component.}
\label{fig:loop}
\end{figure*}

Time-series foundation models forecast many steps ahead and are deployed frozen
\citep{auer2025tirex,ansari2025chronos2,das2024decoderonly}. Test-time adaptation corrects their
forecasts from newly revealed labels \citep{lee2025elf,dai2026orca}, but learning from a corrupted
or shifting stream can cause harm. Such an adaptive system therefore needs a monitor for protective action.

Conformal test martingales offer an appealing guarantee for such a monitor. On an exchangeable score
sequence the wealth crosses $\nicefrac{1}{\alpha}$ with probability at most $\alpha$, uniformly in time. The adapter's own scores are a readily available monitoring signal, so the natural design lets one score
stream both drive the updates and police them (Figure~\ref{fig:loop}). We call this
self-monitored adaptation. This paper asks whether the guarantee survives that feedback loop.
The question matters because the bounded probability of any false alarm is the reason to prefer
such a gate over a heuristic threshold. A false alarm can trigger a response that damages the
adapter.

We study a block Kalman adapter applied to three frozen backbones: TiRex, Chronos-2, and
TimesFM~2.5. We evaluate it on five public streams: ETTh2, ETTm1, weather, electricity, and
traffic~\citep{zhou2021informer,wu2021autoformer}. The main campaign uses immediate and delayed
reveal. Targeted analyses use composite coordinate and stride-96 feedback. We count fires on clean replay and
under injected contamination, separate the elements of the loop with synthetic nulls, replay the
classical growing-calibration, no-reset wealth process on saved scores, build a conditional null that changes only the
reuse of labels across overlapping horizons, and evaluate the adapter against the frozen backbone
under deployable feedback. Because the deployed gate's sliding calibration window and wealth
resets make it heuristic, the classical replay is its formal audit. We recorded a decision rule before running each campaign. Later
response, fallback, and frozen-residual studies were exploratory follow-ups and are labeled as
such. We report every outcome, including a failed first design, and connect a classical dependence
pattern to the loss of validity that occurs when an exact sequential test is applied outside its
exchangeability null, and to damage from the resulting response.

We make four contributions.
\begin{itemize}[leftmargin=*,itemsep=1pt,topsep=2pt,parsep=0pt]
\item \textbf{Failure of the statistical premise.} We prove that reusing overlapping forecast labels can make even a
correctly specified score sequence non-exchangeable, so the classical anytime-valid guarantee does not apply.
In a controlled null, reuse raises crossings from 2 or 3 of 60 to 58 or 59 of 60; spacing origins one
horizon apart restores validity for the truth-held scorer (Section~\ref{sec:premise}; Appendix~\ref{app:valid}).
\item \textbf{Evidence on real streams.} Frozen-backbone residuals also cross without a learner, so adaptation is not
necessary for the observed alarms. Using disjoint forecast origins greatly reduces crossings, but this
real-stream comparison is descriptive and does not establish a Ville bound (Section~\ref{sec:e6};
Appendix~\ref{app:valid}).
\item \textbf{Value and limits of adaptation.} The adapter absorbs a sustained change within one horizon, while the
frozen residual remains elevated and can support later detection. Under two deployable feedback protocols,
however, the adapter loses to the frozen backbone in 29 of 30 backbone--stream evaluations
(Section~\ref{sec:cutoff}).
\item \textbf{Consequences of intervention.} Detecting a change and responding safely are separate problems. The tested
responses can reproduce divergence, act too often on clean data, or fail under sustained contamination;
removing reset and re-entry helps only the selected cohort (Sections~\ref{sec:attribution}
and~\ref{sec:cutoff}).
\end{itemize}

The conditional-null result is theoretical; the findings on real streams, learned scores, adapter value,
and response design are empirical. Together, they show that a monitor cannot be treated as a plug-in
safeguard: both its statistical assumptions and the actions it triggers must be validated in the full
adaptive system.
 \section{Related work}
\label{sec:related}

\subsection{Anytime-valid testing and conformal martingales}
Ville's inequality \citep{ville1939etude} bounds the probability that a nonnegative martingale
started at one ever crosses a fixed threshold, the basis of anytime-valid testing.
\citet{vovk2003testing} test exchangeability online with conformal martingales,
\citet{fedorova2012plugin} and \citet{volkhonskiy2017inductive} give the plug-in and inductive
variants, and \citet{vovk2021retrain} reset the wealth on a fire for change-point detection, the
rule this gate uses with the randomized $p$-values of \citet{vovk2022algorithmic}.
\citet{shafer2021testing} and \citet{ramdas2023game} frame the family as betting-based
statistics valid under optional stopping. Later work extends the betting machinery
\citep{choe2024combining,szabadvary2026legendre}, tests exchangeability directly
\citep{saha2024pairwise,ramdas2022exchangeability}, and develops sequential change detectors
from e-processes \citep{shin2024edetectors,bhattacharyya2026change}. The conformal martingales studied here
assume an exchangeable score sequence; other e-processes can target other nulls. We ask what the
exchangeability premise is worth when the monitored learner produces the scores.

\subsection{Monitoring deployed models under shift}
Anytime-valid tests can track a deployed model's risk \citep{podkopaev2022tracking}, extend to
unlabeled streams through a secondary error model \citep{amoukou2024harmful}, bound running risk
from few true labels \citep{zhang2026pprm}, or reweight against covariate shift
\citep{prinster2025watch}. Appendix~\ref{app:edetector} reports our replay on the saved streams.
Related work surveys heuristic drift detectors \citep{gama2014survey}, gates adaptation on a
detector \citep{zhang2024d3a}, and supplies the whiteness diagnostic used here
\citep{ljungbox1978}. Those monitors watch a fixed model, a risk bound, or a covariate shift. Here
the monitored score stream can inherit dependence from the forecast residuals, and overlap and
adaptation can amplify it.

\subsection{Conformal prediction under dependence}
Conformal validity can tolerate bounded departures from exchangeability
\citep{barber2023beyond}. Other methods adapt coverage under distribution shift
\citep{gibbs2021adaptive,gibbs2024online}, maintain it through change points
\citep{sun2025changepoints}, or accommodate serial dependence
\citep{xu2021enbpi,xu2023spci,oliveira2024split,lee2025kernel}. Those methods repair the forecast's own interval, while the premise at stake here is
exchangeability of the monitor's scores. Proposition~\ref{prop:reuse} restates for the
conditional null's oracle score the overlapping-observations autocorrelation of forecast
evaluation~\citep{hansen1980forward,diebold1995comparing} and adds the non-exchangeability it
implies; what this paper measures is its consequence for a monitor that shares its scores with
the learner.

\subsection{Test-time adaptation of forecasters}
Frozen forecast foundation models
\citep{auer2025tirex,ansari2024chronos,ansari2025chronos2,das2024decoderonly,woo2024moirai}
are deployed without gradient updates, and \citet{meyer2025leakage} document that their
pretraining exposure to public streams is hard to establish. \citet{wang2021tent} is the
canonical test-time adaptation method, and \citet{liang2025ttasurvey} survey the area.
Error accumulation and collapse under continual adaptation motivate reset mechanisms
\citep{wang2022continual,niu2023stable,press2023rdumb,lim2026reset}; the weather divergences here
instantiate that failure. \citet{bar2024protected} instead drive adaptation with a betting
martingale on the adapted model's own statistic, the design whose premise this paper measures.

For online forecasting, \citet{pham2023fsnet} and \citet{wen2023onenet} use immediate full reveal,
a protocol criticized by \citet{lau2025dsof} and \citet{wang2026principled}. Leakage-free reveal
protocols appear in \citet{zhao2025proceed}, \citet{lau2025dsof}, \citet{kim2025tafas},
\citet{lee2025elf}, and \citet{dai2026orca}. Section~\ref{sec:cutoff} studies these settings, and
we evaluate five released refiners on the same cells. Our adapter
is the dynamic linear model of \citet{westharrison1997}. Its fallback is an exponentially weighted
average forecaster~\citep{cesabianchi2006prediction}, with a fixed share in its first
version~\citep{herbster1998tracking}.
 \section{The online learner and the gate under test}
\label{sec:setup}

\textbf{The adapter.} The online learner is a block-diagonal Kalman filter that calibrates a
frozen backbone's 96-step forecast $f_t \in \mathbb{R}^{96}$ of one target channel, issued from a
96-step context, against the label $y_t$. Its state
$\theta_t$ holds a slope and an intercept per horizon, one $2 \times 2$ block each, so the same
filter runs unchanged on all three backbones. Row $h$ of the observation matrix $H_t$ is
$[f_{t,h}\;\;1]$ in its own block. The filter learns one forgetting factor $\lambda_t$, a
192-dimensional vector $Q_t$ of process-noise diagonal entries, and a 96-dimensional vector
$R_t$ of observation-noise diagonal entries. Its prediction and innovation covariance are
\begin{align}
P_t^-&=P_{t-1}/\lambda_t+\operatorname{diag}(Q_t),\nonumber\\
S_t&=H_tP_t^-H_t^\top+\operatorname{diag}(R_t).
\label{eq:main-filter}
\end{align}
$P_t^-$ is block diagonal, and $S_t$ is diagonal. The identity-calibration state $\theta_0$ has unit slopes and
zero intercepts, so $H_t\theta_0=f_t$ reproduces the frozen forecast. All three parameter groups
adapt online through one-step hypergradients~\citep{baydin2018hypergradient} to an innovation
objective~\citep{westharrison1997} after each non-skipped feedback packet. The $\lambda$ update
uses a diagonal approximation and becomes a surrogate update for the learned base parameter while
the hard recovery override is active. The $Q$ and $R$ derivatives are exact for the stated
fixed-$P_{t-1}$, diagonal-noise objective. No adapter or gate hyperparameter in the primary evaluation is tuned per
stream or backbone
(Appendix~\ref{app:compute} reports the arithmetic cost).

\textbf{One-step loop.} With innovation $e_t=y_t-H_t\theta_t^-$ and the current 250-score window
$W_t$, let $n_t^{>}=\#\{s_j>s_t:j\in W_t\}$ and
$n_t^{=}=\#\{s_j=s_t:j\in W_t\}$. The gate computes
\begin{align}
s_t&=e_t^\top S_t^{-1}e_t,\nonumber\\
p_t&=\begin{cases}
\nicefrac{1}{2}, & |W_t|=0,\\[1mm]
\dfrac{n_t^{>}+u_t(1+n_t^{=})}{|W_t|+1}, & |W_t|>0,
\end{cases}\nonumber\\
M_t&=\frac1{19}\sum_{k=1}^{19}\prod_{j=\tau_t}^{t}\varepsilon_kp_j^{\varepsilon_k-1}.
\label{eq:main-gate}
\end{align}
Here $u_t\sim\operatorname{Unif}(0,1)$ and $\tau_t$ is the first step after the most recent
wealth reset; $\varepsilon_k\in\{0.05,0.10,\ldots,0.95\}$. The deterministic value
$\nicefrac{1}{2}$ is used only for the first scored observation of a run; resetting wealth does
not empty $W_t$. A crossing opens a
50-step recovery window that resets the covariance, boosts $Q$, and overrides $\lambda$; otherwise,
the spike check clips or skips an extreme innovation before the Kalman and hypergradient steps.
Algorithm~\ref{alg:gate} in Appendix~\ref{app:spec} gives the complete state transition, including
the first-score convention and the classical no-reset process used in Section~\ref{sec:e6}.

\textbf{The gate.} A conformal test
martingale~\citep{vovk2003testing,vovk2021retrain,ville1939etude,shafer2021testing,ramdas2023game}
runs on the scores, with the randomized conformal $p$-value of \citet{vovk2022algorithmic} and
a fixed 19-point mixture of power martingales as its wealth (Algorithm~\ref{alg:gate}). The sliding calibration window, the wealth reset, and the learned covariance each
depart from the classical construction, so this gate carries no proven Ville bound of its own.
Section~\ref{sec:e6} replays the textbook version, the same $p$-value over a calibration set that
grows to hold every score and the same wealth with no reset. Its Ville bound applies only when the
score sequence satisfies the required null. We therefore call it the classical no-reset wealth
process when replaying observed scores. Throughout the paper, we say that the gate \emph{fires},
the no-reset process \emph{crosses}, and an e-detector \emph{alarms}. The
spike check runs only on a step with no fire and no open recovery window, and the
three response components act together on a fire (Appendix~\ref{app:spec}).

\textbf{Feedback protocols.} Let $o$ be a forecast origin and $r$ a physical reveal step. At
each $r$, the adapter predicts the current origin, then evaluates any available feedback using the
state and covariance predicted at $r$; it does not restore issue-time state. Immediate reveal uses
the complete packet from $o=r$, an upper-information diagnostic unavailable in deployment.
Delayed reveal uses the packet from $o=r-96$. Composite coordinate reveal
\citep{kim2025tafas,lau2025dsof} combines the newly observed coordinate from each earlier origin;
all 96 rows concern the same physical target time.
The fourth protocol is the \emph{feedback-stride-96 protocol}: every origin is forecast and
evaluated for MSE, but only every 96th delayed packet produces a gate score and adapter update.
It therefore changes both the learner's feedback and the gate's exposure. A separate
\emph{detector-stride-96 variant} updates the adapter from every feedback packet but sends only every
96th score to the wealth process. Appendix~\ref{app:spec} defines both variants explicitly. Two shadow variants,
one learning the observation-noise vector $R$ and one holding it, save the filter's scores under the first two
protocols while the gate acts on nothing (Section~\ref{sec:e6}).

\textbf{Evaluation frame and terms.} We evaluate on the three frozen backbones and five streams
of Section~\ref{sec:intro}. Each stream has a test block of 12{,}922 to 52{,}117 steps, and every
run replays cached backbone forecasts. A cell is one backbone-stream pair, giving fifteen cells.
A contamination setting is the clean condition or one injector at one fraction, ten in all.
Each cell uses three Monte Carlo seeds on the same cached stream. They are algorithmic replicates,
not independent datasets. A variant is
one configuration of the adapter and its gate. A contaminated run's twin is the clean run of the same cell and seed,
and the fed labels are the labels the adapter receives. Three injectors corrupt origin-horizon
feedback packets with a packet spike, packet level shift, or packet freeze at 1, 5, and 10 percent
of origins, in robust units of the clean labels (Appendix~\ref{app:spec}). They do not simulate a
single timestamp-indexed sensor path. A run whose running MSE exceeds 100 times the
seasonal-naive reference's, checked from step 200 on, is truncated and marked diverged. Let $L$
be a run's test-block MSE, the final value of that running mean. We report two primary metrics,
\begin{equation}
\text{ratio} = \frac{L - L_{\text{twin}}}{L_{\text{twin}}},\qquad
\text{gain} = 1 - \frac{L}{L_{\text{frozen}}},
\label{eq:metrics}
\end{equation}
We average the ratio in~\eqref{eq:metrics} over surviving backbone-seed pairs within each stream
and setting, and report counts below nine. Clean gain uses the frozen backbone from the same cell.
All final results use the completed forecast cache; Appendix~\ref{app:repro} audits the earlier
campaign. These panel summaries are descriptive. Ratios are conditional on both runs surviving the
fixed divergence guard; divergence is reported separately.

\textbf{Experimental configurations.} The fixed-construction null feeds i.i.d.\ synthetic scores through the
gate's own logic with no learner behind it. The no-parameter-drift simulation runs the
complete pipeline on a fixed-design regression stream with i.i.d.\ residuals, 60 runs of
20{,}000 steps; its design is time-varying, so its ordered pairs are neither identically
distributed nor exchangeable (Appendix~\ref{app:spec}). The original campaign runs four primary
Kalman variants and two SGD baselines in every real-data cell. The Kalman variants are the ungated filter; the Huber-only variant; the
full gate; and the fixed-cutoff variant, which uses only the spike check and no martingale. The
Huber-only variant retains the martingale and spike classifier. A fire opens or restarts recovery and
suppresses spike classification, but applies none of the three drift responses. Three exploratory
single-component response ablations (soft reset, boost, override) complete the nine-configuration grid. Later
sections add variants where they are used; Appendix~\ref{app:spec} lists them all.
 \section{Premise failure on the deployed score streams}
\label{sec:premise}

We run the gate on the five real streams under immediate reveal and count fires on clean replay
at $\alpha=0.05$. All 135 clean runs of the three martingale-carrying variants fire at least once,
and delayed reveal does not remove the pattern (Figure~\ref{fig:alarmcensus};
Appendix~\ref{app:census}).

\begin{figure*}[t]
\centering
\includegraphics[width=\textwidth]{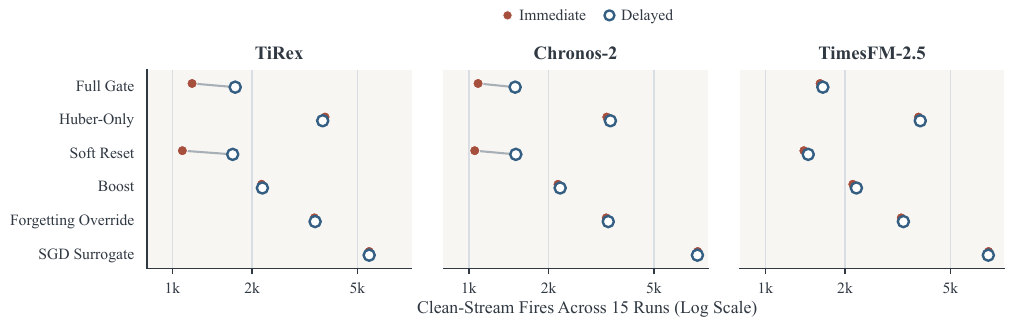}
\caption{Clean-stream alarm census, summed over five streams and three seeds per backbone.
Filled markers denote immediate reveal and hollow markers delayed reveal; connectors pair the same
variant and backbone. Every plotted martingale-carrying variant fires repeatedly under both protocols.
Totals establish firing; rate comparisons require exposure normalization. Appendix~\ref{app:census}
reports per-run ranges and first-fire steps.}
\label{fig:alarmcensus}
\end{figure*}

Two diagnostic experiments locate the failure. The fixed-construction null
fires in at most 1 of 60 runs, below the nominal reference $0.05\times60=3$. This is an empirical
calibration check; the deployed sliding-window gate has no formal Ville bound. The
no-parameter-drift simulation fires in every one of
its 60 runs, 33 to 42 times, the first fire 18 to 67 steps after warm-up ends
(Appendix~\ref{app:null}). In this finite synthetic check, the gate is quiet on i.i.d. input, so
the pipeline's own scores break the premise even though the data-generating parameters do not
drift.

\subsection{Element-level isolation}
\label{sec:isolation}

A second campaign of twenty-two variants holds the adaptive simulation's elements one at a time over
the same 60 seeds (Table~\ref{tab:groups};
Appendix~\ref{app:null} lists each). Nineteen variants keep the campaign's 500-step warm-up and
default initialization. The eleven that learn $R$ online fire or cross in 60 of 60 runs. The
eight that hold $R$ fixed fire or cross in 1 to 3 of 60. Online estimation of an unconverged $R$
is the feature shared by every variant that fires in all 60 runs. Four variants swap the gate for the
classical no-reset wealth process of Section~\ref{sec:setup}, run in shadow. If the scores were
exchangeable, Ville's inequality would put the expected number of runs crossing
$\nicefrac{1}{\alpha}$ at or below 3 of 60; it crosses
in 60 of 60 with the scale learned and in 1 of 60 with it held, 3 of 60 on the permuted stream.
These score processes violate the exchangeability assumption, so Ville's bound does not apply.
The element-level interventions attribute the excess crossings in this simulation to the convergence
transient in the learned observation-noise estimate. The permuted stream
makes the same point on an exchangeable input, where the complete pipeline still fires in 60 of 60
runs and the same pipeline with no hyperparameter learning in 1 of 60 (Appendix~\ref{app:null}).

The mechanism is visible in the score. Held fixed, the mean across the 96 coordinates of $R$
stays at $0.0079$, about three times the true coordinate value $0.0025$, and the mean score is
about $30$, whereas a well-specified $\chi^2_{96}$ score has mean 96. Learned online, the same
coordinate mean converges while the score climbs from $90.2$ to $96.1$. Each new score then exceeds most of its
calibration window, so the $p$-values run small and dependent and both constructions read the
convergence as drift. Neither longer-warm-up variant fires in any of its 60 runs, while Ljung-Box
rejects at level 0.05 in 17 and 13 of 60 for the 2{,}000- and 5{,}000-step variants, respectively.
Starting $R$ at its true value reduces firing to 2 of 60 runs, while Ljung-Box rejects in 58 of
60 (Table~\ref{tab:groups}). Appendix~\ref{app:null} also reports the stricter 0.01
level. The firing in this simulation is therefore the estimate's convergence transient.

\subsection{The frozen-scale variant on the real streams}
\label{sec:frozenr}
We next test this modification on the real streams. The variant holds $R$ at its warm-up value while
$\lambda$ and $Q$ keep learning. It fires in 45 of 45 clean runs under each protocol, as the
full gate does, so the convergence transient does not explain the real-stream firing
(Appendix~\ref{app:census}).

\subsection{The classical no-reset wealth process and origin overlap on real streams}
\label{sec:e6}

The classical no-reset wealth process replays the saved scores of the two shadow variants
(Section~\ref{sec:setup}) offline. It crosses in 180 of 180 runs with the noise scale learned or
held. If these scores were exchangeable, Ville's inequality would put the expected count at or
below 9. In all 180 runs, the marginal-uniformity and serial-dependence diagnostics show
substantial departures from the joint i.i.d.-uniform premise. The sliding window
and wealth reset are not necessary for crossing. The phenomenon is not unique to the block
Kalman design. The SGD adapter's own score stream crosses in 45 of 45 runs under each protocol.
Thinning the saved scores offline to stride 96 lowers the crossing fraction of run-offset pairs from 0.631 and 0.673 to 0.032 and 0.044 under the two reveals, a descriptive contrast since offsets of one run share its stream (Appendix~\ref{app:valid}). Figure~\ref{fig:acf} shows the shape of the dependence. The saved
scores' autocorrelation falls from near one at lag 1 to zero near the horizon with a kink at lag
96, faster than the linear decline of a moving sum sharing $96-k$ of its 96 terms at lag $k$ and
unlike a geometric decay.

\begin{figure*}[t]
\centering
\includegraphics[width=\textwidth]{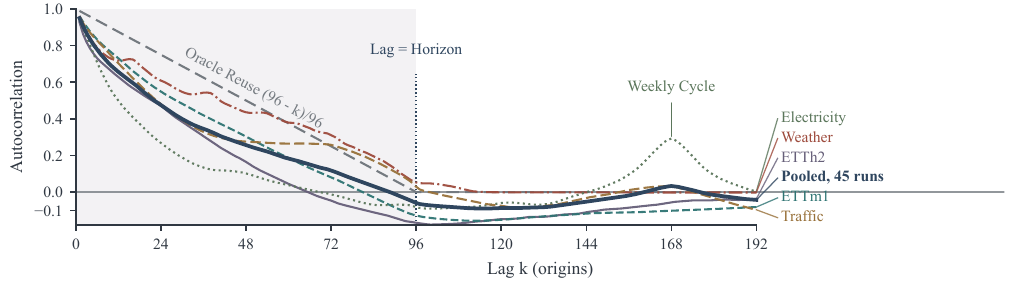}
\caption{Autocorrelation of saved learned-scale scores under immediate reveal. Thin curves average
nine runs per stream; the bold curve averages all 45 runs, and the dashed line is the oracle reuse
prediction $(96-k)/96$. The negative post-horizon floor combines estimator bias with stream
dependence; electricity's rise near lag 168 is its weekly cycle.}
\label{fig:acf}
\end{figure*}

A conditional null changes the reuse and nothing else (Table~\ref{tab:overlap},
Appendix~\ref{app:valid}). The independent variant draws fresh label noise at every origin; the reused variant slices one noise path into 96-step labels, so consecutive labels share 95 of 96 entries, as under immediate reveal. With the scorer held at the truth, the
independent variant's scores are i.i.d.\ $\chi^2_{96}$ and Ville's bound holds, while the
reused-label scores form a 96-term moving sum that is not exchangeable, with autocorrelation
$(96-k)/96$ at lag $k$ (Proposition~\ref{prop:reuse}). Reuse alone takes the classical no-reset
wealth process from 2 or 3 of 60 crossings to 58 or 59 of 60, with the scorer at the truth, and from
1 or 2 of 60 to 54 or 55 of 60 with the adapter learning. Thinning to origins 96 apart restores
validity for the truth-held scorer under Corollary~\ref{cor:thin}. For the learning adapter,
thinning reduces the observed crossing rate, but no theoretical guarantee follows
(Appendix~\ref{app:valid}). The contrast survives Student-t and autoregressive noise
and horizons 24 and 192 (Appendix~\ref{app:valid}).

\begin{proposition}\label{prop:reuse}
Let $\eta_1, \eta_2, \ldots$ be i.i.d.\ with $\eta_1^2$ non-degenerate, let
$0 < \sigma < \infty$, and let $s_i = \sigma^{-2}\sum_{j=0}^{m-1}\eta_{i+j}^2$ for some
$m \ge 2$. The independent score is the same sum over a row of $m$ i.i.d.\ copies of $\eta_1$ drawn
independently at every origin. Then
(i) the independent scores are i.i.d.\ while the sequence $(s_i)_{i \ge 1}$ is not exchangeable,
and (ii) if $\mathbb{E}\,\eta_1^4 < \infty$, then $\operatorname{corr}(s_i, s_{i+k}) = (m-k)/m$
for $0 \le k < m$ and $0$ for $k \ge m$.
\end{proposition}

\begin{table*}[t]
\centering
\caption{Label-reuse intervention on the conditional null, 60 runs of 20{,}000 steps per variant.
Gate fires count runs where the deployed gate fires. No-reset crossings list three tie-break RNG
streams; under exchangeability, Ville's expected reference is at most 3 of 60. The last two columns
give descriptive fractions over 5{,}760 dependent run-offset pairs. Appendix~\ref{app:valid}
reports fire counts, first crossings, and score autocorrelation.}
\label{tab:overlap}
\small
\setlength{\tabcolsep}{12pt}
\begin{tabular}{llcccc}
\toprule
scorer & labels & gate fires & no-reset crossings & stride-96 & matched block \\
\midrule
truth & independent & 0/60 & 2, 3, 3 & 0.013 & 0.011 \\
truth & reused & 60/60 & 59, 58, 59 & 0.009 & 0.735 \\
learning & independent & 1/60 & 1, 2, 1 & 0.013 & 0.024 \\
learning & reused & 60/60 & 55, 55, 54 & 0.003 & 0.719 \\
\bottomrule
\end{tabular}
\end{table*}

Two further analyses separate reuse from other sources of dependence
(Appendix~\ref{app:valid}). Scoring one horizon at a time removes the direct label reuse while
keeping the adapter. The classical no-reset wealth process still crosses in 21 to 90 of 90 runs by horizon.
Median peak log-wealth is 0.3 to 19.3, versus 47.6 and 49.9 for the pooled scores. Removing the reuse silences 53 of the 480 stream-horizons
(Table~\ref{app:tab:horizons}). A no-learner analysis scores the frozen backbones' own
residuals; they cross in 45 of 45 runs pooled with each horizon's scale fixed on the training
block, and in 39 of 45 under a trailing 250-step scale matched to the adapter's. On four of the
five streams the dependence therefore precedes any adaptation. On weather the adapter adds
dependence that the matched raw residuals show much more weakly. The pooled scores show stronger
dependence than the horizon-separated analyses in both cases.
 \section{Attribution of the damage to the soft reset and the protocol}
\label{sec:attribution}

One transient in the weather stream anchors this section. It lands at steps 18{,}536 to
18{,}538 on every backbone and seed under immediate reveal. The full gate turns it into a
divergence on two of three backbones on 3 of 3 seeds. TimesFM~2.5 survives every seed; the full
gate finishes at 2.57--3.06, versus 0.36 for the ungated filter
(Table~\ref{app:tab:transient-imm}). In an exploratory component-ablation campaign, the
soft-reset variant alone reproduces the divergence on the same transient, on 3 of 3 seeds on TiRex
and Chronos-2 (Figure~\ref{fig:anatomy}; the divergences recur on the completed cache, Appendix~\ref{app:repro}).

\begin{figure*}[t]
\centering
\includegraphics[width=\figfourwidth]{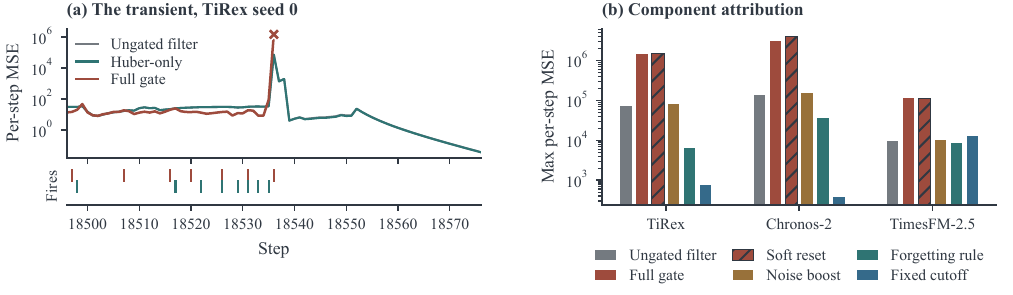}
\caption{Weather transient under immediate reveal. (a) TiRex, seed 0: per-step MSE, guard
termination (red $\times$), and fires (ticks). (b) Completed-cache maximum per-step MSE by variant and backbone.}
\label{fig:anatomy}
\end{figure*}

Under delayed reveal the full gate goes from 60 divergences to 2 over 450 runs
(Appendix~\ref{app:divergence}). It survives the transient on every backbone and seed, and it
still has the worst MSE among the four primary Kalman variants there
(Table~\ref{app:tab:transient-del}). The three single-component variants each diverge in 2 of 450
delayed runs, at the full gate's two weather spike cells. The delayed-reveal data therefore do not
identify one component as the unique cause. Over twenty seeds per backbone on the clean weather
stream, the soft-reset variant finishes above the ungated filter in 20 of 20 seeds on every backbone.
It finishes above the full gate in 18, 14, and 13 of 20 paired seeds, while the boost and override
variants finish below the full gate on every backbone (Table~\ref{app:tab:seeddeltas}).

Removing the soft reset alone from the full gate, with the recovery rule kept, removes the
immediate-reveal divergence and most of the cost under both protocols. That variant finishes the
clean weather cell at 0.136 to 0.434 under immediate reveal, where the full gate diverges or
reaches 3.06, finishes below the full gate in 20 of 20 delayed seeds on every backbone, and
diverges in 1 of 450 immediate runs against the full gate's 60, and in 3 of 450 delayed runs
against its 2 (Appendix~\ref{app:divergence}). On this event, within this response design,
the soft-reset-only variant reproduces the observed divergence and most of the measured damage, and
removing the reset alone avoids most of that damage.
 \section{Alternative monitors and adapter value}
\label{sec:cutoff}

The frozen backbone provides the reference for the adapter's value (Appendix~\ref{app:gain}).
Under diagnostic immediate reveal, the adapter gains on 12 of 15 cells but multiplies the frozen
error 2.9 to 22 times on weather. Under composite coordinate and delayed reveal, it loses on 14
and 15 of 15 cells, with median gains of $-0.34$ and $-0.60$ (Appendix~\ref{app:causal}). As a
third-party diagnostic, the ORCA residual refiner~\citep{dai2026orca} gains on all 12
non-weather cells in our evaluation, suggesting that the Kalman result is not universal. Different inputs
and warm-up rules prevent a fair ranking.

An offline oracle using the adjacent label recovers 84 to 95 percent of the frozen model's error on
four streams; the tested causal residual summaries recover at most 23 percent. Fixing the forgetting
factor still yields negative median gain. Under causal reveal, the adapter's post-change error is
0.16 to 0.73 of the frozen error off weather but 2.0 to 323 times
it on weather (Appendix~\ref{app:causal}).

\begin{figure*}[t]
\centering
\includegraphics[width=\textwidth]{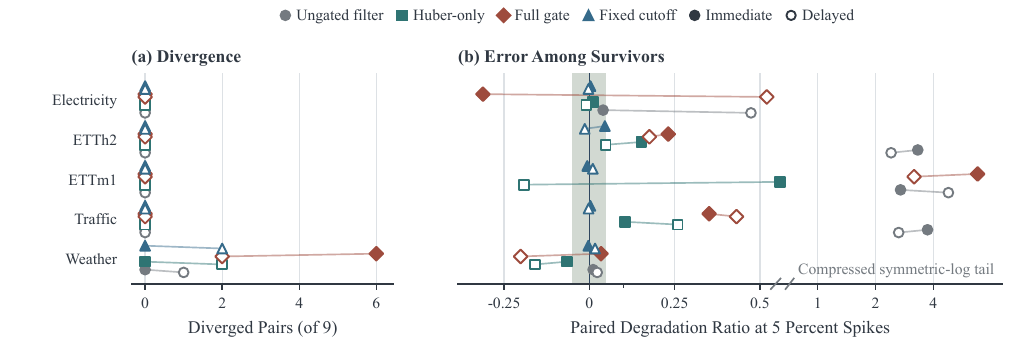}
\caption{Packet-spike divergence and survivor-only degradation at 5 percent. Filled and hollow
markers denote immediate and delayed reveal. The fixed cutoff remains within the shaded $\pm0.05$
band among survivors; two delayed-weather pairs diverge. The full gate reaches $+6.81$ on ETTm1;
Appendix~\ref{app:divergence} gives all cells.}
\label{fig:spikeprofile}
\end{figure*}

In the selected 36 immediate-reveal runs where the ungated adapter has positive clean gain,
excluding weather, the capped policy retains 0.96 of that gain under packet spikes, versus
$-0.35$ for the ungated filter. This cohort does not establish general effectiveness. The full
gate, fixed cutoff, and capped policy all lose under packet-level shifts and freezes. Appendix~\ref{app:gain}
reports the complete retention analysis and the exploratory frozen-backbone mixture.

The fixed cutoff has no anytime-valid guarantee. Its median clean action rate is 34.0 percent,
against a nominal 0.1 percent. Its survivor-only ratio stays within $\pm0.05$ on every stream, but
two delayed-weather pairs diverge. Under the fixed 100-fold guard from step 200, adverse delayed
totals (divergence or higher error) are 45/45 ungated, 22/45 Huber-only, 40/45 full gate, and 18/45 cutoff
(Figure~\ref{fig:spikeprofile}; Appendix~\ref{app:divergence}).

Two exploratory variants test response capping and detector thinning. Capping removes covariance reset
and recovery re-entry; during each 50-step recovery window, it holds $\lambda=0.95$ and linearly
decays the tenfold process-noise boost. It
diverges in 1 of 450 immediate runs and no clean runs. Detector stride 96 lowers clean-run firing
from 45/45 to 11/45 and 10/45 under immediate and delayed reveal, with 6/450 and 2/450 divergences
(Tables~\ref{app:tab:census} and~\ref{tab:divergence}).

Within 960 steps at magnitude 1.0, the $A=400$ e-detector alarms on frozen residuals in 28/45 runs,
versus 0/45 clean twins and 0/45 adapter-residual runs. Frozen scores remain 14--17 times baseline
through step 2{,}000. Over the full post-onset test block, the live stride-96 detector alarms in
43/45 runs under either reveal versus 8/45 clean twins; among detected runs, median actionable
latencies are 405 and 501 steps~\citep{shin2024edetectors}. The signal improves detection, but hold
and handover force a tradeoff between contamination robustness and genuine-change adaptation.
Dependence leaves the detector's average-run-length guarantee unjustified; Appendix~\ref{app:power}
reports the negative short-window results.

At 5 percent, packet freezes degrade the fixed cutoff on four non-weather streams ($+1.72$ to
$+13.76$); packet shifts degrade every tested ETTh2 variant (13/13) and all Kalman variants on ETTm1 and traffic.
The spike check therefore does not prevent sustained-change damage: the prespecified composite fails
on all three backbones
(Table~\ref{app:tab:sustained-all}, Figure~\ref{fig:negatives}).
 \section{Implications and scope}
\label{sec:discussion}

\noindent\textbf{Proved.}\enspace Overlap breaks exchangeability of oracle scores; disjoint origins
restore it only for that scorer.

\noindent\textbf{Measured.}\enspace On real streams, alarms occur on dependent, marginally
nonuniform scores. Adaptation can amplify damage, and neither system improves prediction
uniformly. Neither procedure has its nominal guarantee on these streams.

\noindent\textbf{Scope.}\enspace Evidence covers one adapter, three backbones, five fixed paths,
and horizon 96. Most cells use three seeds; the delayed-reveal clean-weather ablation uses 20 per
backbone. A guarantee matters only when its assumptions hold for the deployed score process.
 
\subsection*{AI use statement}

In this work, we used generative AI to help interpret saved experimental outputs, draft the
statement and proof of Corollary~\ref{cor:thin} in Appendix~\ref{app:valid}, and edit the
manuscript, including suggesting word choices. The authors checked all interpretations against
the saved records, verified the proof line by line, and reviewed every AI-assisted change. We did
not use generative AI to generate synthetic data, form hypotheses, design the methodology,
implement code, or clean data; translation and qualitative analysis are not applicable to this
work. Every quantitative claim traces to a saved experiment record, and every rule fixed before
the runs is reported as measured. We take responsibility for the final content of this work,
including text, claims, and files produced with the aid of generative AI.

\subsection*{Reproducibility statement}

Every final result replays the completed backbone-forecast cache. Appendix~\ref{app:repro}
records the earlier incomplete-cache campaign as a provenance audit, specifies the rule that
selects completed reruns for the final record, and gives the dataset, split, normalization,
model-release, seed, and configuration conventions. Appendix~\ref{app:spec} specifies the
adapter recurrence, reveal-time event order, gate, injectors, and synthetic nulls. Every number in the main text is
stated with its cohort and traces to a saved result record or action log. The classical no-reset
wealth process's replays check the saved $p$-values and crossings against an offline recomputation in
every run, and the conditional-null variants use matched pre-treatment draws. The supplementary material
contains the code for the main experiment families and the dependencies they require. Their
configurations generate the primary settings, diagnostic experiments, and follow-up settings rather than one entry
point per reported variant. The full codebase, saved result records, action logs, and
completed forecast cache will be released upon publication. The external-refiner evaluation is
reported as a third-party diagnostic and uses the released methods at their default settings.
The ETT streams are released under CC BY-ND 4.0 and the electricity stream derives from the
UCI ElectricityLoadDiagrams20112014 release~\citep{trindade2015electricity} under CC BY 4.0. The weather and traffic streams are the long-horizon
benchmark versions of \citet{wu2021autoformer}, obtained from a maintained benchmark release, and
carry no license identifier of their own. Chronos-2 and TimesFM 2.5 are released under Apache 2.0 and TiRex under
the NXAI community license. None of these assets is redistributed.
The saved score and $p$-value streams, the coordinate and residual-permutation bundles, and the
refiners' saved forecasts behind Table~\ref{app:tab:refiner-self} are not included in the
supplement. They will be released with the full codebase upon publication.

\subsection*{Ethics statement}

This work uses five public forecasting benchmarks and three publicly released frozen forecast
models, and it involves no human subjects and no personal data. Its subject is a deployed safety
mechanism whose guarantee does not hold under the conditions studied here. We report that
failure, its measured causes, constructive diagnostics, and response baselines. These results
help practitioners avoid relying on the guarantee where its premise fails. The authors declare
no conflicts of interest.
 
\begingroup
\setlength{\emergencystretch}{3em}
\bibliographystyle{plainnat}

\endgroup

\clearpage
\onecolumn
\appendix
\paragraph{Appendix roadmap.}
The appendix is organized by the questions raised in the main paper. Because several sections
contain both original and later analyses, the roadmap assigns status to analyses rather than to
entire sections. A follow-up is called decision-ruled only when its rule was fixed before that
campaign ran; all other later readings are labeled exploratory in their text or captions.

\begin{center}
\small
\setlength{\tabcolsep}{5pt}
\begin{tabular}{@{}p{0.25\textwidth}p{0.22\textwidth}p{0.43\textwidth}@{}}
\toprule
evidence & status & location \\
\midrule
Four primary outcomes & original campaign & Primary analyses in Appendices~\ref{app:census}, \ref{app:divergence}, and~\ref{app:tables}. \\
Premise and null analyses & theory-guided analysis & Proposition~\ref{prop:reuse} and the core analyses in Appendices~\ref{app:valid} and~\ref{app:null}. \\
Recorded-rule follow-ups & decision-ruled follow-up & The named rules for causal feedback, external refiners, and detection response. \\
Additional diagnostics and negative results & exploratory follow-up & Later or post-hoc readings in Appendices~\ref{app:census}--\ref{app:gain}, \ref{app:edetector}, \ref{app:power}, and~\ref{app:latency}. \\
Methods, records, and compute & procedure/audit & Appendices~\ref{app:spec} and~\ref{app:repro}. \\
\bottomrule
\end{tabular}
\end{center}

\section{Experimental specifications}
\label{app:spec}

Figure~\ref{fig:loop} draws the adaptation loop of Section~\ref{sec:setup}. The two synthetic nulls
below sever that loop at the stream and at the learner.

\textbf{Adapter.} The filter's state holds a slope and an intercept per horizon, 192 entries for
the 96-step horizon, and the frozen forecast enters through each horizon's observation row
$[f_{t,h}\;\;1]$, so the state dimension is set entirely by the forecast horizon. Every run starts from the same
library defaults. After each non-skipped feedback packet, the hypergradient step on the innovation
likelihood updates $\lambda$ and the entries of $Q$ and $R$, so the filter's own noise model tracks the stream it
calibrates without a held-out validation split.

\textbf{Drift response and constants.} On a fire the covariance is soft reset to $4P_t$, the
process noise is boosted tenfold with linear decay over a 50-step recovery window, and the
effective forgetting factor is held at 0.95 throughout that window; at the weather event traced in
Section~\ref{sec:attribution} the boost multiplier is 9.6, and the logged process noise there is
$2.0\times$ the ungated filter's. Spike classification is suppressed while a window is open. The
learned forgetting factor sits at its 0.90 floor on weather, where the override lengthens the
memory, and at medians of 0.95 to 0.997 on the other four streams, where it shortens the effective
memory to about 20 origins. The spike check clips an innovation whose score exceeds $q$, the
0.999 quantile of $\chi^2_{96}$ (144.6), to Mahalanobis norm $\sqrt{q}$, and skips the update
above $3q$. The hypergradient step is Adam~\citep{kingma2015adam} with step sizes $10^{-3}$ on
$\lambda$ and $10^{-2}$ on $\log Q$ and $\log R$, started at $\lambda = 1$, every entry of
$Q$ at $10^{-4}$, and every entry of $R$ at 1, with $\lambda$ kept in $[0.90, 1]$. The
coordinates of $Q$ and $R$ then adapt separately. The capped policy removes the covariance soft
reset and disables recovery re-entry. It retains the full gate's $\lambda=0.95$ override and
tenfold process-noise boost with 50-step linear decay. A fire during an open recovery window is
recorded but does not restart the window. The full gate without its soft reset retains re-entry.

\textbf{Contamination injectors.} These interventions operate on the origin-horizon label array
after overlapping labels have been formed; they do not first corrupt one timestamp-indexed sensor
path and then slice it into labels. The same physical timestamp can therefore receive different
values in different origin-horizon entries. The experiments measure corrupted feedback vectors
and do not model a timestamp-consistent sensor failure. Each injector places its corruption at exactly
$\operatorname{round}(fT)$ steps of the $T$-step test block, for a contamination fraction
$f \in \{0.01, 0.05, 0.10\}$. Each injection's magnitude is set in robust units of the clean labels'
own scale, 1.4826 times the median absolute deviation computed per horizon. We call these packet
interventions ``spike,'' ``level shift,'' and ``freeze'' throughout. Spike alters
isolated uniformly chosen steps by six scale units, with a per-step random sign. Level shift
applies a three-unit offset over one contiguous block, with one sign per horizon. Packet freeze holds
one contiguous block at its onset value and carries no magnitude at all.

\textbf{Variants and protocols.} The original campaign contains four primary Kalman variants and two SGD
baselines. Every variant runs under immediate reveal, and fifteen also run under delayed reveal. Those
fifteen are the six original variants (the ungated filter, the Huber-only variant, the full gate, the
fixed-cutoff variant, and the two SGD baselines), the frozen-scale
variant of Section~\ref{sec:frozenr}, the two capped policies and the detector-stride-96 gate of
Section~\ref{sec:cutoff}, the three single-component variants, whose delayed campaign ran after the
immediate one on the completed cache, and the two variants added last, the full gate without its soft
reset (Section~\ref{sec:attribution}) and the detector-stride-96 gate with the capped response
(Section~\ref{sec:cutoff}). Each has its own result file under both protocols. Two reference
variants complete the set. The frozen-backbone variant emits $f_t = H_t\theta_0$ and performs no update.
It reads no label, so label contamination cannot touch it, and one clean cell per
backbone and stream serves as its reference for every cell. The fallback mixture forecasts
\begin{equation}
\hat y^{\mathrm{mix}}_t = w_t\,\hat y^{\mathrm{ad}}_t + (1-w_t)\,\hat y^{\mathrm{fr}}_t,\qquad
w_t = \frac{\exp(-\eta C^{\mathrm{ad}}_t)}{\exp(-\eta C^{\mathrm{ad}}_t) + \exp(-\eta C^{\mathrm{fr}}_t)},\qquad
\eta = \frac{1}{8D^2},
\label{eq:fallback}
\end{equation}
where $C^{\mathrm{ad}}_t$ and $C^{\mathrm{fr}}_t$ are the cumulative squared losses of the ungated
filter and of the frozen backbone on the fed labels, averaged over the horizon, over the origins
revealed before the forecast, and $D$ is the training block's largest frozen residual. The
adapter's forecast and the fed label are clipped to the band of half-width $D$ around the frozen
forecast. On a run whose frozen residual stays inside that band, the mixture's cumulative
fed-label loss is within $8D^2\log 2$ of the better expert~\citep{cesabianchi2006prediction}, and
it forecasts as the frozen backbone after the guard truncates the adapter. Its first version,
fixed before the runs, set the rate as $\eta = 1/(8B^2)$ from the training block's largest label
magnitude $B$, clipped nothing, and carried a fixed share of $10^{-3}$ (Appendix~\ref{app:gain}). Two protocol and feature variants
enter Appendix~\ref{app:causal} alone, the composite coordinate reveal and the fresh-residual features, each as its own configuration. The two live two-stream variants of
Appendix~\ref{app:power} hold the adapter for 480 steps after an actionable alarm of the
frozen-residual e-detector, the handover variant emitting the frozen forecast during the hold. The
second SGD baseline is gated by a surrogate score whose covariance proxy is a trailing estimate of
the innovation's own scale. All final evaluations share every backbone, stream, and seed, the
same completed cached forecasts, and the same contamination draws. Appendix~\ref{app:repro}
retains earlier incomplete-cache records only as a provenance audit. Within each paired ablation,
the remaining design is held fixed, so the observed difference isolates the named component
under that experimental intervention.

\begin{table*}[t]
\centering
\caption{Operational definitions of the primary gate variants and the two response variants used in
the main conclusions. ``Re-entry'' means that a fire during an open 50-step recovery window
restarts that window.}
\label{app:tab:variant-semantics}
\footnotesize
\setlength{\tabcolsep}{3pt}
\begin{tabular}{@{}p{0.12\textwidth}p{0.11\textwidth}p{0.17\textwidth}p{0.43\textwidth}p{0.07\textwidth}@{}}
\toprule
variant & monitor & spike check & fire and recovery behavior & re-entry \\
\midrule
full gate & martingale & outside recovery & opens recovery; suppresses the spike check; applies the $4P_t$ reset, a tenfold $Q$ boost that decays over 50 steps, and a $\lambda=0.95$ override held throughout recovery & yes \\
Huber-only & martingale & outside recovery & opens recovery and suppresses the spike check; applies no reset, $Q$ boost, or $\lambda$ override & yes \\
fixed cutoff & none & every step & no fire or recovery window & n/a \\
capped policy & martingale & outside recovery & opens recovery and suppresses the spike check; applies no reset and the same $Q$ and $\lambda$ response as the full gate & no \\
reset-free gate & martingale & outside recovery & opens recovery and suppresses the spike check; applies no reset and the same $Q$ and $\lambda$ response as the full gate & yes \\
\bottomrule
\end{tabular}
\end{table*}

\paragraph{Compute accounting.}\label{app:compute}
Under the documented scalar-operation convention, the block-aware filter and one
hyperparameter-learning step, including the diagonal approximation for $\lambda$ and its Adam
optimizer update, cost 12{,}230 scalar operations per step. The gate adds 1{,}792, for 14{,}022 in
total, or $0.187\times$ one 75{,}168-operation SGD step;
spike-classified steps add 388. Three generic versions without the block structure cost about
$119\times$, $120\times$, and $239\times$ one SGD step. These arithmetic counts do not imply
speed: the measured Python block loop took 114 microseconds per step versus 27 for dense SGD.

\textbf{Gate state.} The deployed gate holds two pieces of state, a sliding calibration window of
the last 250 scores and the martingale's wealth. The window is built once and every score is
pushed into it. It is never reset, and a fire clears no calibration score. Only the wealth resets
on a fire, restarting at one. The classical no-reset wealth process of Section~\ref{sec:setup}
changes exactly two settings of Algorithm~\ref{alg:gate}, a calibration set that grows and no
wealth reset. Its Ville guarantee requires the stated null; naming the construction does not
establish that null on an observed score stream.

\begin{table*}[t]
\centering
\caption{Feedback packet available at physical reveal step $r$. Here $H_o$ and $y_o$ are the
design and complete label vector of forecast origin $o$, and $(H_o)_{h,:}$ and $(y_o)_h$ denote
their horizon-$h$ row and coordinate. A dash means that no feedback packet is formed.}
\label{app:tab:feedback-packets}
\small
\setlength{\tabcolsep}{5pt}
\begin{tabular}{@{}p{0.15\textwidth}p{0.46\textwidth}p{0.29\textwidth}@{}}
\toprule
protocol & feedback packet $B_r=(H_r^{\mathrm{fb}},y_r^{\mathrm{fb}})$ & availability and use \\
\midrule
immediate vector & $(H_r,y_r)$ & every step; diagnostic and causally unavailable \\
delayed vector & $(H_{r-96},y_{r-96})$ & $r\ge96$; causal, with earlier steps predict-only \\
composite reveal & row $h$ is $((H_{r-h})_{h,:},(y_{r-h})_h)$, $h=1,\ldots,96$ & $r\ge96$; causal; evaluated for the ungated and reference variants \\
feedback-stride-96 & $(H_{r-96},y_{r-96})$ when $(r-96)\bmod96=0$ & all other steps predict-only \\
\bottomrule
\end{tabular}
\end{table*}

\begin{table*}[t]
\centering
\caption{The two stride settings change different parts of the online loop. Every origin is
forecast and evaluated for MSE in both variants. A stride of 1 means that the operation occurs at
every eligible feedback step.}
\label{app:tab:stride-definitions}
\small
\setlength{\tabcolsep}{6pt}
\begin{tabular}{@{}p{0.24\textwidth}cccc@{}}
\toprule
variant & feedback stride & detector stride & adapter updates & gate scores \\
\midrule
detector-stride-96 variant & 1 & 96 & every feedback step & every 96th score \\
feedback-stride-96 protocol & 96 & 1 & every 96th packet & every update \\
\bottomrule
\end{tabular}
\end{table*}

\textbf{Filter and hyperparameter recurrence.} Let $(\theta_{r-1},P_{r-1})$ be the posterior
carried into reveal step $r$, and let $(\widetilde\lambda_r,\widetilde Q_r)$ include any recovery
response decided earlier. The prediction and feedback innovation are
\begin{align}
\theta_r^-&=\theta_{r-1}, &
P_r^-&=P_{r-1}/\widetilde\lambda_r+\operatorname{diag}(\widetilde Q_r),\nonumber\\
e_r&=y_r^{\mathrm{fb}}-H_r^{\mathrm{fb}}\theta_r^-, &
S_r&=H_r^{\mathrm{fb}}P_r^-(H_r^{\mathrm{fb}})^\top+\operatorname{diag}(R_r).
\label{eq:filter-predict}
\end{align}
A relative jitter of $10^{-10}$ times each diagonal entry's magnitude is added to $S_r$. Because
$H_r^{\mathrm{fb}}$ and $P_r^-$ are block diagonal, $S_r$ is diagonal. With the innovation after
the gate's clip, $\bar e_r$, the measurement update is
\begin{equation}
K_r=P_r^-(H_r^{\mathrm{fb}})^\top S_r^{-1},\qquad
\theta_r=\theta_r^-+K_r\bar e_r,\qquad
P_r=(I-K_rH_r^{\mathrm{fb}})P_r^-.
\label{eq:filter-update}
\end{equation}
The implementation symmetrizes $P_r$; a skipped update carries $(\theta_r^-,P_r^-)$ forward. A
soft reset, when enabled by a fire, multiplies the posterior $P_r$ by four after this update.

The hyperparameters take one approximate Adam ascent step while holding $P_{r-1}$ fixed. During
recovery, $\widetilde\lambda_r=0.95$ is a hard override and only the multiplier in
$\widetilde Q_r$ decays. The implementation nevertheless evaluates the diagonal $\lambda$
hypergradient at $\widetilde\lambda_r$ and applies the resulting Adam step to the learned base
$\lambda_r$. This is a surrogate update for the base parameter while the hard override is active;
outside recovery, it is the ordinary derivative for the learned forgetting factor. With
$G_r=\tfrac12(S_r^{-1}\bar e_r\bar e_r^\top S_r^{-1}-S_r^{-1})$, this is the Gaussian
one-step predictive-likelihood gradient when $\bar e_r=e_r$. After the gate clips the innovation,
it is a robustified pseudo-likelihood update. The one-step derivative with respect to the effective
forgetting factor, before the diagonal approximation, is
\begin{align}
\nabla_{\widetilde\lambda}\ell_r^{\mathrm{full}}
 &=\left\langle G_r,H_r^{\mathrm{fb}}(-P_{r-1}/\widetilde\lambda_r^2)
 (H_r^{\mathrm{fb}})^\top\right\rangle.\nonumber
\end{align}
The implementation instead passes only $p=\operatorname{diag}(P_{r-1})$ and uses the diagonal
hypergradient approximation
\begin{align}
\widetilde\nabla_{\lambda}\ell_r
 &=\left\langle G_r,H_r^{\mathrm{fb}}\operatorname{diag}(-p/\widetilde\lambda_r^2)
 (H_r^{\mathrm{fb}})^\top\right\rangle,\nonumber\\
\nabla_{\log Q_j}\ell_r
 &=\widetilde Q_{r,j}(H_{r,:,j}^{\mathrm{fb}})^\top G_rH_{r,:,j}^{\mathrm{fb}}, &
\nabla_{\log R_h}\ell_r&=R_{r,h}(G_r)_{hh}.
\label{eq:hypergradient}
\end{align}
Thus the $\lambda$ update drops the nonzero slope--intercept covariance within each horizon block;
we do not claim that it equals the full-block derivative. The $Q$ and $R$ expressions are the
one-step derivatives for diagonal process and observation scales under the same fixed-$P_{r-1}$
objective.
Adam performs gradient ascent with learning rates $10^{-3}$, $10^{-2}$, and $10^{-2}$ for
$\lambda$, $\log Q$, and $\log R$, respectively, and $(\beta_1,\beta_2,\epsilon)=(0.9,0.999,10^{-8})$.
After each step, $\lambda$ is clipped to $[0.90,1]$, $\log Q$ to $[-20,5]$, and $\log R$ to
$[-10,10]$. Initially, $\lambda=1$, every $Q_j=10^{-4}$, and every $R_h=1$. Thus the scalar
constants specify shared initialization, while the 192 process-noise and 96 observation-noise
coordinates subsequently adapt separately. This is a myopic
one-step gradient, not backpropagation through the earlier filter trajectory.

\begin{algorithm*}[t]
\caption{Reveal-time state transition of the self-monitored loop.}
\label{alg:gate}
\begin{algorithmic}
\Require current-origin design $H_r$, optional packet $B_r=(H_r^{\mathrm{fb}},y_r^{\mathrm{fb}})$,
posterior $(\theta_{r-1},P_{r-1})$, base $(\lambda_r,Q_r,R_r)$, window $W$, wealth $M$
\State \textbf{Predict and emit.} Form $(\theta_r^-,P_r^-)$ by~\eqref{eq:filter-predict} and emit
$\hat y_r=H_r\theta_r^-$ for the current origin.
\If{$B_r$ is absent}
\State Carry $(\theta_r^-,P_r^-)$ forward; do not score, bet, or update the hyperparameters.
\Else
\State \textbf{Score.} Compute $(e_r,S_r)$ from $B_r$ and
$s_r=e_r^\top S_r^{-1}e_r=\sum_{h=1}^{96}e_{r,h}^2/S_{r,hh}$.
\State \textbf{Rank.} If $|W|=0$, set $p_r=1/2$; otherwise set
$p_r=\big(\#\{s_j>s_r\}+u_r(1+\#\{s_j=s_r\})\big)/(|W|+1)$ over $j\in W$,
with $u_r$ uniform on $[0,1]$.
\State \textbf{Bet.} $M\leftarrow K^{-1}\sum_{k=1}^{K}\prod_j
\varepsilon_k p_j^{\varepsilon_k-1}$ over the $p$-values since the last fire, where
$K=19$ and $\varepsilon_k$ ranges from $0.05$ to $0.95$
\If{$M\ge\nicefrac{1}{\alpha}$, with $\alpha=0.05$}
\State \textbf{Fire.} Reset $M$ to one; schedule the posterior soft reset, the next-step
tenfold $Q$ boost with 50-step linear decay, and the $\lambda=0.95$ override held throughout recovery.
\ElsIf{no recovery window is open and $q<s_r\le3q$, where $q$ is the $0.999$ quantile of $\chi^2_{96}$}
\State \textbf{Clip.} Scale $e_r$ to Mahalanobis norm $\sqrt q$.
\ElsIf{no recovery window is open and $s_r>3q$}
\State \textbf{Skip.} Carry $(\theta_r^-,P_r^-)$ forward and skip both the filter and
hyperparameter updates.
\EndIf
\State \textbf{Update.} Unless skipped, apply~\eqref{eq:filter-update} using the treated
innovation, then any scheduled posterior reset and~\eqref{eq:hypergradient}. Append the raw $s_r$ to $W$.
\EndIf
\end{algorithmic}
\end{algorithm*}

\textbf{Synthetic nulls.} The fixed-construction null matches the deployed calibration window from its empty start and runs at both $T = 20{,}000$ steps and the longest deployed test-block horizon
of 52{,}117 steps. The no-parameter-drift simulation uses the same 250-score calibration window and the same level as every
real-data run, with a 500-step warm-up before the gate begins scoring.

\section{Reproducibility}
\label{app:repro}

\textbf{Datasets, splits, and scaling.} Table~\ref{app:tab:data-spec} gives the exact origin
ranges used by the main experiments. Indices are zero-based rolling-window origins. The split is
chronological with nominal proportions 20:5:75. At both internal boundaries, 191 candidate
origins are dropped so that a 96-step context and 96-step label cannot cross the boundary. The
loader drops the date column, retains numeric columns in file order, and uses the last numeric
column, the benchmark's oil-temperature target channel, as the target. It performs no
interpolation, imputation, or row deletion. Both the cached forecast and label are divided by the
target's standard deviation over the training rows, with no centering and a floor of one only if
that standard deviation is at most $10^{-8}$. Reported MSE is on this normalized scale; raw-scale
MSE is obtained by multiplying by the squared recorded scale.

\begin{table*}[t]
\centering
\caption{Dataset and chronological split specification. ``Rows'' is the number of numeric time
steps after removing the date column. Origin ranges are inclusive.}
\label{app:tab:data-spec}
\small
\setlength{\tabcolsep}{6pt}
\begin{tabular}{@{}lrlrrr@{}}
\toprule
stream & cadence & rows & training origins & validation origins & test origins \\
\midrule
ETTh2 & hourly & 17{,}420 & 0--3{,}254 & 3{,}446--4{,}115 & 4{,}307--17{,}228 \\
ETTm1 & 15 minutes & 69{,}680 & 0--13{,}706 & 13{,}898--17{,}180 & 17{,}372--69{,}488 \\
weather & 10 minutes & 52{,}695 & 0--10{,}309 & 10{,}501--12{,}934 & 13{,}126--52{,}503 \\
electricity & hourly & 26{,}304 & 0--5{,}031 & 5{,}223--6{,}337 & 6{,}529--26{,}112 \\
traffic & hourly & 17{,}544 & 0--3{,}279 & 3{,}471--4{,}147 & 4{,}339--17{,}352 \\
\bottomrule
\end{tabular}
\end{table*}

\textbf{Backbone releases and inference.} All three models are frozen and receive only the
last target channel from a 96-step context. Each emits a deterministic 96-step point forecast in
float32, cached once per origin; the adapter subsequently operates in float64. CUDA is used when
available, with CPU fallback. Table~\ref{app:tab:backbone-spec} records the model releases and
forecast channels.

\begin{table*}[t]
\centering
\caption{Frozen backbone releases and point-forecast channels.}
\label{app:tab:backbone-spec}
\scriptsize
\setlength{\tabcolsep}{5pt}
\begin{tabular}{@{}p{0.14\textwidth}p{0.43\textwidth}p{0.36\textwidth}@{}}
\toprule
backbone & model release & inference used for the cache \\
\midrule
TiRex & public TiRex release & native deterministic mean head; no compilation \\
Chronos-2 & public Chronos-2 release & returned median quantile channel; float32 weights \\
TimesFM~2.5 & public TimesFM~2.5 PyTorch release & native point channel; maximum context 1{,}024, maximum horizon 256, input normalization and quantile-crossing correction enabled \\
\bottomrule
\end{tabular}
\end{table*}

The validity analysis conditions on the realized forecast and score streams and therefore does
not require an assumption about whether these five public streams appeared in pretraining, which
\citet{meyer2025leakage} document is hard to establish for released time-series checkpoints.
Possible pretraining overlap may nevertheless affect forecast quality and limits broader
generalization.

Every backbone forecast is computed once per origin and cached. Final tables and figures use the
immutable completed-cache record. An earlier
immediate-reveal campaign began before cache construction had finished; it is retained only as
provenance and is not used for final results. Within each completed-cache campaign, runs that
share a stream, backbone, and context receive byte-identical forecast inputs, and each
contaminated run is paired with a clean twin built from the same cell. No value from the earlier
incomplete-cache campaign appears in a final claim.

We reran every cell of the six original variants on the completed cache, including the clean setting
and all three injectors at 1, 5, and 10 percent. Of the 300 stream-setting-configuration ratios, 263 move by
at most 0.01 and 280 by at most 0.05; divergence agrees in 297 cells, and differs by one run in the
remaining three. The full gate's weather divergence recurs in every corresponding rerun. The
largest ratio changes are listed in Table~\ref{app:tab:regen-full}, and
Tables~\ref{app:tab:regen} and~\ref{app:tab:regen2} verify the main spike, freeze, and shift
evaluations on the completed cache. Figure~\ref{fig:anatomy}(b) and the twenty-seed analysis in
Table~\ref{app:tab:seeddeltas} likewise use completed-cache records.

Every saved record carries the complete experimental configuration, including the feedback
protocol, gate settings, and contamination settings. Records from different protocols or gate
configurations are kept separate. The released provenance record keeps the earlier campaign separate
from the completed-cache evidence.

Three seeds cover every cell. Seeds are Monte Carlo replicates conditional on each fixed
cached stream. Each gated run initializes its conformal tie-break generator from the run seed.
Matched clean and contaminated cells with the same seed therefore use the same tie-break sequence
over their common scored prefix. Contamination uses a separate child generator. On clean streams the ungated and fixed-cutoff variants
reproduce across seeds, agreeing to eight significant figures, because neither consumes a
randomized $p$-value. The gated variants vary. The full gate's clean weather runs on TiRex finish
at 80.6, 75.4, and 76.4 across seeds 0, 1, and 2, where the ungated filter finishes at 2.032
in all three. Three seeds are a coarse resolution for any per-cell rate, and every pooled
statistic in this paper carries its raw counts.

Every degradation ratio in this paper, main text and appendix alike, follows~\eqref{eq:metrics}. For a given variant, stream, contamination kind, and fraction, each of the three backbones times three seeds gives one pair, a contaminated run and its twin. A pair enters the mean only when both runs finished without crossing the divergence guard, which truncates a run once its cumulative MSE exceeds 100 times the seasonal-naive
reference's cumulative MSE, floored at $10^{-4}$, checked from step 200 on, and at once on a non-finite
state. This operational guard was fixed before the runs. Because a triggered run is truncated,
the saved records cannot support a higher-threshold sensitivity analysis without rerunning the
affected cells; no such analysis is claimed. A ratio where fewer than nine of nine pairs survive
carries an explicit $[n/9]$ mark. Divergence itself is reported as its own outcome wherever it
occurs, never folded into a mean.

All verdicts, exclusion rules, and contamination severities were fixed before the runs. The
composite gate rule requires the gated variant to beat the ungated
variant, per contamination kind, on at least three of five comparable stream cells. Cells where
the ungated variant itself degrades by less than five percent are excluded as incomparable, and
the gate passes only if every evaluated kind passes. Measured against that rule, the
gate fails on three of three backbones.

The supplementary material contains the code for the main experiment families and the
dependencies they require. The full codebase, immutable completed forecast cache, saved result
records, and action logs will be released upon publication. Earlier incomplete-cache records are
retained separately as provenance and are not used for final results.

\section{Clean-stream fire census}
\label{app:census}

Section~\ref{sec:premise} counts fires on clean, uncontaminated replay across the 135
clean-stream runs of the three martingale-carrying variants. Every clean-stream cell fires many
times over its block. The Huber-only variant fires 31 to 72 times on ETTh2, 459 to 557 times on
ETTm1, 462 to 567 times on weather, 77 to 94 times on electricity, and 26 to 43 times on
traffic, each range spanning three backbones and three seeds. The full gate fires 18 to 35 times
on ETTh2, 146 to 176 times on ETTm1, 115 to 277 times on weather, 36 to 56 times on electricity,
and 22 to 33 times on traffic. The weather range is truncated by clean-stream divergence, taken
up in Section~\ref{sec:attribution}. An SGD-driven surrogate of the same gate fires 121 to 173
times on ETTh2, 634 to 712 times on ETTm1, 787 to 1{,}340 times on weather, 57 to 75 times on
electricity, and 118 to 210 times on traffic, the traffic range also truncated by divergence.
These test blocks differ in length across streams, so raw fire counts should not be quoted as
rates without normalizing for horizon.

Table~\ref{app:tab:census} gives, for the eight variants that carry a live martingale, the step of each
run's first fire and its fire count over the deployed clean runs of each protocol, 45 per variant
(three backbones, five streams, three seeds). Runs from the five swept spike quantiles
(Table~\ref{app:tab:budget}) do not enter.
Over the 135 immediate-reveal runs of the three primary variants the first fire lands at steps 6 to
4{,}274 with a median of 460 and each run fires 18 to 1{,}340 times; Section~\ref{sec:premise}
counts the runs. Under delayed reveal the first fire lands at step 100 or later
on every variant. The capped policy fires at a median of 54 per run under immediate reveal against
the full gate's 42, the frozen-scale variant fires in 45 of 45 runs under each protocol, 24 to 313
times per run under immediate reveal and 25 to 327 under delayed, and the detector-stride-96 gate
fires in 11 and 10 of 45, at most 3 and 7 times. The gate without its soft reset fires at a
median of 54 per run under each protocol against the full gate's 42 and 41, so removing that
response component raises the fire count, and the combined variant fires in 11 and 10 of 45
runs as the detector-stride-96 gate does, at most 4 and 6 times.

\begin{table}[t]
\centering
\caption{First-fire step and fire count over the deployed clean runs, 45 per variant and protocol
(three backbones, five streams, three seeds), the detector-stride-96 gate and the two variants of
Sections~\ref{sec:attribution} and~\ref{sec:cutoff} each from their own store. Fired
counts the runs with at least one fire; first fire is min / median / max of the step of the first
fire over those runs, and fires is min / median / max of the fire count over all 45 runs.}
\label{app:tab:census}
\footnotesize
\setlength{\tabcolsep}{4pt}
\begin{tabular}{llccc}
\toprule
variant & protocol & fired & first fire & fires \\
\midrule
Huber-only & immediate & 45/45 & 11 / 522 / 963 & 21 / 82 / 573 \\
Huber-only & delayed & 45/45 & 100 / 138 / 1{,}057 & 28 / 84 / 586 \\
full gate & immediate & 45/45 & 11 / 522 / 963 & 18 / 42 / 277 \\
full gate & delayed & 45/45 & 100 / 138 / 1{,}057 & 21 / 41 / 299 \\
SGD surrogate & immediate & 45/45 & 6 / 265 / 4{,}274 & 57 / 171 / 1{,}340 \\
SGD surrogate & delayed & 45/45 & 102 / 361 / 4{,}370 & 57 / 171 / 1{,}337 \\
capped policy & immediate & 45/45 & 11 / 522 / 963 & 20 / 54 / 482 \\
capped policy & delayed & 45/45 & 100 / 138 / 1{,}057 & 21 / 60 / 498 \\
frozen scale & immediate & 45/45 & 11 / 538 / 968 & 24 / 49 / 313 \\
frozen scale & delayed & 45/45 & 100 / 139 / 1{,}063 & 25 / 48 / 327 \\
detector-stride-96 gate & immediate & 11/45 & 18{,}432 / 19{,}200 / 50{,}496 & 0 / 0 / 3 \\
detector-stride-96 gate & delayed & 10/45 & 1{,}248 / 12{,}384 / 51{,}264 & 0 / 0 / 7 \\
full gate, no soft reset & immediate & 45/45 & 11 / 522 / 963 & 18 / 54 / 454 \\
full gate, no soft reset & delayed & 45/45 & 100 / 138 / 1{,}057 & 20 / 54 / 492 \\
detector-stride-96, capped response & immediate & 11/45 & 18{,}432 / 19{,}200 / 50{,}496 & 0 / 0 / 4 \\
detector-stride-96, capped response & delayed & 10/45 & 1{,}248 / 12{,}384 / 51{,}264 & 0 / 0 / 6 \\
\bottomrule
\end{tabular}
\end{table}

Section~\ref{sec:cutoff} reports what the variants do to the update channel on those same clean
streams. Per-run spike classifications for the fixed-cutoff variant range from 4{,}190 to 21{,}599
over runs of 12{,}922 to 52{,}117 steps, under both protocols, against per-backbone medians of
5{,}636 to 6{,}109 under immediate reveal. The full gate's per-backbone medians are 3{,}198 to
4{,}219 under immediate and 4{,}058 to 4{,}291 under delayed reveal over the same blocks. Read
as a fraction of each run's own length over the 90 clean runs of each variant and both protocols,
the fixed-cutoff variant classifies 28.6 to 41.4 percent of steps as spikes, at a median of 34.0
percent, the full gate 9.7 to 27.7 percent at a median of 23.4, and the Huber-only variant 8.2 to
26.6 percent at a median of 21.7. The nominal $0.999$ chi-square quantile predicts 0.1 percent
for a well-specified filter.

Table~\ref{app:tab:cutoffq} reruns the fixed-cutoff variant at the $0.9999$ and $0.99999$ quantiles
under both protocols, each as its own configuration, over the full contamination grid.
The clean share falls to medians of 27.7 and 23.5 percent, so even the $0.99999$ quantile
classifies 23{,}500 times its nominal share, and the per-backbone action medians bracket the full
gate's under both protocols as overlapping ranges. The spike ratios at 5 percent stay within
0.05 in magnitude at the deployed quantile and 0.04 at the two raised ones, the divergence
counts do not move, and the packet-freeze and level-shift ratios on the four non-weather streams stay
within their deployed-quantile range or worsen. The sustained-contamination negatives are
therefore not a matter of the cutoff's height.

Table~\ref{app:tab:budget} makes the spike comparison symmetric. Both variants were rerun on the 5
percent spike cell and its clean pair over three seeds per backbone under both protocols, the
cutoff at the $0.999999$ and $0.9999999$ quantiles and the full gate at $0.998$, $0.995$,
$0.99$, $0.98$, and $0.95$, each as its own configuration. The budget is the clean run's
spike-action count, the median over seeds, and for each backbone, stream, and protocol each variant
is read at the available quantile nearest the other variant's deployed budget in log ratio. The
gate's count answers weakly to its quantile, and a twin of the full gate explains why. The twin
is identical in configuration and stepper and records the flag the session raises when an active
recovery period overrides a spike classification, on every clean cell under both
protocols at the deployed quantile and on weather at $0.95$ (Table~\ref{app:tab:suppress}). At
the deployed quantile the override removes 981 to 1{,}274 classifications per weather run, 0.26
to 0.34 of the actions the gate applies, and 263 to 1{,}647 elsewhere, 0.05 to 0.15. Applied and
overridden together reach 0.34 to 0.38 of the cutoff's count on weather under delayed reveal
with equal exposure and 0.60 to 0.91 on the other streams, so most of the gap is the gated
trajectory, on which fewer innovations exceed the same threshold. A second twin with the
override disabled applies every classification. On weather under delayed reveal it acts on 0.57
to 0.62 of the cutoff's count at the deployed quantile, two of its nine runs ending at the
divergence guard, one Chronos-2 and one TimesFM~2.5 seed, and at $0.95$ on 0.75 and 0.82 on
TiRex and TimesFM~2.5 while all three Chronos-2 seeds end at the guard at step 18{,}584. Under
immediate reveal nine of its nine weather runs end at the guard at the deployed quantile and
eight of nine at $0.95$, against six of the deployed gate's, so applying every classification
inside recovery costs divergences.

On weather, where the gate acts through 3{,}272 to 8{,}036 boosted steps per run at the
per-backbone median, neither variant reaches the other's budget. The cutoff at $0.9999999$ still
acts 1.97 to 2.99 times as often as the gate at $0.999$, and the gate at $0.95$ acts 0.24 to
0.42 times as often as the cutoff. Under immediate reveal the gate's clean weather runs on TiRex
and Chronos-2 end at the divergence guard at step 18{,}537 of 39{,}378 while the cutoff's run
the block, so those two raw-count ratios, 2.99 and 2.86, compare unequal exposures. Per step the
cutoff at $0.9999999$ acts 1.41 and 1.34 times as often as the gate there, and under delayed
reveal, with equal exposure, 1.97 to 2.12 times. On the other eight cells exposure is equal and
the match lands within 0.90 to 1.09 of the target in both directions. The largest matched cutoff
ratio is $+0.035$ on ETTh2 under immediate reveal. The cutoff at $0.9999999$ diverges in one of
its nine immediate weather spike runs, TiRex, one seed, its only divergence at any quantile
outside the two delayed weather cells shared by every variant.

The budget counts the spike check the two variants share, and the gate also modifies steps through
its drift response. At the spike-matched quantile its spike actions plus boosted steps come to
1.16 to 1.43 times the deployed cutoff's on the eight matched cells and 0.47 to 0.94 on weather.
Re-matching on that total, with no new runs, chooses the deployed cutoff or its $0.9999$
quantile and the gate's $0.995$ to $0.999$ on those cells and lands within 0.89 to 1.09 of the
target. It leaves the cutoff within 0.05 in magnitude, $+0.05$ on ETTh2 under immediate reveal
at its largest, and the gate at $+5.37$ and $+3.66$ on ETTm1, $+0.18$ to $+0.54$ on ETTh2,
electricity, and traffic under delayed reveal and ETTh2 and traffic under immediate reveal, and
$-0.31$ on electricity under immediate reveal. Fires, which reset the covariance, and the
intensity of each boosted step are counted in neither currency.

Table~\ref{app:tab:regen} reruns the four primary variants of Section~\ref{sec:cutoff} on the
completed cache (Appendix~\ref{app:repro}), under immediate reveal on the clean and 5 percent
spike cells. The regenerated fixed cutoff reproduces its saved records exactly in all 90 runs
and the regenerated full gate's clean runs reproduce the override-recording twin's in all 45, so
the pipeline is bit-reproducible on the completed cache. The pooled spike ratios move by less
than 0.02 on 18 of the 20 stream-configuration cells. The full gate's ETTm1 ratio rises from $+6.28$ to
$+6.81$, and the Huber-only weather ratio moves from $-0.03$ over 8 surviving pairs to $-0.07$
over 9, the saved run's one spike divergence not recurring. Every ordering the text states
holds on both generations, and the cutoff stays within 0.05 in magnitude on every stream and
remains the smallest change of the four variants on each. Table~\ref{app:tab:regen2} extends the
rerun to the packet-freeze and level-shift cells at 5 percent. The fixed cutoff again reproduces its
saved records exactly, 135 of 135 runs, and the pooled ratios move by at most 0.03 on 32 of the
40 stream, injector, and variant cells. The eight that move more are ETTm1 and ETTh2 cells of the
three variants whose saved runs consumed pre-cache forecasts, the largest the full gate's ETTm1
packet-freeze ratio, $+14.17$ saved and $+16.98$ regenerated, then its ETTh2 level shift, $+19.29$ to
$+19.03$, and the ungated filter's ETTm1 level shift, $+15.21$ to $+15.02$. Both negative
results of Section~\ref{sec:cutoff} hold on both generations. Each injector degrades every
gated Kalman variant on the four non-weather streams, and the fixed cutoff's ordering against the
full gate is the same on every cell. On the weather transient under immediate reveal the
fixed-cutoff variant reaches the lowest maximum per-step MSE of the ablation variants on TiRex and
Chronos-2, 793 and 381, and is the weakest Kalman variant on TimesFM~2.5, at test-block MSE
3.413 against an ungated 0.361. Under delayed reveal it is the best Kalman variant on all three
backbones, at 0.167, 0.181, and 0.128.

The detector-stride-96 gate of Section~\ref{sec:cutoff} runs the same per-step spike check, so its
spike-action share and per-backbone action medians match the fixed cutoff's to within 0.003, at
0.337 and 0.340 median share under the two protocols against the cutoff's 0.340 and 0.340, and
6{,}109, 5{,}801, and 5{,}636 actions under immediate reveal. Its martingale, fed every 96th
scored origin, fires 68 times over all 90 clean runs (Table~\ref{app:tab:census}), with a median
of 0 boosted steps per run and at most 350. Its six immediate divergences are all weather cells,
five on TiRex and one on Chronos-2, one seed each, and its two delayed divergences are the full
gate's.

\begin{table}[t]
\centering
\caption{Fixed-cutoff variant at three spike quantiles. Share is the fraction of clean steps
classified as spikes over 45 clean runs per protocol (minimum, median, maximum). Actions are
per-backbone medians over 15 clean runs for TiRex, Chronos-2, and TimesFM~2.5, against the full
gate's 3{,}198, 3{,}531, 4{,}219 under immediate and 4{,}095, 4{,}058, 4{,}291 under delayed
reveal. Spike is the largest survivor-only paired degradation ratio in magnitude over the five streams at 5 percent, with delayed weather averaging its 7 of 9 survivors. Packet freeze and level
shift are the ETTm1 and ETTh2 ratios at 5 percent over 9 of 9 pairs. Diverged counts runs of
450.}
\label{app:tab:cutoffq}
\scriptsize
\setlength{\tabcolsep}{3pt}
\begin{tabular}{llcccccc}
\toprule
quantile & protocol & share & actions & spike & freeze ETTm1 & level shift ETTh2 & diverged \\
\midrule
0.999 & immediate & 0.288, 0.340, 0.402 & 6{,}109, 5{,}801, 5{,}636 & 0.05 & $+13.76$ & $+10.04$ & 0 \\
0.999 & delayed & 0.286, 0.340, 0.414 & 6{,}079, 5{,}788, 5{,}602 & 0.02 & $+7.27$ & $+2.20$ & 2 \\
0.9999 & immediate & 0.213, 0.279, 0.333 & 4{,}780, 4{,}527, 4{,}172 & 0.03 & $+16.91$ & $+12.42$ & 0 \\
0.9999 & delayed & 0.213, 0.277, 0.336 & 4{,}768, 4{,}522, 4{,}165 & 0.04 & $+6.35$ & $+2.66$ & 2 \\
0.99999 & immediate & 0.166, 0.235, 0.299 & 3{,}814, 3{,}635, 3{,}254 & 0.03 & $+20.36$ & $+13.24$ & 0 \\
0.99999 & delayed & 0.166, 0.234, 0.299 & 3{,}807, 3{,}636, 3{,}249 & 0.01 & $+6.37$ & $+2.85$ & 2 \\
\bottomrule
\end{tabular}
\end{table}

\begin{table}[t]
\centering
\caption{Completed-cache reproducibility check for the immediate-reveal spike comparison. Survivor-only paired
degradation ratios at 5 percent spike contamination, pooled over three backbones and three seeds, saved
campaign records against the four primary variants rerun on the completed cache. Brackets give
surviving pairs of 9 where runs diverged.}
\label{app:tab:regen}
\scriptsize
\setlength{\tabcolsep}{4pt}
\begin{tabular}{llccc}
\toprule
stream & variant & saved & regenerated & difference \\
\midrule
ETTh2 & ungated filter & $+3.32$ & $+3.32$ & $-0.00$ \\
ETTh2 & Huber-only & $+0.15$ & $+0.15$ & $-0.00$ \\
ETTh2 & full gate & $+0.24$ & $+0.23$ & $-0.01$ \\
ETTh2 & fixed cutoff & $+0.05$ & $+0.05$ & $+0.00$ \\
ETTm1 & ungated filter & $+2.68$ & $+2.70$ & $+0.02$ \\
ETTm1 & Huber-only & $+0.63$ & $+0.63$ & $+0.00$ \\
ETTm1 & full gate & $+6.28$ & $+6.81$ & $+0.53$ \\
ETTm1 & fixed cutoff & $-0.01$ & $-0.01$ & $+0.00$ \\
weather & ungated filter & $+0.01$ & $+0.01$ & $+0.00$ \\
weather & Huber-only & $-0.03$\,[8/9] & $-0.07$ & $-0.03$ \\
weather & full gate & $+0.03$\,[3/9] & $+0.03$\,[3/9] & $+0.00$ \\
weather & fixed cutoff & $-0.00$ & $-0.00$ & $+0.00$ \\
electricity & ungated filter & $+0.04$ & $+0.04$ & $+0.00$ \\
electricity & Huber-only & $+0.01$ & $+0.01$ & $+0.00$ \\
electricity & full gate & $-0.31$ & $-0.31$ & $-0.00$ \\
electricity & fixed cutoff & $+0.00$ & $+0.00$ & $+0.00$ \\
traffic & ungated filter & $+3.73$ & $+3.73$ & $+0.00$ \\
traffic & Huber-only & $+0.11$ & $+0.10$ & $-0.00$ \\
traffic & full gate & $+0.35$ & $+0.35$ & $+0.00$ \\
traffic & fixed cutoff & $+0.00$ & $+0.00$ & $+0.00$ \\
\bottomrule
\end{tabular}
\end{table}

\begin{table}[t]
\centering
\caption{Completed-cache reproducibility check for the sustained-contamination cells. Survivor-only paired
degradation ratios at 5 percent under immediate reveal for the packet-freeze and level-shift injectors,
pooled over three backbones and three seeds, saved campaign records against the four primary
variants rerun on the completed cache. Brackets give surviving pairs of 9 where runs diverged.}
\label{app:tab:regen2}
\scriptsize
\setlength{\tabcolsep}{4pt}
\begin{tabular}{lllccc}
\toprule
stream & injector & variant & saved & regenerated & difference \\
\midrule
ETTh2 & freeze & ungated filter & $+0.22$ & $+0.22$ & $+0.00$ \\
ETTh2 & freeze & Huber-only & $+3.24$ & $+3.24$ & $+0.00$ \\
ETTh2 & freeze & full gate & $+8.40$ & $+8.29$ & $-0.10$ \\
ETTh2 & freeze & fixed cutoff & $+12.26$ & $+12.26$ & $+0.00$ \\
ETTm1 & freeze & ungated filter & $+0.33$ & $+0.26$ & $-0.07$ \\
ETTm1 & freeze & Huber-only & $+11.76$ & $+11.60$ & $-0.16$ \\
ETTm1 & freeze & full gate & $+14.17$ & $+16.98$ & $+2.82$ \\
ETTm1 & freeze & fixed cutoff & $+13.76$ & $+13.76$ & $+0.00$ \\
weather & freeze & ungated filter & $+0.00$ & $+0.00$ & $-0.00$ \\
weather & freeze & Huber-only & $+0.00$ & $+0.00$ & $+0.00$ \\
weather & freeze & full gate & $-0.02$\,[3/9] & $-0.02$\,[3/9] & $+0.00$ \\
weather & freeze & fixed cutoff & $-0.00$ & $-0.00$ & $+0.00$ \\
electricity & freeze & ungated filter & $-0.08$ & $-0.08$ & $-0.00$ \\
electricity & freeze & Huber-only & $+1.41$ & $+1.41$ & $+0.01$ \\
electricity & freeze & full gate & $+1.74$ & $+1.74$ & $-0.00$ \\
electricity & freeze & fixed cutoff & $+1.72$ & $+1.72$ & $+0.00$ \\
traffic & freeze & ungated filter & $+1.75$ & $+1.75$ & $-0.00$ \\
traffic & freeze & Huber-only & $+7.18$ & $+7.18$ & $+0.00$ \\
traffic & freeze & full gate & $+6.14$ & $+6.14$ & $-0.00$ \\
traffic & freeze & fixed cutoff & $+6.38$ & $+6.38$ & $+0.00$ \\
ETTh2 & level shift & ungated filter & $+7.79$ & $+7.79$ & $-0.00$ \\
ETTh2 & level shift & Huber-only & $+16.98$ & $+16.98$ & $-0.00$ \\
ETTh2 & level shift & full gate & $+19.29$ & $+19.03$ & $-0.26$ \\
ETTh2 & level shift & fixed cutoff & $+10.04$ & $+10.04$ & $+0.00$ \\
ETTm1 & level shift & ungated filter & $+15.21$ & $+15.02$ & $-0.19$ \\
ETTm1 & level shift & Huber-only & $+10.64$ & $+10.60$ & $-0.04$ \\
ETTm1 & level shift & full gate & $+13.41$ & $+13.50$ & $+0.09$ \\
ETTm1 & level shift & fixed cutoff & $+4.15$ & $+4.15$ & $+0.00$ \\
weather & level shift & ungated filter & $-0.00$ & $-0.00$ & $+0.00$ \\
weather & level shift & Huber-only & $-0.00$ & $-0.00$ & $+0.00$ \\
weather & level shift & full gate & $+0.04$\,[3/9] & $+0.04$\,[3/9] & $+0.00$ \\
weather & level shift & fixed cutoff & $-0.05$ & $-0.05$ & $+0.00$ \\
electricity & level shift & ungated filter & $+1.16$ & $+1.16$ & $-0.00$ \\
electricity & level shift & Huber-only & $+1.14$ & $+1.13$ & $-0.00$ \\
electricity & level shift & full gate & $+1.22$ & $+1.20$ & $-0.01$ \\
electricity & level shift & fixed cutoff & $+2.20$ & $+2.20$ & $+0.00$ \\
traffic & level shift & ungated filter & $+9.54$ & $+9.54$ & $-0.00$ \\
traffic & level shift & Huber-only & $+18.35$ & $+18.35$ & $-0.00$ \\
traffic & level shift & full gate & $+10.57$ & $+10.57$ & $-0.00$ \\
traffic & level shift & fixed cutoff & $+0.97$ & $+0.97$ & $+0.00$ \\
\bottomrule
\end{tabular}
\end{table}

\begin{table}[t]
\centering
\caption{Every cell of the immediate-reveal primary grid rerun on the completed cache, six variants,
the clean cell and the three injectors at 1, 5, and 10 percent, 300 stream, cell, and variant ratios
in all. Listed are the 20 survivor-only paired degradation ratios that move by more than 0.05
between the saved campaign and the rerun, largest move first; the other 280 agree within 0.05
and 263 within 0.01. Brackets give surviving pairs of 9 where runs diverged.}
\label{app:tab:regen-full}
\scriptsize
\setlength{\tabcolsep}{4pt}
\begin{tabular}{llclccc}
\toprule
stream & injector & fraction & variant & saved & regenerated & difference \\
\midrule
traffic & spike & 0.05 & SGD surrogate & $+3.40$\,[6/9] & $+0.51$\,[6/9] & $-2.89$ \\
ETTm1 & freeze & 0.05 & full gate & $+14.17$ & $+16.98$ & $+2.82$ \\
ETTm1 & spike & 0.1 & full gate & $+19.10$ & $+20.36$ & $+1.27$ \\
ETTm1 & spike & 0.05 & full gate & $+6.28$ & $+6.81$ & $+0.53$ \\
ETTm1 & level shift & 0.1 & ungated filter & $+30.59$ & $+30.21$ & $-0.38$ \\
ETTm1 & spike & 0.1 & ungated filter & $+19.37$ & $+19.01$ & $-0.36$ \\
ETTh2 & freeze & 0.1 & full gate & $+11.11$ & $+11.41$ & $+0.30$ \\
ETTh2 & level shift & 0.05 & full gate & $+19.29$ & $+19.03$ & $-0.26$ \\
ETTm1 & level shift & 0.05 & ungated filter & $+15.21$ & $+15.02$ & $-0.19$ \\
traffic & level shift & 0.01 & Huber-only & $+1.65$ & $+1.46$ & $-0.19$ \\
ETTh2 & level shift & 0.1 & full gate & $+14.62$ & $+14.46$ & $-0.17$ \\
ETTm1 & freeze & 0.05 & Huber-only & $+11.76$ & $+11.60$ & $-0.16$ \\
ETTm1 & freeze & 0.1 & Huber-only & $+2.90$ & $+3.05$ & $+0.15$ \\
traffic & spike & 0.1 & SGD surrogate & $+5.17$\,[6/9] & $+5.03$\,[6/9] & $-0.15$ \\
ETTm1 & freeze & 0.1 & full gate & $+4.90$ & $+5.01$ & $+0.11$ \\
ETTh2 & freeze & 0.05 & full gate & $+8.40$ & $+8.29$ & $-0.10$ \\
ETTm1 & level shift & 0.05 & full gate & $+13.41$ & $+13.50$ & $+0.09$ \\
ETTm1 & freeze & 0.05 & ungated filter & $+0.33$ & $+0.26$ & $-0.07$ \\
ETTh2 & level shift & 0.01 & full gate & $+1.20$ & $+1.15$ & $-0.05$ \\
ETTm1 & level shift & 0.01 & full gate & $+3.07$ & $+3.12$ & $+0.05$ \\
\bottomrule
\end{tabular}
\end{table}

\begin{table}[t]
\centering
\caption{Recovery-period override on the full gate's clean spike actions, ranges over three
backbones of per-cell medians over three seeds. Applied and overridden count the twin's spike
actions and the classifications its recovery periods turned into normal updates. Classified is
their sum over the cutoff's clean count at the same quantile, live the count of the twin with
the override disabled over the same. The $0.95$ weather rows divide by the cutoff at $0.999$,
the cutoff having no $0.95$ run. An n/a marks a cell where the live twin has no saved value. Immediate-reveal weather counts on TiRex and Chronos-2 cover
the 18{,}537 steps before the guard. The live twin ends at the guard in nine (0.999) and eight
(0.95) of nine immediate weather runs and in two (0.999) and three (0.95) delayed ones.}
\label{app:tab:suppress}
\scriptsize
\setlength{\tabcolsep}{3pt}
\begin{tabular}{llcccccc}
\toprule
protocol & stream & $q$ & applied & overridden & overridden/applied & classified/cutoff & live/cutoff \\
\midrule
immediate & ETTh2 & 0.999 & 3,190--3,531 & 321--448 & 0.09--0.14 & 0.82--0.91 & n/a \\
immediate & ETTm1 & 0.999 & 10,981--11,284 & 1,587--1,632 & 0.14--0.14 & 0.60--0.66 & n/a \\
immediate & weather & 0.999 & 2,912--4,437 & 981--1,144 & 0.26--0.34 & 0.25--0.38 & 0.26--0.30 \\
immediate & weather & 0.95 & 3,684--6,196 & 1,244--1,777 & 0.29--0.38 & 0.33--0.54 & 0.34--0.37 \\
immediate & electricity & 0.999 & 4,219--4,890 & 263--435 & 0.05--0.09 & 0.82--0.89 & n/a \\
immediate & traffic & 0.999 & 3,159--3,202 & 388--458 & 0.12--0.14 & 0.82--0.84 & n/a \\
delayed & ETTh2 & 0.999 & 3,203--3,497 & 334--446 & 0.10--0.14 & 0.83--0.91 & n/a \\
delayed & ETTm1 & 0.999 & 10,960--11,374 & 1,519--1,647 & 0.13--0.15 & 0.63--0.68 & n/a \\
delayed & weather & 0.999 & 4,058--4,470 & 1,167--1,274 & 0.26--0.31 & 0.34--0.38 & 0.57--0.62 \\
delayed & weather & 0.95 & 5,420--6,178 & 1,698--1,794 & 0.28--0.33 & 0.47--0.53 & 0.34--0.82 \\
delayed & electricity & 0.999 & 4,291--4,864 & 265--427 & 0.05--0.09 & 0.83--0.89 & n/a \\
delayed & traffic & 0.999 & 3,139--3,208 & 386--404 & 0.12--0.13 & 0.82--0.85 & n/a \\
\bottomrule
\end{tabular}
\end{table}

\begin{table}[t]
\centering
\caption{Symmetric comparison at matched spike-action budgets, on spike contamination at 5 percent, both variants
rerun across spike quantiles, the cutoff up to $0.9999999$ and the full gate down to $0.95$.
The deployed columns use completed-cache records. Appendix~\ref{app:repro} audits their saved
campaign counterparts; every matched column is unchanged to three decimals.
Deployed columns are the survivor-only paired degradation ratios at $q = 0.999$ pooled over
three backbones and three seeds. Matched columns read each variant, per backbone, at the available
quantile nearest the other variant's deployed clean spike-action count in log ratio, listing the
quantiles chosen across backbones and the range of achieved count over target. Brackets give
surviving pairs of 9 where runs diverged. Weather is not budget-matched in either direction. Its
immediate gate counts on TiRex and Chronos-2 cover 18{,}537 of 39{,}378 steps.}
\label{app:tab:budget}
\scriptsize
\setlength{\tabcolsep}{2.5pt}
\begin{tabular}{llcccccccc}
\toprule
& & \multicolumn{2}{c}{deployed, $q=0.999$} & \multicolumn{3}{c}{cutoff at the gate's budget} & \multicolumn{3}{c}{gate at the cutoff's budget} \\
\cmidrule(lr){3-4}\cmidrule(lr){5-7}\cmidrule(lr){8-10}
protocol & stream & cutoff & gate & $q$ & count/target & ratio & $q$ & count/target & ratio \\
\midrule
immediate & ETTh2 & $+0.05$ & $+0.23$ & 0.9999 & 0.90--1.07 & $+0.03$ & 0.98--0.995 & 0.94--1.02 & $+0.16$ \\
immediate & ETTm1 & $-0.01$ & $+6.81$ & 0.999999 & 0.94--0.96 & $+0.03$ & 0.95--0.98 & 0.95--1.07 & $+3.94$ \\
immediate & weather & $-0.00$ & $+0.03$\,[3/9] & 0.9999999 & 1.98--2.99 & $-0.00$\,[8/9] & 0.95 & 0.24--0.42 & $-0.00$\,[3/9] \\
immediate & electricity & $+0.00$ & $-0.31$ & 0.9999 & 0.93--1.03 & $+0.00$ & 0.99--0.995 & 0.96--1.04 & $-0.32$ \\
immediate & traffic & $+0.00$ & $+0.35$ & 0.9999--0.99999 & 0.96--1.08 & $+0.00$ & 0.98 & 0.98--1.03 & $+0.35$ \\
delayed & ETTh2 & $-0.01$ & $+0.18$ & 0.9999 & 0.90--1.06 & $+0.01$ & 0.95--0.995 & 0.98--1.02 & $+0.19$ \\
delayed & ETTm1 & $+0.01$ & $+3.18$ & 0.999999 & 0.94--0.95 & $-0.00$ & 0.95--0.98 & 0.96--1.08 & $+2.73$ \\
delayed & weather & $+0.02$\,[7/9] & $-0.20$\,[7/9] & 0.9999999 & 1.97--2.12 & $+0.03$\,[7/9] & 0.95 & 0.35--0.42 & $-0.14$\,[7/9] \\
delayed & electricity & $-0.00$ & $+0.54$ & 0.9999 & 0.93--1.04 & $-0.00$ & 0.99--0.995 & 1.02--1.06 & $+0.48$ \\
delayed & traffic & $-0.00$ & $+0.43$ & 0.9999--0.99999 & 0.96--1.09 & $-0.00$ & 0.98 & 1.00--1.02 & $+0.48$ \\
\bottomrule
\end{tabular}
\end{table}

Nominal Ljung--Box diagnostics on the first-step innovation column~\citep{ljungbox1978} reject the null
of uncorrelated residuals at $p<0.01$ in 119 of the 135 clean runs of Section~\ref{sec:premise}.
The portmanteau over-rejects under conditional heteroskedasticity, so these innovation-level
diagnostics are inconclusive and the raw fire counts carry the argument there; the Ljung-Box
tests of Sections~\ref{sec:isolation} and~\ref{sec:e6} run on the conformal $p$-streams, a
different object. The rejection is decisive across
the full-gate clean runs on three of the five streams. The largest $p$-value within each cell's runs is below $10^{-8}$ on ETTh2, below
$10^{-15}$ on ETTm1, and below $10^{-10}$ on traffic. Electricity's largest within-cell
$p$-value is $0.25$. The weather cell is the remaining exception. Its full-gate runs report a
Ljung-Box statistic with $p \approx 1$. This is a lower-tail degeneracy. One extreme
innovation drives every sample autocorrelation toward zero, inflating the portmanteau $p$-value
without indicating whiteness, and those degenerate-diagnostic weather runs still fire 115 to 277
times each.

The frozen-scale variant of Section~\ref{sec:frozenr} extends this census. Under immediate reveal on
TiRex it fires 30 to 33 times on ETTh2, 238 to 250 times on ETTm1, 130 to 131 times on weather,
48 to 50 times on electricity, and 36 to 40 times on traffic, where the full gate on the same
backbone fires 32 to 35, 160 to 167, 115 to 120, 51 to 56, and 31 to 33 times. Its median clean run fires 49 times under immediate reveal and 48 under
delayed reveal, against the full gate's 42 and 41 on the same 45 runs, and the full gate's own
per-run counts span 18 to 277 under immediate reveal and 21 to 299 under delayed reveal. A nominal Ljung--Box diagnostic on the
first-step innovation column rejects at $p<0.01$ in 35 of its 45 immediate runs and in 45 of 45
of its delayed runs. Of the 10 immediate runs that do not reject, 9 are the weather runs, which
report $p \approx 1$ under the same lower-tail degeneracy and still fire 122 to 302 times each,
and one is an electricity run at $p = 0.027$.

Under delayed reveal the full gate's clean-stream fires do not fall, so the
firing is not an artifact of immediate reveal alone. Totalled over five streams and three seeds per backbone, fires rise from 1{,}186 to 1{,}726 on TiRex, from 1{,}083
to 1{,}495 on Chronos-2, and from 1{,}610 to 1{,}649 on TimesFM~2.5. Part of the rise is
exposure. The immediate-reveal weather runs on TiRex and Chronos-2 are truncated by clean-stream
divergence at step 18{,}537 of 39{,}378, so their delayed counterparts run 2.1 times as long.
The per-step fire rate still does not fall on any backbone. On weather it moves from
$6.3\times10^{-3}$ to $7.4\times10^{-3}$ on TiRex, from $6.67\times10^{-3}$ to
$6.70\times10^{-3}$ on Chronos-2, and from $6.8\times10^{-3}$ to $7.1\times10^{-3}$ on
TimesFM~2.5. Consecutive origins still overlap in 95 of their 96 entries under delayed reveal,
so the change of protocol removes none of the overlap-induced dependence. Delayed first fires do arrive earlier
on two streams, through a start-up transient. Under delay 96 the first 96 test steps are
predict-only and divide the filter's covariance by the forgetting factor on each, so the first
scored steps come from an inflated-covariance state. On the saved score streams of
Appendix~\ref{app:valid} the classical no-reset wealth process crosses at a median scored index of 8 on
weather and 12 on ETTh2 over the 9 runs of each stream, against 877 and 470 under immediate
reveal, with 9 runs per stream and protocol. Electricity crosses at index 42 under both protocols and traffic at 967 against 977. Figure~\ref{fig:sweep} extends the
census to every variant. The premise failure holds across the configurations tested.

\begin{figure}[t]
\centering
\includegraphics[width=0.9\linewidth]{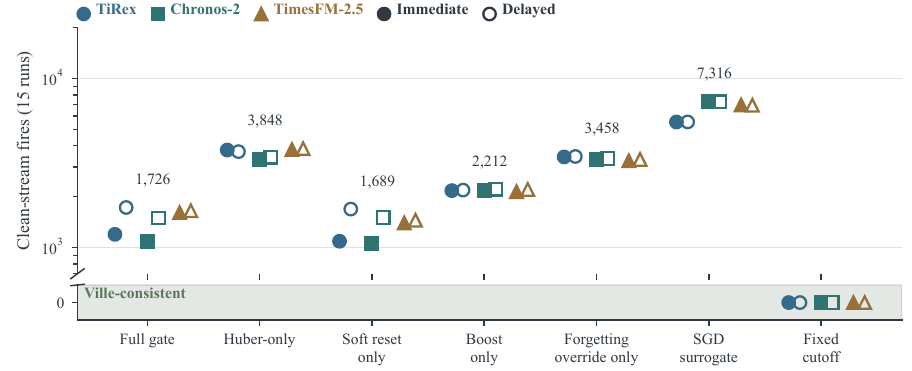}
\caption{Clean-stream fires per variant, totalled over five streams and three seeds per
backbone, immediate reveal filled and delayed reveal hollow. Every variant in this figure that carries the
martingale fires more than a thousand times on every backbone under every protocol tested; the
detector-stride-96 gate, which also carries it, fires 68 times over its 90 clean runs and is not shown. The
fixed-cutoff variant carries no martingale, so it has no fire event to count and its count is zero under both protocols by construction.}
\label{fig:sweep}
\end{figure}

\FloatBarrier
\section{Classical no-reset replay and origin-overlap analyses}
\label{app:valid}

Two shadow variants rerun the deployed loop on the real streams and save every score,
one with all three hyperparameters learning and one with the observation-noise scale held at its
warm-up value. Both run in shadow, so the gate observes every score and acts on none. Across three backbones, five streams, three seeds, and two protocols the bundle holds 180 clean runs
and 4{,}932{,}540 steps, of which 4{,}923{,}900 are scored, each of the 90 delayed-reveal runs
spending its first 96 steps predict-only. An offline replay recomputes the deployed $p$-values and the fire
sequence from those raw scores and reproduces both exactly in 180 of 180 runs, so every number in
Table~\ref{app:tab:valid} is measured on the score sequences the deployed gate itself saw. The
classical no-reset wealth process keeps every score in its calibration set and uses the wealth of
Algorithm~\ref{alg:gate} with no reset, at $\alpha=0.05$, so its log-wealth threshold is
$\log 20$. If a cell's score sequences were exchangeable, Ville's inequality would limit the
expected number of crossing runs to 2.25 per 45-run cell. No such premise is assumed for these
observed streams.
Every cell crosses in 45 of 45. The deployed gate's first fire on the same streams has a median of 515 steps under immediate reveal and 138 under delayed with the scale learned, and 535
and 139 with it held, over the 45 runs of each variant and protocol. The classical no-reset
wealth process's pooled first crossing has a median of 528 and 138 steps with the scale learned and
776 and 139 with it held, so
the medians coincide under delayed reveal and differ by 13 steps under immediate reveal with the
scale learned.

The level and the calibration length are the two settings the validity argument turns on, so both
were swept on these same saved streams. The shadow
variants apply no response, so the score stream does not depend on what the gate decided, and a gate configured differently and fed these scores would have seen exactly them. Over the 90
clean runs with the scale learned, pooling both protocols, the deployed gate fires in 90
of 90 at every calibration length of 100, 250, and 500 and at every level from $0.05$ down to
$10^{-6}$. At $10^{-8}$ it still fires in 90 of 90 at lengths 100 and 250 and in 87 of 90 at 500.
The classical no-reset wealth process crosses in 90 of 90 down to $10^{-6}$ and in 80 of 90 at $10^{-8}$.
The clean-stream firing is therefore not a property of the level or the window length this campaign chose, since a threshold of $10^{8}$ leaves it in place. The swept lengths span the range the
adapter's configuration admits, and the growing calibration set of the no-reset variant is the
unbounded end of that same axis. In an exploratory, post hoc analysis, the betting family was swept on the same 90 streams at level
$0.05$ and length 250. The deployed mixture
over the epsilon grid crosses in 90 of 90 and fires in 90 of 90, and a single fixed bet of $0.9$
does the same. A fixed bet of $0.5$ crosses in 43 and fires in 53 of 90, and a fixed bet of $0.1$
crosses in 5 and fires in 5 of 90, near the nominal reference of 4.5. The betting family is therefore
not immaterial. The exponent 0.1, the one that stays at its expected count, bets most on very
small $p$-values and loses wealth on every $p$-value above 0.077, so on these streams, whose
$p$-values are small but rarely that small, its wealth collapses; at a pseudo-onset placed at 60
percent of each run its log-wealth sits 10{,}642 to 43{,}491 below the threshold of $\log 20$,
a deficit that 960 steps cannot recover, since a
calibration set of several thousand scores caps each step's gain near 7. This deficit makes
recovery within that window implausible, but it is not a power evaluation. The exponent 0.9,
whose increment stays between 0.96 and 1.43 for $p$ between
0.01 and 0.5, is the closest to not betting and crosses everywhere. No exponent's power against a
change was measured. On the 180 shadow runs, nominal Kolmogorov--Smirnov marginal-uniformity and
Ljung--Box serial-dependence diagnostics both reject at level 0.05 for every deployed and
classical no-reset $p$-stream. The median nominal Kolmogorov--Smirnov $p$-values are
$3\times10^{-31}$ and $2\times10^{-8}$. Because observations within a real stream are dependent,
these values are descriptive diagnostics rather than calibrated inferential $p$-values.

\begin{table}[t]
\centering
\caption{Classical no-reset wealth process replayed on the saved real score streams, per stream and protocol. Each cell pools three backbones and three seeds, so every median and range
covers 9 runs. First crossing is the test step at which the wealth first reaches
$\nicefrac{1}{\alpha}$, and $\max \log M$ is the peak log-wealth over the run. The left pair of
columns is the variant with all three hyperparameters learning, the right pair the variant with the
observation-noise scale held at its warm-up value. All 90 runs in each pair cross.}
\label{app:tab:valid}
\scriptsize
\setlength{\tabcolsep}{2.5pt}
\begin{tabular}{llrrrr}
\toprule
& & \multicolumn{2}{c}{scale learned} & \multicolumn{2}{c}{scale held} \\
\cmidrule(lr){3-4}\cmidrule(lr){5-6}
stream & protocol & first crossing & $\max \log M$ & first crossing & $\max \log M$ \\
\midrule
ETTh2 & immediate & 470 (469 to 470) & 19.5 (16.6 to 57.3) & 667 (660 to 672) & 123.0 (77.1 to 169.3) \\
ETTm1 & immediate & 528 (212 to 539) & 56.0 (18.2 to 56.6) & 2{,}634 (230 to 2{,}705) & 121.1 (52.8 to 206.1) \\
weather & immediate & 877 (776 to 881) & 279.9 (278.2 to 301.8) & 877 (776 to 881) & 327.9 (314.5 to 348.6) \\
electricity & immediate & 42 (11 to 44) & 47.6 (28.9 to 59.6) & 43 (11 to 46) & 234.6 (88.7 to 443.2) \\
traffic & immediate & 977 (965 to 1{,}503) & 22.7 (15.5 to 23.2) & 1{,}542 (1{,}486 to 1{,}559) & 154.4 (112.3 to 203.4) \\
\midrule
ETTh2 & delayed 96 & 108 (108 to 392) & 39.0 (24.4 to 67.1) & 108 (108 to 561) & 166.6 (124.3 to 201.4) \\
ETTm1 & delayed 96 & 615 (121 to 625) & 61.2 (46.7 to 63.4) & 654 (121 to 2{,}777) & 132.2 (72.0 to 221.0) \\
weather & delayed 96 & 104 (100 to 113) & 348.0 (325.5 to 362.1) & 104 (100 to 113) & 398.3 (362.3 to 412.7) \\
electricity & delayed 96 & 138 (107 to 141) & 49.9 (33.3 to 65.3) & 139 (107 to 143) & 239.0 (93.1 to 454.6) \\
traffic & delayed 96 & 1{,}063 (1{,}055 to 1{,}585) & 23.7 (16.6 to 24.3) & 1{,}631 (1{,}245 to 1{,}634) & 142.2 (111.2 to 203.5) \\
\bottomrule
\end{tabular}
\end{table}

Two further score streams were replayed through the same construction, each reproducing its saved fires exactly in 90 of 90 runs. The first is the SGD adapter's shadow stream, the surrogate score of Appendix~\ref{app:spec} saved
with the gate observing and never acting, so that the learner path is the ungated SGD
adapter's. It crosses in 45 of 45 runs under each protocol, 9 of 9 on every stream, at median
first crossings of 293 under immediate reveal (198 on ETTm1 to 2{,}582 on electricity) and 389
under delayed reveal (294 to 2{,}678). A nominal Ljung--Box diagnostic rejects its no-reset $p$-value stream at
$p<0.01$ in 45 of 45. Thinned offline at stride 96 it crosses in 0.008 of 4{,}320 pairs under
either protocol against 0.651 for consecutive blocks of the same count, and 0.144, 0.357, 0.670,
and 0.724 for bursts of 2, 4, 8, and 16 under immediate reveal (0.142, 0.401, 0.649, and 0.726
delayed). At strides 16 and 32 it crosses in 0.354 and 0.126, above the Kalman stream's 0.230
at stride 16 and below its 0.193 at stride 32, with ETTm1 at 0.993 at stride 16. The live detector-stride-96 gate of Section~\ref{sec:cutoff} saves the stream its wealth process saw, 135 to 543 scores per run, median 204. On that stream the classical no-reset wealth process crosses in 12 of 45 immediate and 10 of 45 delayed runs, 0 of 9 on ETTh2, electricity, and
traffic under either protocol, 3 and 1 of 9 on ETTm1, and 9 of 9 on weather, with a median peak log-wealth of 0.30 and 0.33 against the threshold $\log 20 = 3.00$. The weather crossings
under immediate reveal fall at steps 18{,}432 to 20{,}064, within one stride of the transient at 18{,}536. Under delayed reveal they fall at 1{,}248 to 1{,}344 on TiRex and 12{,}384 to
12{,}480 on the other two backbones. The same gated learner's every-origin stream crosses in 45
of 45 under each protocol. Counted per scored origin, the live detector-stride-96 gate fires 26 times over
12{,}861 scores under immediate reveal and 42 over 12{,}816 under delayed, 2.0 and 3.3 per
thousand, against the full gate's 3{,}890 over 1{,}108{,}089 and 4{,}870 over 1{,}228{,}815, 3.5
and 4.0 per thousand. Per calendar step the two read 0.021 against 3.5 per thousand under
immediate reveal, so the fall in fire count is mostly the 96-fold fall in exposure, with a
reduction per score of 42 and 17 percent.

The feedback-stride-96 protocol thins the adapter's feedback as well as the martingale's. Every
origin is forecast and evaluated for MSE, but only every 96th delayed packet produces a gate score
and adapter update. A run of $T$ test steps therefore supplies $\lfloor (T-96)/96 \rfloor + 1$
scores to the martingale, 134 to 542 per run
(Table~\ref{app:tab:stride}, the per-cell record). Clean-stream firing falls from 45 of 45 runs,
median 41 fires per run under delayed reveal, to 7 of 36 non-diverged runs, one fire each, in
three cells. Two confounds move with the stride. The filter takes 95 predict-only steps between
reveals, and 3 of 15 clean cells then cross the divergence guard for every variant, the ungated
filter included; those cells are left with 9, 7, and 21 scored steps, and their nine gated clean
runs are excluded from the 36 non-diverged runs, because a run that sees at most 21 scores
carries no evidence for the martingale. The martingale also sees 96 times fewer scores. The Huberized variants
coincide in every zero-fire cell, so the feedback-stride-96 spike gains are Huberization alone.

\begin{table}[t]
\centering
\caption{Non-overlapping-origin protocol, per backbone and stream, clean cells. Scored steps per run is what the martingale sees under stride 96. The divergence column gives the guard step where a
cell truncates, identically for every variant including the ungated filter. The last two columns give
the full gate's fire count for seeds 0, 1, and 2 under stride 96 and under delayed reveal at
stride 1 on the same cell. The stride-1 column pools to 45 of 45 runs firing at a median of 41
fires per run over those 45 runs.}
\label{app:tab:stride}
\footnotesize
\setlength{\tabcolsep}{5pt}
\begin{tabular}{llrrrr}
\toprule
backbone & stream & scored steps & divergence & fires, stride 96 & fires, stride 1 \\
\midrule
TiRex & ETTh2 & 134 & & 1, 1, 1 & 37, 37, 36 \\
Chronos-2 & ETTh2 & 134 & & 0, 0, 0 & 22, 22, 21 \\
TimesFM~2.5 & ETTh2 & 134 & & 1, 1, 1 & 37, 34, 35 \\
TiRex & ETTm1 & 9 & 942 & 0, 0, 0 & 175, 143, 172 \\
Chronos-2 & ETTm1 & 542 & & 0, 0, 0 & 156, 142, 149 \\
TimesFM~2.5 & ETTm1 & 542 & & 0, 0, 0 & 169, 170, 171 \\
TiRex & weather & 7 & 747 & 0, 0, 0 & 299, 285, 295 \\
Chronos-2 & weather & 21 & 2{,}073 & 0, 0, 0 & 289, 264, 238 \\
TimesFM~2.5 & weather & 410 & & 0, 0, 0 & 277, 285, 277 \\
TiRex & electricity & 204 & & 0, 0, 0 & 52, 57, 42 \\
Chronos-2 & electricity & 204 & & 0, 0, 0 & 37, 38, 38 \\
TimesFM~2.5 & electricity & 204 & & 0, 0, 0 & 39, 43, 41 \\
TiRex & traffic & 135 & & 0, 0, 0 & 31, 32, 33 \\
Chronos-2 & traffic & 135 & & 0, 0, 0 & 28, 23, 28 \\
TimesFM~2.5 & traffic & 135 & & 0, 1, 0 & 25, 23, 23 \\
\bottomrule
\end{tabular}
\end{table}

Table~\ref{app:tab:thin} is the offline thinning of the saved shadow scores. For each stride
$k$ and each offset $o$ in $0$ to $k-1$, the run's scored steps are thinned to every $k$th
starting at $o$, and both constructions replay on that stream with the run's own window size and
$\alpha$. A consecutive block of the same length starting at the same offset replays alongside,
with the same random seed, so each thinned slice is compared with equally many overlapping scores
from the same run. Stride 1 at offset 0 reproduces the saved deployed fires exactly. The unit
is the run-offset pair, 45 runs times $k$ offsets per cell, and offsets of one run share its
stream, so the fractions are descriptive, and this table is the record under tie-break stream 0
(Table~\ref{app:tab:burst} averages three streams). At stride 96 with the scale learned, the
four streams other than weather cross in 0.005 to 0.014 of pairs under either protocol, and weather in 0.132
under immediate reveal and 0.191 under delayed reveal. With the scale held, electricity crosses
in 0.152 and 0.154, weather in 0.167 and 0.234, and the other three streams in 0.035 to 0.059.

\begin{table}[t]
\centering
\caption{Fraction of run-offset pairs in which the classical no-reset wealth process crosses
$\nicefrac{1}{\alpha}$ on the thinned score stream, against the same fraction for a consecutive
block of the same length from the same run, pooled over three backbones, five streams, and three
seeds, under tie-break stream 0. Under the premise, Ville's inequality bounds the crossing probability of one run chosen in advance by 0.05. The pairs are correlated, so the fractions are descriptive. The pairs column counts the 45 runs times $k$
offsets in each cell.}
\label{app:tab:thin}
\footnotesize
\setlength{\tabcolsep}{4pt}
\begin{tabular}{rccccr}
\toprule
& \multicolumn{2}{c}{scale learned} & \multicolumn{2}{c}{scale held} & \\
\cmidrule(lr){2-3}\cmidrule(lr){4-5}
stride & immediate & delayed 96 & immediate & delayed 96 & pairs \\
\midrule
1 & 1.000 vs 1.000 & 1.000 vs 1.000 & 1.000 vs 1.000 & 1.000 vs 1.000 & 45 \\
2 & 1.000 vs 1.000 & 1.000 vs 1.000 & 1.000 vs 1.000 & 1.000 vs 1.000 & 90 \\
4 & 0.800 vs 1.000 & 0.933 vs 1.000 & 1.000 vs 1.000 & 1.000 vs 1.000 & 180 \\
8 & 0.489 vs 1.000 & 0.597 vs 1.000 & 0.981 vs 1.000 & 1.000 vs 1.000 & 360 \\
16 & 0.232 vs 0.825 & 0.268 vs 0.849 & 0.817 vs 0.832 & 0.896 vs 0.850 & 720 \\
32 & 0.194 vs 0.719 & 0.203 vs 0.801 & 0.511 vs 0.682 & 0.548 vs 0.745 & 1{,}440 \\
48 & 0.124 vs 0.728 & 0.145 vs 0.795 & 0.302 vs 0.671 & 0.325 vs 0.746 & 2{,}160 \\
96 & 0.032 vs 0.631 & 0.044 vs 0.673 & 0.087 vs 0.615 & 0.106 vs 0.659 & 4{,}320 \\
\bottomrule
\end{tabular}
\end{table}

Table~\ref{app:tab:burst} holds the score count fixed and varies only the clustering, answering
whether the contrast in Table~\ref{app:tab:thin} could come from the early placement of the
consecutive block. For every offset of each stride and three tie-break streams, the thinned
count $n$ is taken as bursts of $b$ consecutive scored origins spread evenly over the whole run,
with $b=1$ the thinned slice and $b \ge n$ the early block, and a consecutive block of $n$ scores
is replayed at 8 start positions spread from the beginning to the end of the run. At stride 96
the crossing fraction rises with $b$ on every variant and protocol from $b=1$ to $b=16$ and falls at
$b=32$ on three of the four cells, and consecutive blocks placed late in the run cross as often
as early ones. The across-stream standard deviation is at most 0.015 in every cell and 0.006 at
stride 96 with the scale learned, and the $b=1$ column reproduces the all-offset rows of
Table~\ref{app:tab:thin} to within 0.007, the largest gap 0.268 against 0.262 with the scale
learned at stride 16 under delayed reveal. Taking each of the 15 backbone-stream pairs as one
unit of 288 pairs, bursts of 16 cross more often than single origins in 15 of 15 units under
every variant and protocol, by at least 0.41, and leaving any one stream out keeps the pooled
fraction at $b=16$ between 0.655 and 0.751 and at $b=1$ between 0.007 and 0.127, the low end
when weather is left out. Bursts change which origins are selected along with their clustering,
and each start position draws its own random ranks, so the analysis is observational.

\begin{table}[t]
\centering
\caption{Burst dose-response and temporal-coverage analysis on the saved shadow scores, fraction
of run-offset pairs in which the classical no-reset wealth process crosses, 45 runs times every offset of
the stride per cell, 720, 1{,}440, and 4{,}320 pairs at strides 16, 32, and 96, pooled over three
backbones and five streams and averaged over three tie-break streams. Bursts of $b$ consecutive
origins carry the thinned slice's own score count and are spread over the whole run. The last
column gives the range over 8 start positions and three streams of a consecutive block of that
count. Fractions are descriptive.}
\label{app:tab:burst}
\footnotesize
\setlength{\tabcolsep}{3.5pt}
\begin{tabular}{llrrrrrrrc}
\toprule
scale & stride & $b=1$ & $b=2$ & $b=4$ & $b=8$ & $b=16$ & $b=32$ & early block & blocks by start \\
\midrule
learned, immediate & 16 & 0.230 & 0.338 & 0.473 & 0.596 & 0.749 & 0.883 & 0.827 & 0.60 to 1.00 \\
learned, immediate & 32 & 0.193 & 0.249 & 0.316 & 0.523 & 0.712 & 0.690 & 0.721 & 0.60 to 1.00 \\
learned, immediate & 96 & 0.032 & 0.164 & 0.372 & 0.544 & 0.699 & 0.589 & 0.632 & 0.36 to 0.80 \\
learned, delayed 96 & 16 & 0.262 & 0.417 & 0.568 & 0.637 & 0.882 & 0.974 & 0.849 & 0.73 to 1.00 \\
learned, delayed 96 & 32 & 0.203 & 0.297 & 0.383 & 0.745 & 0.923 & 0.799 & 0.800 & 0.80 to 1.00 \\
learned, delayed 96 & 96 & 0.044 & 0.185 & 0.359 & 0.652 & 0.707 & 0.774 & 0.674 & 0.42 to 1.00 \\
held, immediate & 16 & 0.821 & 0.868 & 0.838 & 0.812 & 0.800 & 0.918 & 0.831 & 0.60 to 1.00 \\
held, immediate & 32 & 0.514 & 0.547 & 0.543 & 0.548 & 0.756 & 0.782 & 0.684 & 0.47 to 1.00 \\
held, immediate & 96 & 0.089 & 0.242 & 0.443 & 0.571 & 0.702 & 0.675 & 0.617 & 0.27 to 0.96 \\
held, delayed 96 & 16 & 0.896 & 0.931 & 0.877 & 0.879 & 0.850 & 0.946 & 0.850 & 0.80 to 1.00 \\
held, delayed 96 & 32 & 0.546 & 0.581 & 0.627 & 0.738 & 0.881 & 0.907 & 0.744 & 0.67 to 1.00 \\
held, delayed 96 & 96 & 0.110 & 0.256 & 0.436 & 0.657 & 0.720 & 0.714 & 0.661 & 0.53 to 0.93 \\
\bottomrule
\end{tabular}
\end{table}

Figure~\ref{fig:acf} reports the real score streams' autocorrelation by lag against the
oracle reuse prediction of Proposition~\ref{prop:reuse}. A moving sum
sharing $96-k$ of its 96 terms has autocorrelation $(96-k)/96$ at lag $k$, zero at and beyond
the horizon. The saved scores of the learned-scale shadow variant under immediate reveal fall from 0.95 at lag 1
to 0.48, 0.26, and 0.12 at lags 24, 48, and 72, pooled over 45 runs. They reach zero between lag
66 on ETTh2 and lag 116 on weather, at 89 pooled, faster than the oracle's linear decline on
three of the five streams. Beyond the horizon they stay within $-0.17$ and $+0.05$, apart from
electricity's weekly peak of $0.29$ at lag 168, over 9 runs per stream. The
pooled decline over lags 90 to 95 is four times that over lags 97 to 102, 0.008 against 0.002
per lag, and a geometric decay matched at lag 1 would give 0.29 and 0.085 at lags 24 and 48.
The streams with a positive floor just beyond the horizon, weather at 0.05 and traffic at 0.03
with the scale learned and electricity at 0.05 to 0.11 with it held, are the residual streams of
Table~\ref{app:tab:thin}. The delayed-reveal streams carry the same labels shifted and give the
same curves to two decimals.

Table~\ref{tab:overlap} reports the reuse intervention of Section~\ref{sec:e6}. Conditional on a timestamp-indexed mean path $g$, both settings share the forecast stream and the marginal law of the label at origin $i$ and horizon $j$,
\begin{equation}
f_i[j] = g(i+j+1),\qquad y_i = f_i + \nu_i,\qquad \nu_i \sim \mathcal{N}(0, \sigma^2 I_{96}),\qquad \sigma = 0.05.
\label{eq:condnull}
\end{equation}
The fresh-label setting draws $\nu_i$ independently at every origin. The reused-label setting uses $\nu_i[j] = \eta_{i+j}$ from one white path $\eta$, so $y_i[j] = y_{i+1}[j-1]$ exactly. Each seed
draws the initial state, the pattern phase, the warm-up noise, $\eta$, and the independent noise
in that order, and one call returns both settings, so the two share everything but the reuse. The oracle pair scores at the truth, $e_i = y_i - f_i$ and $S = \sigma^2 I$ in Algorithm~\ref{alg:gate}, so fresh-label scores are i.i.d.\ $\chi^2_{96}$. The first randomized
$p$-value of an empty calibration set is 0.5 by construction, which only makes the wealth
conservative at that step. Ville's bound requires tie-break randomness independent of the labels.
The finite implementation uses three deterministic pseudorandom streams. Stream 0 reuses the run
seed, and streams 1 and 2 use offset seeds. All three agree in every variant. This agreement checks
reproducibility only and does not prove probabilistic independence. The learning pair runs the deployed stepper in shadow with all three
hyperparameters learning, a 2{,}000-step warm-up, and the noise scale initialized at $\sigma^2$,
the two transient remedies of Section~\ref{sec:isolation} applied together. Each variant runs 60
runs of 20{,}000 test steps. The offline replay reproduces the learning variants' deployed
$p$-values and fire sequences exactly in 60 of 60 runs per variant, and the oracle variants' deployed
values come from that replay.

The classical no-reset wealth process is replayed under three tie-break streams, and each of the
96 stride-96 offsets, 208 or 209 scores each, is matched against a consecutive block of the same
count from the same start step under the same seed. That block sits at the start of the run, so
for the learning pair the contrast may carry a time-location component. The oracle scores are
stationary, so there it cannot. For the learning pair thinning removes the direct label reuse but not dependence carried through the adapter state, so its thinned result is an
empirical crossing-rate restoration with no bound behind it. The oracle scorer holds the truth
and not the adapter's warm-up state, because a scorer frozen at warm-up keeps a covariance that
is not small, so $H_i P H_i^\top$ swings with the pattern and modulates the score scale at the
pattern's period even under independent labels, which would confound the reuse contrast. Under the
independent variants the deployed $p$-stream is uniform, median Kolmogorov-Smirnov~\citep{massey1951ks} $p$ of 1.00 and
0.999, and Ljung-Box rejects at $p<0.01$ in 0.03 of oracle runs and 0.48 of learning runs, so
the learning adapter leaves detectable serial dependence without crossings. Under the reused
variants it rejects in 60 of 60 with a lag-one autocorrelation of 0.98.

The oracle pair admits a closed form. With $\xi_t = \eta_t^2/\sigma^2$, the reused-label score at origin $i$ is
\begin{equation}
s_i = \sum_{j=0}^{m-1}\xi_{i+j},
\label{eq:movingsum}
\end{equation}
a moving sum of $m$ terms sharing $m-1$ with its successor, and the fresh-label score is the same sum over a fresh row.

\begin{proof}[Proof of Proposition~\ref{prop:reuse}, stated in Section~\ref{sec:e6}]
The fresh-label scores are one function of disjoint rows of an i.i.d.\ array, hence
i.i.d. For (ii), the $\xi_t$ are i.i.d.\ with variance $v \in (0, \infty)$; $s_i$ and
$s_{i+k}$ share the $m-k$ terms $\xi_{i+k}, \ldots, \xi_{i+m-1}$ and are otherwise sums of
disjoint terms, so $\operatorname{cov}(s_i, s_{i+k}) = (m-k)v$ for $k < m$ and $0$ for
$k \ge m$, while $\operatorname{var}(s_i) = mv$. For (i), exchangeability would give
$(s_1, s_2)$ and $(s_1, s_{m+1})$ the same joint law. The pair $(s_1, s_{m+1})$ is independent,
being functions of the disjoint blocks $\xi_1, \ldots, \xi_m$ and $\xi_{m+1}, \ldots, \xi_{2m}$.
Write $s_1 = A + B$ and $s_2 = B + C$ with $A = \xi_1$, $B = \sum_{t=2}^{m}\xi_t$, and
$C = \xi_{m+1}$ independent. Were $s_1$ and $s_2$ independent, their characteristic functions
would satisfy $\phi_A(u)\,\phi_B(u+w)\,\phi_C(w) = \phi_A(u)\,\phi_B(u)\,\phi_B(w)\,\phi_C(w)$
for all $u, w$. On a neighbourhood of the origin where $\phi_A$ and $\phi_C$ do not vanish
this gives $\phi_B(u+w) = \phi_B(u)\phi_B(w)$, and $w = -u$ yields $|\phi_B(u)| = 1$ on that
neighbourhood. Modulus one at $u$ makes $e^{iuB}$ almost surely constant, so $B$ lies almost
surely in a lattice of spacing $2\pi/u$. Two incommensurable $u$ in the neighbourhood force $B$
to be a point mass. But $B$ is a sum of $m-1 \ge 1$ i.i.d.\ copies of $\xi_1$, and
$|\phi_B| = |\phi_{\xi_1}|^{m-1} = 1$ would make $\xi_1$ itself degenerate, against the
assumption. So $(s_1, s_2)$ is dependent and the sequence is not exchangeable.
\end{proof}

The saved oracle scores agree with (ii). The median lag-one autocorrelation over 60 runs is
0.989, 0.958, and 0.995 at $m = 96$, 24, and 192 against $(m-1)/m = 0.990$, 0.958, and 0.995,
and at lag $m$ it is $-0.003$, $+0.006$, and $-0.004$. The Student-t family has no fourth
moment, so (ii) does not apply there and its 0.990 and $-0.011$ are descriptive. Under the
Gaussian AR(1) path the terms are dependent with finite covariances, $\operatorname{cov}(s_1, s_2) > 0$ while
$\operatorname{cov}(s_1, s_{1+k}) \to 0$ as $k \to \infty$, and exchangeability would equate
every pair's covariance, so (i) holds there too. Under the heteroscedastic family neither variant
is identically distributed. The proposition removes the premise and says nothing about how
often the conformal wealth process then crosses, which Table~\ref{tab:overlap} measures. With the adapter learning, the scores are no longer a moving sum, and the learning pair remains
a measurement.

The thinned variant of that table has a stated sufficient condition.

\begin{corollary}\label{cor:thin}
Under the assumptions of Proposition~\ref{prop:reuse}, the thinned scores $t_k = s_{1+(k-1)m}$,
$k \ge 1$, are i.i.d. Let the $k$-th smoothed conformal $p$-value rank $t_k$ against itself and
every earlier thinned score,
$p_k = \big(\#\{j \le k : t_j > t_k\} + u_k\,\#\{j \le k : t_j = t_k\}\big)/k$, with tie-break variables $u_k$ i.i.d.\ uniform on $[0,1]$ and independent of $(\eta_t)$. Then the
$p_k$ are i.i.d.\ uniform on $[0,1]$. For the grid
$\varepsilon_1,\ldots,\varepsilon_K \in (0,1)$ of Algorithm~\ref{alg:gate}, replace the first
$p$-value by $\widetilde p_1=1/2$ and set $\widetilde p_k=p_k$ for $k\ge2$. Then
\begin{equation}
E_n^{(r)}=\prod_{j=1}^{n}\varepsilon_r\widetilde p_j^{\varepsilon_r-1},
\qquad
M_n=\frac{1}{K}\sum_{r=1}^{K}E_n^{(r)}
\label{eq:mixture-wealth}
\end{equation}
is a nonnegative supermartingale with $M_0=1$, and
$\Pr\big(\sup_n M_n\ge\nicefrac{1}{\alpha}\big)\le\alpha$.
\end{corollary}

\begin{proof}
The score $t_k$ is a function of the block $\xi_{1+(k-1)m}, \ldots, \xi_{km}$ alone, the blocks
are disjoint, and the $\xi_t$ are i.i.d., so the thinned scores are independent with a common
law. An i.i.d.\ sequence is exchangeable, and on an exchangeable sequence with tie-break variables
i.i.d.\ uniform and independent of it the smoothed conformal $p$-values are i.i.d.\ uniform
\citep{vovk2003testing,vovk2022algorithmic}. For each fixed grid point $\varepsilon_r$ and every
$n\ge2$,
\[
\mathbb{E}\!\left[\varepsilon_r p_n^{\varepsilon_r-1}\mid
p_1,\ldots,p_{n-1}\right]
=\int_0^1\varepsilon_r p^{\varepsilon_r-1}\,dp=1.
\]
At the deterministic first step,
$E_1^{(r)}=\varepsilon_r2^{1-\varepsilon_r}\le1=E_0^{(r)}$ because this expression increases
to one on $(0,1]$. Hence every $(E_n^{(r)})_{n\ge0}$ is a nonnegative supermartingale,
conservative at the first step and with conditionally mean-one increments thereafter. Their fixed
average $(M_n)$ is also a nonnegative supermartingale with $M_0=1$. Ville's inequality
\citep{ville1939etude} therefore gives the stated bound.
\end{proof}

The corollary covers the scorer held at the truth under ideal independent tie-break randomness.
The seeded streams in the finite experiment are reproducible implementation approximations. With the adapter learning, a thinned
score still depends on every earlier label through the state. The learning variant's low observed
crossing frequency on the Gaussian null is a measurement and is not a consequence of the corollary.
The single-horizon frozen residuals of Table~\ref{app:tab:horizons} share no label and still
cross, so the corollary reaches none of the real streams.

Table~\ref{app:tab:overlapfam} repeats the intervention under three further noise families at
horizon 96 and under Gaussian noise at horizons 24 and 192, each a campaign of 60 runs per variant
with the thinning stride equal to the horizon. Student-t noise has 3 degrees of freedom scaled
to standard deviation $\sigma$. The autoregressive family runs an AR(1) path with coefficient
0.5 along timestamps, so the correlation inside a label is the same in both variants and the oracle
scorer's $S=\sigma^2 I$ ignores it. The heteroscedastic family scales the noise by
$1+0.5\sin(2\pi\tau/5{,}000)$ along timestamps, so the fixed-scale oracle's independent variant is
not exchangeable and crosses in 60 of 60 by construction, while the learning adapter tracks the
scale with a lag and its independent variant crosses in 23 to 25 of 60 by tie-break stream, 21 of
60 under all three. Heavy tails alone lift the learning adapter's independent variant to 11 of 60 by
stream and 9 of 60 under all three, its $p$-stream serially dependent in 60 of 60 runs. Under Student-t noise 4 of the 60 reused-label learning runs first ended at the
divergence guard, cumulative MSE above 0.01, after 201 to 2{,}775 scores, all four having
crossed at steps 17 to 144. No other variant of any family tripped it. The truncated prefixes therefore
retain all 58 crossings. The table reports a guard-disabled completion as a sensitivity diagnostic.
It gives the same 58 crossings, a thinned fraction of 0.004, and a consecutive fraction of 0.841,
compared with 0.783 on the truncated prefixes. The per-seed streams used by the four variants agree in
60 of 60 under every family and reproduce on the run platform. The Student-t sampler's gamma
rejection step is platform dependent, so 4 of 60 Student-t seeds regenerate with different draws
on another machine; the table uses the retained matched draws.

\begin{table}[t]
\centering
\caption{The reuse intervention under further noise families and horizons, 60 runs of 20{,}000
steps per variant. Crossings are runs that ever cross $\nicefrac{1}{\alpha}$ under each of three
tie-break streams (expected count at most 3 of 60 under Ville's bound), independent then reused labels. Thinned and
consecutive fractions pool 60 runs times $m$ offsets, 1{,}440, 5{,}760, and
11{,}520 pairs at horizons 24, 96, and 192, independent then reused, and are descriptive. The
heteroscedastic family is not a null for the fixed-scale oracle.}
\label{app:tab:overlapfam}
\footnotesize
\setlength{\tabcolsep}{3.5pt}
\begin{tabular}{llrcccc}
\toprule
noise & scorer & $m$ & crossings, independent & crossings, reused & thinned & consecutive \\
\midrule
Gaussian & truth & 96 & 2, 3, 3 & 59, 58, 59 & 0.013 / 0.009 & 0.011 / 0.735 \\
Gaussian & learning & 96 & 1, 2, 1 & 55, 55, 54 & 0.013 / 0.003 & 0.024 / 0.719 \\
Student-t & truth & 96 & 2, 1, 1 & 59, 59, 59 & 0.014 / 0.021 & 0.012 / 0.860 \\
Student-t & learning & 96 & 11, 11, 11 & 58, 58, 58 & 0.012 / 0.004 & 0.042 / 0.841 \\
AR(1) & truth & 96 & 2, 2, 3 & 59, 59, 59 & 0.014 / 0.014 & 0.010 / 0.750 \\
AR(1) & learning & 96 & 0, 0, 0 & 60, 60, 60 & 0.014 / 0.010 & 0.017 / 0.856 \\
heteroscedastic & truth & 96 & 60, 60, 60 & 60, 60, 60 & 0.400 / 0.374 & 0.977 / 0.899 \\
heteroscedastic & learning & 96 & 23, 25, 24 & 57, 57, 57 & 0.005 / 0.004 & 0.044 / 0.777 \\
Gaussian & truth & 24 & 0, 0, 0 & 58, 59, 58 & 0.016 / 0.004 & 0.007 / 0.854 \\
Gaussian & learning & 24 & 1, 1, 1 & 54, 53, 52 & 0.017 / 0.003 & 0.016 / 0.833 \\
Gaussian & truth & 192 & 1, 1, 1 & 58, 58, 58 & 0.011 / 0.018 & 0.012 / 0.614 \\
Gaussian & learning & 192 & 0, 2, 0 & 57, 57, 57 & 0.011 / 0.002 & 0.012 / 0.576 \\
\bottomrule
\end{tabular}
\end{table}

Table~\ref{app:tab:resample} is the real-stream intervention of Section~\ref{sec:e6}. For each
clean run of the learned-scale shadow variant the test block's frozen residual vectors $r_t = y_t - f_t$ are permuted across origins within blocks of $B$ consecutive origins by a uniform permutation
seeded from the run seed and $B$, and the label at origin $t$ becomes $f_t + r_{\pi(t)}$. Every
origin keeps its forecast, the multiset of residual vectors is the stream's own, and the
adapter, its warm-up, and the gate run unchanged on the reordered stream, so each block is a new cell with its own configuration and none is a replay. The permutation removes the
original order within a block, and with it the alignment between each forecast and its own
residual. The block size bounds how far residual regimes are mixed and sets the number of block
boundaries, so the treatment is not reuse alone and the block size is not a one-dimensional
dose.

The adapter's test-block MSE on the reordered stream is 0.83 to 1.59 times its
real-label value on electricity, 0.80 to 1.39 on traffic, 0.53 to 3.40 on ETTh2, 0.64 to 6.28
on ETTm1, and 0.04 to 383 on weather, median 1.06 over the 360 runs. The weather cells under
delayed reveal at $B$ of 192 and 512 sit at 4.2 to 383, medians 18 and 21, where the
reordered residual stream turns the delayed corrections harmful. The number of label timestamps shared by adjacent origins falls from 95 to 39.5 to 40.0, 16.6
to 17.2, 4.5 to 4.8, and 0.16 to 0.71 by stream at $B$ of 192, 512, 2{,}048, and the whole run,
each within 0.3 of its expectation. The score extract and the classical no-reset wealth replay reproduce the saved $p$-values and fires in 450 of 450 runs. Fourteen of the 360 permuted runs, all weather
under delayed reveal at $B$ of 192, 512, and 2{,}048 (6, 7, and 1), first ended at the
divergence guard at steps 18{,}537 to 18{,}550 of 39{,}378, the cumulative MSE having passed
100 times the seasonal-naive reference computed on the unpermuted series. Thirteen had crossed
at steps 661 to 903 before that. The truncated prefixes determine the reported crossing counts
and first-crossing medians. The remaining truncated run does not cross in completion. Guard-disabled
completions for the 27 delayed weather cells at those blocks provide full-length score paths as a
sensitivity diagnostic, with every run scoring all 39{,}282 steps. ETTm1
crosses in 0 of 9 at every block and ETTh2 in at most 3 of 9. Electricity holds 6 of 9 on two
backbones at every finite block, at steps 2{,}196 to 6{,}508 at $B=2{,}048$, and 1 of 9 with
the whole run permuted. Traffic crosses in 9 of 9 at $B$ of 192 and 512, 3 of 9 at 2{,}048, and
0 of 9 with the whole run permuted. Weather crosses in 9 of 9, 9 of 9, 0 and 1 of 9, and 4 and 7 of 9.
The whole-run remainder sits at steps 14 to 307, at the start of the run.

The single-horizon intervention of Section~\ref{sec:e6} records, on every clean run of the
learned-scale shadow variant, the per-horizon terms $e_{t,h}^2/S_{t,hh}$ of Algorithm~\ref{alg:gate} at all 96 horizons, in a rerun whose pooled score reproduces the saved stream exactly in 90 of 90
runs, so the single-horizon streams are decompositions of the saved stream and not a second
campaign. Each is replayed through both constructions with the run's own seed, window, and
level. Table~\ref{app:tab:coord} gives three horizons in full and Table~\ref{app:tab:horizons}
the distribution over all 96. The pooled score crosses in 45 of 45
under each protocol at a median peak log-wealth of 47.6 and 49.9. Horizon 1 crosses in 27 of 45
under each protocol, in 0 of 9 on electricity and ETTh2 and 9 of 9 on ETTm1, traffic, and
weather, at a median peak of 9.8 and 9.6. Horizon 48 crosses in 44 of 45 under each, 8 of 9 on
traffic and 9 of 9 elsewhere, and horizon 96 in 36 and 39 of 45, 3 of 9 on traffic under both
protocols and 6 then 9 of 9 on ETTh2. The deployed gate fires in 10 and 16 of 45 runs at
horizon 1, 44 and 44 at horizon 48, and 31 and 36 at horizon 96, at a median of 0, 3, and 2 to 3
fires per run against 84 and 85 on the pooled score.

Over the 96 horizons the classical no-reset wealth process crosses in 21 to 90 of 90 runs, median 72 and
6{,}319 of 8{,}640 in all, 3{,}079 of 4{,}320 under immediate reveal and 3{,}240 under delayed,
with one horizon, 49, at 90 of 90 and 23 horizons at or below the 54 of horizon 1. Horizons 2 to
8 cross least, in 195 of 630, horizons 9 to 24 in 697 of 1{,}440, and horizons 25 to 96 in
5{,}373 of 6{,}480. The median peak log-wealth by horizon runs from 0.3 at horizon 3 to 19.3 at
horizon 79, the median first crossing from 1{,}070 steps at horizon 1 to 148 at horizon 96, and
the deployed gate fires in 21 to 88 of 90 runs by horizon at a median of 0 to 5 fires per
run. Weather crosses at every horizon in 18 of 18, ETTm1 at every horizon in at least 3,
electricity at none of horizons 1 to 7 and nearly all beyond, ETTh2 at none of horizons 1 to 6
and 15 to 22 and below 18 at 64 horizons, and traffic below 18 at 94 horizons and at 0 on 32 of
them. A nominal Ljung--Box diagnostic rejects every one of the 8{,}640 single-horizon
no-reset $p$-value streams, crossing or not, so dependence is visible at every horizon on every
stream. A nominal Kolmogorov--Smirnov marginal-uniformity diagnostic does not reject 1{,}084 of
them; neither diagnostic is calibrated for inference on these dependent real-stream sequences.

A single-horizon score is $\chi^2_1$ under the scorer's own null where the pooled score is
$\chi^2_{96}$, and the conformal $p$-value is a rank, so the marginal law leaves the null
$p$-stream uniform for either score and is not what separates them. Whether the aggregation of
one term in place of a 96-term sum separates them is answered by two analyses on the same
saved terms, which hold the one-term statistic fixed and vary the reuse
(Table~\ref{app:tab:hcontrols}). The diagonal stream reads at each origin the horizon that
lands on the last timestamp of its block of 96 origins, so all 96 read one label timestamp, the
most shared label a one-term stream can carry. It crosses in 90 of 90 runs and fires in 90 of
90, at a median peak log-wealth of 98.0 against the pooled score's 48.9 on the same runs, 79.1
under immediate reveal and 110.4 under delayed, at a median first crossing of 121 steps and 80
fires per run. The random stream draws one horizon uniformly at each origin, so adjacent origins
share a label entry with probability near $\nicefrac{1}{96}$. Over three draw streams per run it
crosses in 181 of 270 and fires in 123, at a median peak of 6.6, 5.9 under immediate and 7.7
under delayed reveal, and a median first crossing of 640 steps. By stream it crosses in 54 of 54
on weather at a median peak of 79.5, 53 of 54 on ETTm1 at 19.9, 43 of 54 on electricity at 5.3,
27 of 54 on ETTh2 at 3.1, and 4 of 54 on traffic at 0.1. Thinning a single term to origins 96
apart, over eight offsets, crosses in 0.037, 0.075, and 0.035 of the 720 run-offset pairs at
horizons 1, 48, and 96, against 0.032 and 0.044 for the thinned pooled score, with five
stream-horizon cells above the level, ETTh2 at horizon 48 at 35 of 144, weather at horizons 1,
48, and 96 at 23, 12, and 15, and electricity at horizon 96 at 8. Neither the diagonal nor any random
stream passes the nominal Ljung--Box diagnostic. Restoring the shared label to a one-term stream restores every
crossing at twice the pooled wealth, and removing it leaves two thirds of the streams crossing
at a fifteenth of that wealth, so reuse is the amplifier and the 96-term aggregation is not
necessary for the separation, and on traffic the shared label is nearly the whole of the
crossing.

Two further analyses complete the grid of one-term against 96-term statistics with and without a
shared label (Table~\ref{app:tab:hcontrols2}). The antidiagonal stream reverses the diagonal's
horizon cycle, so a label is read by at most two origins 48 apart where the diagonal's are 96 apart, and crosses in
58 of 90 at a median peak of 4.9 against the diagonal's 90 of 90 at 98.0, in 0 of 18 on traffic
and 18 of 18 on ETTm1 and weather, so the concentration of reuse and not the cycle carries the
diagonal's wealth. The target-sum stream keeps the deployed 96-term sum but groups the saved terms by the timestamp of their target, so each label enters one statistic and consecutive
statistics share none, and crosses in 75 of 90 at 12.4 against the pooled score's 90 of 90 at
48.9 on the same runs, 47.6 and 49.9 by reveal, so removing the shared label from the aggregated statistic lowers its wealth about
fourfold and leaves five sixths of the streams crossing. Read across the grid, the aggregation
without a shared label raises the one-term wealth from 4.9 to 12.4, and the shared label raises
it from 4.9 to 98.0 in the one-term statistic and from 12.4 to 47.6 in the aggregated one. The
corners differ in count and in terms per statistic, so the grid ranks the two factors and does not
decompose them.

\begin{table}[t]
\centering
\caption{Single-horizon crossings over all 96 horizons, learned-scale shadow variant, 18 runs per
stream and horizon (three backbones, three seeds, both protocols). Entries count run-horizon
pairs within the horizon band in which the classical no-reset wealth process ever crosses
$\nicefrac{1}{\alpha}$, out of the pairs in the band; the bound's expected count is 5 percent of
each denominator. Silent horizons counts the horizons the wealth process never crosses on the stream.}
\label{app:tab:horizons}
\scriptsize
\setlength{\tabcolsep}{3pt}
\begin{tabular}{lcccccccc}
\toprule
stream & horizon 1 & 2 to 8 & 9 to 24 & 25 to 48 & 49 to 72 & 73 to 96 & all & silent horizons \\
\midrule
electricity & 0/18 & 1/126 & 279/288 & 432/432 & 427/432 & 412/432 & 1{,}551/1{,}728 & 7 \\
ETTh2 & 0/18 & 2/126 & 33/288 & 391/432 & 271/432 & 378/432 & 1{,}075/1{,}728 & 14 \\
ETTm1 & 18/18 & 54/126 & 84/288 & 426/432 & 432/432 & 432/432 & 1{,}446/1{,}728 & 0 \\
traffic & 18/18 & 12/126 & 13/288 & 195/432 & 149/432 & 132/432 & 519/1{,}728 & 32 \\
weather & 18/18 & 126/126 & 288/288 & 432/432 & 432/432 & 432/432 & 1{,}728/1{,}728 & 0 \\
all & 54/90 & 195/630 & 697/1{,}440 & 1{,}876/2{,}160 & 1{,}711/2{,}160 & 1{,}786/2{,}160 & 6{,}319/8{,}640 & \\
\bottomrule
\end{tabular}
\end{table}

\begin{table}[t]
\centering
\caption{Horizon analyses on the saved single-horizon terms, the one-term statistic held
fixed. Random draws one horizon uniformly at each origin, three draw streams per run (54 per
stream, 270 in all). Diagonal reads at each origin the horizon that lands on the last timestamp
of its block of 96 origins, so the block reads one label (18 per stream, 90 in all). Thinned
keeps one term every 96 origins over eight offsets (144 run-offset pairs per stream, 720 in
all); its fractions describe, offsets of one run sharing its stream. Crossings count streams in
which the classical no-reset wealth process ever crosses $\nicefrac{1}{\alpha}$, fires the streams in which the deployed gate fires at least once, and peak is the median over streams of the peak log-wealth.}
\label{app:tab:hcontrols}
\footnotesize
\setlength{\tabcolsep}{3pt}
\begin{tabular}{lccccccccc}
\toprule
& \multicolumn{3}{c}{random horizon} & \multicolumn{3}{c}{diagonal} & \multicolumn{3}{c}{thinned, horizons 1, 48, 96} \\
stream & crossings & fires & peak & crossings & fires & peak & \multicolumn{3}{c}{crossing fraction} \\
\midrule
electricity & 43/54 & 31 & 5.3 & 18/18 & 18 & 111.6 & 0.000 & 0.014 & 0.056 \\
ETTh2 & 27/54 & 2 & 3.1 & 18/18 & 18 & 367.9 & 0.000 & 0.243 & 0.007 \\
ETTm1 & 53/54 & 34 & 19.9 & 18/18 & 18 & 98.0 & 0.000 & 0.021 & 0.007 \\
traffic & 4/54 & 2 & 0.1 & 18/18 & 18 & 25.1 & 0.028 & 0.014 & 0.000 \\
weather & 54/54 & 54 & 79.5 & 18/18 & 18 & 36.1 & 0.160 & 0.083 & 0.104 \\
all & 181/270 & 123 & 6.6 & 90/90 & 90 & 98.0 & 0.037 & 0.075 & 0.035 \\
\bottomrule
\end{tabular}
\end{table}

\begin{table}[t]
\centering
\caption{Two horizon analyses completing the grid of Table~\ref{app:tab:hcontrols}, 18 per stream
(three backbones, three seeds, both protocols) and 90 in all. Antidiagonal reads at each origin the horizon that reverses the diagonal's cycle, so
a label is read by at most two origins 48 apart. Target sum keeps the 96-term sum but groups the
saved terms by the timestamp of their target, so each label enters one statistic. Columns as
in Table~\ref{app:tab:hcontrols}; the last two rows split the totals by feedback protocol.}
\label{app:tab:hcontrols2}
\footnotesize
\setlength{\tabcolsep}{4pt}
\begin{tabular}{lcccccc}
\toprule
& \multicolumn{3}{c}{antidiagonal} & \multicolumn{3}{c}{target sum} \\
\cmidrule(lr){2-4}\cmidrule(lr){5-7}
stream & crossings & fires & peak & crossings & fires & peak \\
\midrule
electricity & 10/18 & 4 & 3.3 & 16/18 & 15 & 4.7 \\
ETTh2 & 12/18 & 0 & 3.8 & 18/18 & 12 & 12.4 \\
ETTm1 & 18/18 & 6 & 26.7 & 18/18 & 18 & 57.6 \\
traffic & 0/18 & 0 & -0.2 & 5/18 & 5 & 1.6 \\
weather & 18/18 & 18 & 95.1 & 18/18 & 18 & 117.0 \\
all & 58/90 & 28 & 4.9 & 75/90 & 68 & 12.4 \\
immediate & 27/45 & 9 & 4.0 & 37/45 & 34 & 10.6 \\
delayed & 31/45 & 19 & 5.4 & 38/45 & 34 & 15.1 \\
\bottomrule
\end{tabular}
\end{table}

A no-learner analysis asks what the adapter contributes. Each frozen backbone's own
residuals on the same test block are scored with the same split and cached forecasts, the scale
of every horizon fixed at its training-block mean squared residual, and no state learned
anywhere, so nothing is reused across origins beyond what the overlapping horizons force in the
pooled sum. The pooled score and all 96 single-horizon terms are replayed through both
constructions with the run's own seed, window, and level over three backbones, five streams, and
three tie-break seeds, and Table~\ref{app:tab:streamonly} gives the counts. The classical no-reset
wealth process crosses in 45 of 45 runs on the pooled score, at a median first crossing of step 19
and a median peak log-wealth of 430, and in 4{,}320 of 4{,}320 single-horizon streams, at
median peaks of 33.7 to 148.1 by horizon against the adapter's 9.6 to 15.9 in
Table~\ref{app:tab:coord}. The deployed gate fires in 45 of 45 pooled runs, at a median of
131 fires per run against 84 and 85 for the adapter's pooled score, and in 15 to 44 of 45 runs by
horizon, 18 at horizon 1, 44 at 48, and 36 at 96. A nominal Ljung--Box diagnostic rejects all
4{,}320 single-horizon no-reset $p$-value streams, while a nominal Kolmogorov--Smirnov
marginal-uniformity diagnostic does not reject 310. These are descriptive diagnostics because
the real-stream sequences are dependent. The residuals' lag-1 autocorrelation, median over horizons, is 0.72 to 0.98 by
stream, and at horizon 1 it lies between $-0.13$ and $+0.31$, where the stream still crosses
in 45 of 45, late and low on electricity at a median first crossing of 6{,}880 and a median peak
of 3.7, and early on ETTm1 at 59. Where the adapter's horizon-1 stream crosses in 0 of 9 on
electricity and ETTh2, the raw stream under this fixed scale crosses in 9 of 9 on both; under
the matched trailing scale below it crosses in 0 of 9 there as well, so that contrast belongs to
the scale and not to the one-step correction. The dependence the gate meets is
in the forecast-error process before any adaptation. Whether the adapter's streams carry less of
it cannot be read from the fixed-scale peaks, because that scoring fixes each horizon's scale on
the training block while the adapter divides by a time-varying innovation covariance, and holding
the adapter's own scale moves its peak wealth by a factor of 2 to 7 on the same streams
(Table~\ref{app:tab:valid}).

A second scoring therefore divides each horizon's residual by a trailing window of the 250 most recently revealed squared residuals, the length of the deployed calibration window, seeded from the training block, the normalization nearest the adapter's own
(Table~\ref{app:tab:streamonly-trailing}). Under it the raw residuals cross in 39 of 45 pooled
runs at a median peak log-wealth of 59.1, against the adapter's 45 of 45 at 47.6 and 49.9, and in
2{,}991 of 4{,}320 single-horizon streams, 12 of 45 at horizon 1, 87 of 315 at horizons 2 to 8,
488 of 720 at 9 to 24, 802 of 1{,}080 at 25 to 48, and 1{,}602 of 2{,}160 at 49 to 96. The
trailing scale removes most of the crossing at horizon 1, where the fixed scale crossed in 45 of
45, and on weather it leaves the pooled raw stream crossing in 3 of 9 at a median peak of 0.4
where the adapter's pooled stream crosses in 9 of 9, so on the stream that carries the transient
the adapter adds dependence that the raw residuals show much more weakly under a matched scale. On the other
four streams the raw stream still crosses in 9 of 9 pooled runs and the two carry about the same.
At windows of 96 and 1{,}000 steps the pooled crossings are 36 and 40 of 45, the single-horizon
crossings 2{,}663 and 3{,}218 of 4{,}320, and the median peaks 33.7 and 125.1, with weather at 0
and 4 of 9, so the window moves the wealth and not the reading.

A refiner with causal value asks whether the premise fails only where adaptation loses. The ORCA
residual refiner, rerun under the delayed reveal with its forecast at
every origin saved, is scored on its own residuals with the split, cached forecasts, training-block
scale, and tie-break seeds of the no-learner analysis, so the two streams share every origin and differ
in the residual alone (Table~\ref{app:tab:refiner-self}). The classical no-reset wealth process crosses in 45 of
45 clean runs on the pooled score, 9 of 9 on every stream, at a median first crossing of step 57
against the frozen stream's 19 on the same cells, later on 25 runs, earlier on 15, and at the same step
on 5, with a median peak log-wealth of 397 against 430 and a median run-by-run ratio of 0.96. The
refiner lowers the peak on ETTm1, 284 against 400, and raises it on traffic and weather, 701 and 549
against 665 and 430. Under the trailing scale it crosses in 37 of 45 and in 1 of 9 on weather, where
the raw residuals cross in 39 and 3 of 9. Thinned to one origin in 96 at the run's drawn offset it
crosses in 24 of 45, at a median first crossing of 11{,}929 and a median peak of 3.1 against the
threshold of 3.0, where the frozen stream's single offset crosses in 26 of 45
(Appendix~\ref{app:power}). The deployed gate fires in 45 of 45 runs, at a median of 145 fires per run
against 131. The ridge refiner used as the linear foil
crosses in 45, 39, and 33 of 45 under the three scorings at a median peak ratio of 1.10. The premise
therefore fails on the score stream of an adapter that keeps its causal gain exactly as it fails on the
streams of the adapters that lose and on the stream with no learner, and in the pooled count the
refiner's correction neither removes the dependence the overlapping residual carries nor adds to it.

\begin{table}[t]
\centering
\caption{The no-learner analysis on the real streams. Each frozen backbone's own residuals on the
test block are scored with the scale of every horizon fixed at its training-block mean squared
residual, and the pooled score and all 96 single-horizon terms are replayed through both
constructions with the run's own seed, window, and level, 9 runs per stream (three backbones,
three tie-break seeds). Lag-1 is the residuals' lag-1 autocorrelation, median over horizons then
range over backbones, and at horizon 1 alone. Crossings count runs in which the classical no-reset
wealth process ever crosses $\nicefrac{1}{\alpha}$, pooled then over all 96 horizons (the bound's
expected count is at most 0.45 of 9 and 43.2 of 864), fires the runs in which the deployed gate fires at least once, pooled then at horizons 1, 48, and 96. First crossing and peak
log-wealth are medians over runs of the pooled score.}
\label{app:tab:streamonly}
\footnotesize
\setlength{\tabcolsep}{3pt}
\begin{tabular}{lcccccccc}
\toprule
stream & lag-1 & lag-1, horizon 1 & pooled & 96 horizons & fires & fires 1, 48, 96 & first crossing & peak \\
\midrule
electricity & 0.91 to 0.93 & $-0.07$ to 0.00 & 9/9 & 864/864 & 9/9 & 0, 9, 9 & 50 & 463 \\
ETTh2 & 0.97 to 0.98 & $-0.03$ to +0.16 & 9/9 & 864/864 & 9/9 & 0, 9, 9 & 8 & 375 \\
ETTm1 & 0.97 to 0.98 & +0.05 to +0.21 & 9/9 & 864/864 & 9/9 & 9, 9, 9 & 10 & 400 \\
traffic & 0.72 to 0.74 & $-0.13$ to $-0.02$ & 9/9 & 864/864 & 9/9 & 0, 8, 0 & 1{,}514 & 665 \\
weather & 0.94 & +0.03 to +0.31 & 9/9 & 864/864 & 9/9 & 9, 9, 9 & 690 & 430 \\
all &  &  & 45/45 & 4{,}320/4{,}320 & 45/45 & 18, 44, 36 & 19 & 430 \\
\bottomrule
\end{tabular}
\end{table}

\begin{table}[t]
\centering
\caption{The no-learner analysis under a trailing scale. Each horizon's residual is divided by the
mean of the 250 most recently revealed squared residuals at that horizon, seeded from the training
block and sliding causally, and the pooled score and all 96 single-horizon terms are replayed as in
Table~\ref{app:tab:streamonly}, 9 runs per stream. Crossings count runs in which the classical no-reset
wealth process ever crosses $\nicefrac{1}{\alpha}$, pooled, over all 96 horizons, and at horizons 1,
48, and 96. First crossing and peak log-wealth are medians over runs of the pooled score, the first
crossing over the runs that cross.}
\label{app:tab:streamonly-trailing}
\footnotesize
\setlength{\tabcolsep}{4pt}
\begin{tabular}{lccccc}
\toprule
stream & pooled & 96 horizons & horizons 1, 48, 96 & first crossing & peak \\
\midrule
electricity & 9/9 & 805/864 & 0, 9, 9 & 49 & 61.3 \\
ETTh2 & 9/9 & 712/864 & 0, 8, 9 & 8 & 63.1 \\
ETTm1 & 9/9 & 864/864 & 9, 9, 9 & 14 & 69.4 \\
traffic & 9/9 & 171/864 & 3, 7, 0 & 942 & 19.8 \\
weather & 3/9 & 439/864 & 0, 3, 9 & 21 & 0.4 \\
all & 39/45 & 2{,}991/4{,}320 & 12, 36, 36 & 17 & 59.1 \\
\bottomrule
\end{tabular}
\end{table}

\begin{table}[t]
\centering
\caption{The refiner's own score stream. The ORCA residual refiner is scored on its own residuals
under the delayed reveal with the split, cached forecasts,
training-block scale, and tie-break seeds of Table~\ref{app:tab:streamonly}, 9 runs per stream.
Crossings count runs in which the classical no-reset wealth process ever crosses $\nicefrac{1}{\alpha}$
under the fixed scale, under the trailing scale of Table~\ref{app:tab:streamonly-trailing}, and
under the fixed scale thinned to one origin in 96 at the run's drawn offset (the bound's expected
count is at most 0.45 of 9). First crossing and peak log-wealth are medians over runs of the
fixed-scale pooled score, the first crossing over the runs that cross. The last two columns give
the frozen stream's values on the same cells.}
\label{app:tab:refiner-self}
\footnotesize
\setlength{\tabcolsep}{4pt}
\begin{tabular}{lccccccc}
\toprule
stream & fixed & trailing & thinned & first crossing & peak & frozen first crossing & frozen peak \\
\midrule
electricity & 9/9 & 9/9 & 3/9 & 55 & 423 & 50 & 463 \\
ETTh2 & 9/9 & 9/9 & 1/9 & 8 & 360 & 8 & 375 \\
ETTm1 & 9/9 & 9/9 & 5/9 & 12 & 284 & 10 & 400 \\
traffic & 9/9 & 9/9 & 8/9 & 1{,}483 & 701 & 1{,}514 & 665 \\
weather & 9/9 & 1/9 & 7/9 & 700 & 549 & 690 & 430 \\
all & 45/45 & 37/45 & 24/45 & 57 & 397 & 19 & 430 \\
\bottomrule
\end{tabular}
\end{table}

\begin{table}[t]
\centering
\caption{Single-horizon scoring on the real streams, learned-scale shadow variant, 45 clean runs
per protocol (three backbones, five streams, three seeds), immediate then delayed reveal. Crossings count runs in which the classical no-reset
wealth process ever crosses $\nicefrac{1}{\alpha}$ (the bound's expected count is at most 2.25 of 45). First
crossing is the median step over crossing runs, peak the median over runs of the maximum
log-wealth, fires the runs in which the deployed gate fires at least once, and per run the
median fire count.}
\label{app:tab:coord}
\footnotesize
\setlength{\tabcolsep}{3.5pt}
\begin{tabular}{lccccc}
\toprule
score & crossings & first crossing & peak log-wealth & fires & per run \\
\midrule
pooled, 96 horizons & 45/45, 45/45 & 528, 138 & 47.6, 49.9 & 45/45, 45/45 & 84, 85 \\
horizon 1 & 27/45, 27/45 & 1{,}072, 1{,}013 & 9.8, 9.6 & 10/45, 16/45 & 0, 0 \\
horizon 48 & 44/45, 44/45 & 460.5, 167 & 12.2, 14.2 & 44/45, 44/45 & 3, 3 \\
horizon 96 & 36/45, 39/45 & 63.5, 153 & 11.4, 15.9 & 31/45, 36/45 & 2, 3 \\
\bottomrule
\end{tabular}
\end{table}

\begin{table}[t]
\centering
\caption{Residual-vector permutation on the real streams, learned-scale shadow variant, 45 clean
runs per cell (three backbones, five streams, three seeds), immediate then delayed reveal.
Shared entries is the realized mean number of label timestamps shared by adjacent origins,
range over streams. Crossings count runs in which the classical no-reset wealth process ever crosses
$\nicefrac{1}{\alpha}$ (the bound's expected count is at most 2.25 of 45). First crossing is the median step over
crossing runs. Thinned, consecutive, and burst fractions pool 45 runs times 96 offsets at stride
96 under tie-break stream 0 and are descriptive. Bursts are 16 consecutive origins at the
thinned count.}
\label{app:tab:resample}
\footnotesize
\setlength{\tabcolsep}{3.5pt}
\begin{tabular}{lcccccc}
\toprule
block $B$ & shared entries & crossings & first crossing & thinned & consecutive & bursts of 16 \\
\midrule
real labels & 95 & 45/45, 45/45 & 528, 138 & 0.032, 0.044 & 0.631, 0.673 & 0.701, 0.702 \\
192 & 39.5 to 40.0 & 24/45, 27/45 & 954, 1{,}186 & 0.019, 0.019 & 0.039, 0.039 & 0.054, 0.060 \\
512 & 16.6 to 17.2 & 24/45, 24/45 & 697, 780 & 0.018, 0.019 & 0.013, 0.014 & 0.090, 0.067 \\
2{,}048 & 4.5 to 4.8 & 9/45, 10/45 & 6{,}371, 6{,}398 & 0.012, 0.013 & 0.020, 0.016 & 0.018, 0.022 \\
whole run & 0.16 to 0.71 & 5/45, 8/45 & 215, 207 & 0.008, 0.007 & 0.060, 0.062 & 0.007, 0.010 \\
\bottomrule
\end{tabular}
\end{table}

\FloatBarrier
\section{Divergence counts by variant and protocol}
\label{app:divergence}

Three backbones, five streams, ten contamination settings, and three seeds give 450 runs per variant
per protocol for each of the six original variants. Divergence counts fall sharply when the loop is made
causal. Figure~\ref{fig:dumbbell} reports the six original variants. Table~\ref{tab:divergence} adds the
nine later variants. The full gate goes from 60 divergences under immediate reveal to 2 under delayed reveal.
Eight Kalman divergences remain under delayed reveal, two full gate, one ungated filter,
three Huber-only, and two fixed cutoff. Seven of the eight belong to one stream and setting, the weather stream at seed 1 under the 5 percent spike injector, with the divergence guard triggered
at exactly step 200. The eighth belongs to the Huber-only variant on ETTm1 under the packet-freeze injector,
triggered at step 8{,}982.

\textbf{Divergence-inclusive sensitivity.} For the 45 paired runs behind Figure~\ref{fig:spikeprofile},
an adverse outcome is either divergence of the contaminated run or, when both runs survive, a
larger final MSE than its clean twin. Under immediate and delayed reveal, respectively, adverse
outcomes occur in 42 and 45 runs for the ungated filter, 32 and 22 for Huber-only, 34 and 40 for
the full gate, and 29 and 18 for the fixed cutoff. These counts include contaminated-run
divergences of 0 and 1, 0 and 2, 6 and 2, and 0 and 2 in the same order; no clean-only run
diverges. This sensitivity assigns no numerical loss to a truncated run and does not replace the
separate divergence and survivor-only results.

Four later variants extend the table. The frozen-scale variant diverges in 60 of 450 runs under
immediate reveal and 2 of 450 under delayed reveal, matching the full gate on both, with the
immediate divergences falling on TiRex and Chronos-2 alone. Its clean weather cell finishes at a
median 84.8 on TiRex and 211.2 on Chronos-2, both on 3 of 3 diverged seeds, against the full
gate's 76.4 and 161.8 on the same cells. The capped-policy variant diverges in 1
of 450 immediate runs, at weather under the 10 percent spike injector on TiRex, and in none of
its clean cells. It cuts boosted steps per run from per-backbone medians of 1{,}262 to 1{,}584 to
medians of 950 to 1{,}200. On the four non-weather streams its clean test-block MSE lands between 23
percent below and 6 percent above the ungated filter's on every backbone. The variant that leaves the spike check live during recovery diverges in 47 of
450, 15 on TiRex, 30 on Chronos-2, and 2 on TimesFM~2.5, with 5 of its 9 clean weather runs among
them. Under delayed reveal both capped variants diverge in 2 of 450 runs, at the same two weather
cells as the full gate, and in none of their clean cells.

The three single-component variants complete the table. Under immediate reveal the soft-reset variant
diverges in 60 of 450 runs, the full gate's own count, the boost variant in 1, and the override variant
in 0. Under delayed reveal each diverges in 2 of 450, at the two weather cells above, and in
none of its clean cells. On the clean weather cell under delayed reveal, in the order TiRex, Chronos-2, and TimesFM~2.5,
the soft-reset variant finishes at 0.98 to 1.03, 8.43 to 8.50, and 1.15 to 2.35. The boost variant
finishes at 0.34 to 0.37, 5.62 to 5.69, and 0.69 to 0.72, and the override variant at 0.33 to 0.34,
2.65 to 2.90, and 0.65 to 0.72. The full gate finishes at 0.50 to 0.53, 8.32 to 9.14, and 1.18 to
1.23, and the ungated filter at 0.29, 5.75, and 0.52. On the four other clean streams the
soft-reset variant's test-block MSE sits 0.12, 0.19, 0.04, and 0.01 above the ungated
filter's on electricity, ETTh2, ETTm1, and traffic, means over nine runs, against 0.13, 0.20,
0.01, and 0.01 for the full gate, 0.15, 0.10, 0.00, and 0.00 for the boost variant, and 0.01, 0.02,
0.01, and $-0.01$ for the override variant. Tables~\ref{app:tab:transient-imm}
and~\ref{app:tab:transient-del} list the first three seeds of the clean weather cell for every variant
under each protocol.

Table~\ref{app:tab:seeddeltas} gives the paired per-seed differences from the ungated filter
under delayed reveal over twenty seeds per backbone, seven variants rerun on that cell alone under
the same configuration and the three variants added later over the same seeds. Over the twenty seeds
the soft-reset variant's mean excess over the ungated filter is 0.61, 3.57, and 1.37 on TiRex,
Chronos-2, and TimesFM~2.5 against the full gate's 0.23, 2.98, and 1.11, positive in 20 of 20
seeds for both variants. The soft-reset variant exceeds the full gate in the same seed in 18, 14, and 13
of 20 (two-sided sign test 0.0004, 0.12, and 0.26). The table reports 30 sign tests without
adjustment; the TiRex margin survives a Bonferroni factor of 30 at 0.012, and the other two fail
without one. On Chronos-2 the override variant and the capped policy finish below the ungated filter
in 20 of 20 seeds, by 2.89 and 2.86, and on TimesFM~2.5 the capped policy does in 18 of 20, by
0.16, while on TiRex every deployed variant finishes above it in 20 of 20 and the two
detector-stride-96 variants in 14 of 20. The gate without its soft reset finishes below the full gate seed for
seed in 20 of 20 on every backbone and below the soft-reset variant in 20 of 20; it sits 0.04 above
the ungated filter on TiRex in 20 of 20 and below it on Chronos-2 and TimesFM~2.5 in 20 and 19 of
20. The combined variant finishes within 0.001 of the detector-stride-96 gate or below it in 20 of 20 seeds on
every backbone, strictly below in 16 and 18 of 20 on TiRex and TimesFM~2.5.

\begin{table}[t]
\centering
\caption{Final cumulative MSE on the clean weather cell under immediate reveal, every variant, the
three seeds per backbone in seed order. A dagger marks a run that ended at the divergence guard,
whose value is the cumulative MSE at truncation. Section~\ref{sec:attribution} and
Section~\ref{sec:cutoff} quote their immediate-reveal ranges and medians from these rows.}
\label{app:tab:transient-imm}
\scriptsize
\setlength{\tabcolsep}{3pt}
\begin{tabular}{lccc}
\toprule
variant & TiRex & Chronos-2 & TimesFM~2.5 \\
\midrule
ungated filter & 2.032 / 2.032 / 2.032 & 3.575 / 3.575 / 3.575 & 0.361 / 0.361 / 0.361 \\
Huber-only & 2.003 / 2.005 / 2.001 & 3.519 / 3.520 / 3.517 & 0.358 / 0.357 / 0.359 \\
full gate & 80.6$^\dagger$ / 75.4$^\dagger$ / 76.4$^\dagger$ & 172.6$^\dagger$ / 161.8$^\dagger$ / 89.7$^\dagger$ & 3.051 / 2.571 / 3.061 \\
fixed cutoff & 0.374 / 0.374 / 0.374 & 0.231 / 0.231 / 0.231 & 3.413 / 3.413 / 3.413 \\
soft-reset variant & 83.1$^\dagger$ / 83.1$^\dagger$ / 70.4$^\dagger$ & 187.9$^\dagger$ / 216.1$^\dagger$ / 189.1$^\dagger$ & 2.500 / 2.696 / 3.026 \\
boost variant & 2.246 / 2.167 / 2.154 & 4.046 / 4.183 / 3.899 & 0.383 / 0.380 / 0.378 \\
override variant & 0.358 / 0.356 / 0.358 & 1.066 / 1.067 / 1.064 & 0.341 / 0.341 / 0.342 \\
frozen scale & 84.8$^\dagger$ / 87.5$^\dagger$ / 83.8$^\dagger$ & 211.2$^\dagger$ / 208.8$^\dagger$ / 228.0$^\dagger$ & 2.577 / 2.693 / 2.583 \\
capped policy & 0.309 / 0.287 / 0.275 & 0.358 / 0.347 / 0.350 & 0.159 / 0.927 / 0.933 \\
capped policy, spike check live & 31.1$^\dagger$ / 4.925 / 44.3$^\dagger$ & 32.5$^\dagger$ / 40.8$^\dagger$ / 31.6$^\dagger$ & 7.163 / 0.505 / 4.184 \\
detector-stride-96 gate & 0.396 / 0.396 / 0.396 & 0.231 / 0.231 / 0.231 & 3.413 / 3.413 / 3.413 \\
full gate, no soft reset & 0.257 / 0.247 / 0.249 & 0.415 / 0.383 / 0.434 & 0.141 / 0.136 / 0.138 \\
detector-stride-96, capped response & 0.362 / 0.362 / 0.362 & 0.231 / 0.231 / 0.231 & 3.413 / 3.413 / 3.413 \\
\bottomrule
\end{tabular}
\end{table}

\begin{table}[t]
\centering
\caption{Final cumulative MSE on the clean weather cell under delayed reveal, every variant, the three
seeds per backbone in seed order. No run diverges. Section~\ref{sec:attribution} quotes its
delayed-reveal ranges from these rows.}
\label{app:tab:transient-del}
\scriptsize
\setlength{\tabcolsep}{3pt}
\begin{tabular}{lccc}
\toprule
variant & TiRex & Chronos-2 & TimesFM~2.5 \\
\midrule
ungated filter & 0.286 / 0.286 / 0.286 & 5.753 / 5.753 / 5.753 & 0.517 / 0.517 / 0.517 \\
Huber-only & 0.337 / 0.342 / 0.343 & 5.554 / 5.587 / 5.571 & 0.683 / 0.378 / 0.558 \\
full gate & 0.530 / 0.502 / 0.513 & 9.143 / 8.322 / 9.066 & 1.223 / 1.184 / 1.231 \\
fixed cutoff & 0.167 / 0.167 / 0.167 & 0.181 / 0.181 / 0.181 & 0.128 / 0.128 / 0.128 \\
soft-reset variant & 1.006 / 0.984 / 1.027 & 8.430 / 8.497 / 8.498 & 2.346 / 2.090 / 1.147 \\
boost variant & 0.368 / 0.342 / 0.352 & 5.616 / 5.681 / 5.695 & 0.718 / 0.691 / 0.693 \\
override variant & 0.326 / 0.337 / 0.329 & 2.650 / 2.656 / 2.897 & 0.650 / 0.651 / 0.716 \\
frozen scale & 1.260 / 1.193 / 1.116 & 8.457 / 8.492 / 8.456 & 2.320 / 2.002 / 2.731 \\
capped policy & 0.341 / 0.342 / 0.343 & 2.841 / 2.788 / 2.868 & 0.294 / 0.363 / 0.359 \\
capped policy, spike check live & 0.388 / 0.446 / 0.403 & 1.505 / 0.922 / 1.454 & 1.114 / 0.265 / 0.255 \\
detector-stride-96 gate & 0.167 / 0.933 / 0.167 & 0.181 / 0.181 / 0.181 & 0.128 / 0.128 / 0.128 \\
full gate, no soft reset & 0.315 / 0.325 / 0.326 & 2.635 / 2.321 / 2.017 & 0.302 / 0.324 / 0.299 \\
detector-stride-96, capped response & 0.167 / 0.731 / 0.167 & 0.181 / 0.181 / 0.181 & 0.128 / 0.128 / 0.128 \\
\bottomrule
\end{tabular}
\end{table}

\begin{table}[t]
\centering
\caption{Paired per-seed differences in test-block MSE from the ungated filter at the clean
weather cell under delayed reveal, conditional on this fixed weather path, with twenty seeds (0 to 19)
per backbone. The interval is a bootstrap 95 percent
interval of the mean over seeds, positive counts the seeds whose difference is above zero out of
the nonzero ones, and the last column is a two-sided sign test. Seeds move the adapter's initial
perturbation and the conformal tie-break stream only. The gate without its soft reset shares the
deployed configuration and its ungated reference; the detector-stride-96 gate and its capped variant
use a matched ungated reference over the same twenty seeds.}
\label{app:tab:seeddeltas}
\scriptsize
\setlength{\tabcolsep}{3pt}
\begin{tabular}{llrrrrrrr}
\toprule
backbone & variant & $n$ & ungated & variant & difference & interval & positive & sign $p$ \\
\midrule
TiRex & full gate & 20 & 0.29 & 0.52 & $+0.23$ & +0.20 to +0.26 & 20/20 & $<2\times10^{-6}$ \\
TiRex & soft-reset variant & 20 & 0.29 & 0.90 & $+0.61$ & +0.53 to +0.69 & 20/20 & $<2\times10^{-6}$ \\
TiRex & boost variant & 20 & 0.29 & 0.35 & $+0.07$ & +0.06 to +0.07 & 20/20 & $<2\times10^{-6}$ \\
TiRex & override variant & 20 & 0.29 & 0.33 & $+0.05$ & +0.04 to +0.06 & 20/20 & $<2\times10^{-6}$ \\
TiRex & capped policy & 20 & 0.29 & 0.36 & $+0.08$ & +0.06 to +0.09 & 20/20 & $<2\times10^{-6}$ \\
TiRex & capped policy, spike check live & 20 & 0.29 & 0.45 & $+0.16$ & +0.15 to +0.18 & 20/20 & $<2\times10^{-6}$ \\
TiRex & Huber-only & 20 & 0.29 & 0.34 & $+0.05$ & +0.05 to +0.05 & 20/20 & $<2\times10^{-6}$ \\
TiRex & full gate, no soft reset & 20 & 0.29 & 0.32 & $+0.04$ & +0.04 to +0.04 & 20/20 & $<2\times10^{-6}$ \\
TiRex & detector-stride-96 gate & 20 & 0.29 & 0.52 & $+0.24$ & +0.12 to +0.36 & 14/20 & 0.115 \\
TiRex & detector-stride-96, capped response & 20 & 0.29 & 0.41 & $+0.13$ & +0.03 to +0.22 & 14/20 & 0.115 \\
Chronos-2 & full gate & 20 & 5.75 & 8.74 & $+2.98$ & +2.66 to +3.29 & 20/20 & $<2\times10^{-6}$ \\
Chronos-2 & soft-reset variant & 20 & 5.75 & 9.32 & $+3.57$ & +3.08 to +4.04 & 20/20 & $<2\times10^{-6}$ \\
Chronos-2 & boost variant & 20 & 5.75 & 5.60 & $-0.15$ & $-0.20$ to $-0.11$ & 0/20 & $<2\times10^{-6}$ \\
Chronos-2 & override variant & 20 & 5.75 & 2.86 & $-2.89$ & $-2.96$ to $-2.82$ & 0/20 & $<2\times10^{-6}$ \\
Chronos-2 & capped policy & 20 & 5.75 & 2.89 & $-2.86$ & $-2.92$ to $-2.78$ & 0/20 & $<2\times10^{-6}$ \\
Chronos-2 & capped policy, spike check live & 20 & 5.75 & 3.52 & $-2.23$ & $-3.46$ to $-0.93$ & 4/20 & 0.012 \\
Chronos-2 & Huber-only & 20 & 5.75 & 5.57 & $-0.18$ & $-0.19$ to $-0.16$ & 0/20 & $<2\times10^{-6}$ \\
Chronos-2 & full gate, no soft reset & 20 & 5.75 & 2.33 & $-3.42$ & $-3.52$ to $-3.31$ & 0/20 & $<2\times10^{-6}$ \\
Chronos-2 & detector-stride-96 gate & 20 & 5.75 & 0.18 & $-5.57$ & $-5.57$ to $-5.57$ & 0/20 & $<2\times10^{-6}$ \\
Chronos-2 & detector-stride-96, capped response & 20 & 5.75 & 0.18 & $-5.57$ & $-5.57$ to $-5.57$ & 0/20 & $<2\times10^{-6}$ \\
TimesFM~2.5 & full gate & 20 & 0.52 & 1.63 & $+1.11$ & +0.94 to +1.28 & 20/20 & $<2\times10^{-6}$ \\
TimesFM~2.5 & soft-reset variant & 20 & 0.52 & 1.89 & $+1.37$ & +1.18 to +1.57 & 20/20 & $<2\times10^{-6}$ \\
TimesFM~2.5 & boost variant & 20 & 0.52 & 0.70 & $+0.19$ & +0.17 to +0.22 & 20/20 & $<2\times10^{-6}$ \\
TimesFM~2.5 & override variant & 20 & 0.52 & 0.67 & $+0.15$ & +0.14 to +0.16 & 20/20 & $<2\times10^{-6}$ \\
TimesFM~2.5 & capped policy & 20 & 0.52 & 0.36 & $-0.16$ & $-0.20$ to $-0.10$ & 2/20 & 0.0004 \\
TimesFM~2.5 & capped policy, spike check live & 20 & 0.52 & 0.58 & $+0.06$ & $-0.07$ to +0.21 & 11/20 & 0.824 \\
TimesFM~2.5 & Huber-only & 20 & 0.52 & 0.62 & $+0.10$ & +0.06 to +0.14 & 18/20 & 0.0004 \\
TimesFM~2.5 & full gate, no soft reset & 20 & 0.52 & 0.32 & $-0.20$ & $-0.22$ to $-0.18$ & 1/20 & $<5\times10^{-5}$ \\
TimesFM~2.5 & detector-stride-96 gate & 20 & 0.52 & 0.27 & $-0.25$ & $-0.39$ to $-0.06$ & 2/20 & 0.0004 \\
TimesFM~2.5 & detector-stride-96, capped response & 20 & 0.52 & 0.23 & $-0.29$ & $-0.39$ to $-0.16$ & 2/20 & 0.0004 \\
\bottomrule
\end{tabular}
\end{table}

Across every variant and cell saved for these campaigns, 284 run records are marked diverged, and
their guard steps concentrate on two events. The weather transient accounts for 181 records at
step 18{,}536, 15 at step 18{,}537, and 10 at step 18{,}538. A second family accounts for 15
records at step 12{,}911 and 11 at step 12{,}912, all of them the two SGD variants on traffic under
immediate reveal, on TiRex and Chronos-2. Nine of the 284 trigger at step 200, the
divergence guard's own first check, so the guard's onset step accounts for 9 of 284 records. The
full gate without its soft reset and the combined variant add 13, 8 at the weather transient, 4 at step 200, and 1 at step
8{,}983 on ETTm1, none on a clean stream.

\begin{table}[t]
\centering
\caption{Divergence counts over 450 runs per variant (3 backbones $\times$ 5 streams $\times$ 3
seeds $\times$ 10 contamination settings), immediate versus delayed reveal. The middle four rows
are the variants introduced in Sections~\ref{sec:frozenr} and~\ref{sec:cutoff}, the next three
the single-component variants of Section~\ref{sec:attribution}, and the last two the full gate
without its soft reset (Section~\ref{sec:attribution}) and the combined variant of Section~\ref{sec:cutoff}.}
\label{tab:divergence}
\footnotesize
\begin{tabular}{lrr}
\toprule
variant & immediate & delayed \\
\midrule
full gate & 60 & 2 \\
SGD surrogate & 6 & 0 \\
ungated SGD & 20 & 17 \\
ungated filter & 0 & 1 \\
Huber-only & 2 & 3 \\
fixed cutoff & 0 & 2 \\
frozen scale & 60 & 2 \\
capped policy & 1 & 2 \\
capped policy, spike check live & 47 & 2 \\
detector-stride-96 gate & 6 & 2 \\
soft-reset variant & 60 & 2 \\
boost variant & 1 & 2 \\
override variant & 0 & 2 \\
full gate, no soft reset & 1 & 3 \\
detector-stride-96, capped response & 7 & 2 \\
\bottomrule
\end{tabular}
\end{table}

\begin{figure}[t]
\centering
\includegraphics[width=0.85\linewidth]{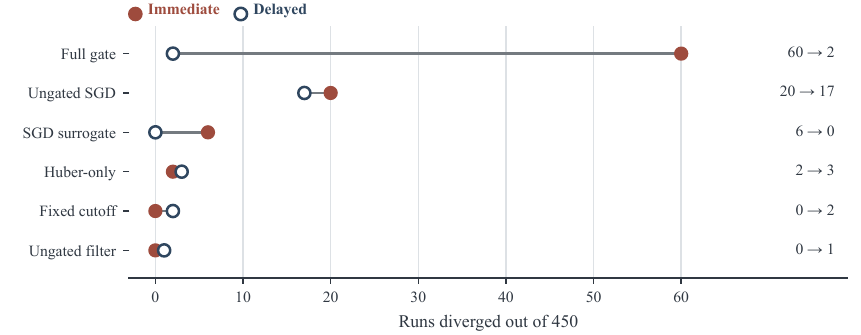}
\caption{Runs diverged out of 450 per variant (3 backbones $\times$ 5 streams $\times$ 3 seeds
$\times$ 10 contamination cells), immediate reveal filled and delayed reveal hollow. The full
gate falls from 60 to 2. Eight Kalman divergences remain under delayed reveal, two full
gate, one ungated filter, three Huber-only, and two fixed cutoff. Seven of the eight are one
cell family, weather at seed 1 under the 5 percent spike injector with guard
onset at step 200, and the eighth is the Huber-only variant on ETTm1 under the packet-freeze injector.}
\label{fig:dumbbell}
\end{figure}

\section{Full degradation tables}
\label{app:tables}

Figure~\ref{fig:spikeprofile} in the main text displays the spike injector at 5 percent
across five streams, four variants, and both feedback protocols, under the survivor-only
convention of Appendix~\ref{app:repro}. Tables~\ref{app:tab:spike-full},
\ref{app:tab:level-shift-full}, and \ref{app:tab:stuck-full} below give that figure's
numeric record and extend the identical convention to the level-shift and packet-freeze injectors,
at the same 5 percent contamination fraction and the same four variants. Those variants are the
ungated filter, the Huber-only variant, the full gate, and the fixed-cutoff variant.
Table~\ref{app:tab:later-variants} gives the same spike cell for the later variants. The frozen-scale variant's delayed-reveal values on the same cell are ETTh2
$+0.50$, ETTm1 $+1.49$, weather $-0.30\,[7/9]$, electricity $+0.12$, and traffic $+0.24$.

\begin{table}[t]
\centering
\caption{Spike at 5 percent, survivor-only paired degradation ratios for the frozen-scale variant of
Section~\ref{sec:frozenr}, the two capped policies and the detector-stride-96 gate of
Section~\ref{sec:cutoff}, the full gate without its soft reset (Section~\ref{sec:attribution}), and
the combined variant of Section~\ref{sec:cutoff}, pooled over three backbones and three seeds (9 pairs per
stream and setting), under immediate reveal for all six and under delayed reveal for the five later ones. Check live is the capped policy with spike classification live during recovery, no reset the full gate
without its soft reset, and combined the detector-stride-96 gate with the capped response. Bracketed counts
mark cells where fewer than 9 of 9 pairs survived.}
\label{app:tab:later-variants}
\resizebox{\textwidth}{!}{\scriptsize
\setlength{\tabcolsep}{2.2pt}
\begin{tabular}{lrrrrrrrrrrr}
\toprule
& \multicolumn{6}{c}{immediate} & \multicolumn{5}{c}{delayed 96} \\
\cmidrule(lr){2-7}\cmidrule(lr){8-12}
stream & frozen & capped & check live & thinned & no reset & combined & capped & check live & thinned & no reset & combined \\
\midrule
ETTh2 & $-0.03$ & $+0.17$ & $+0.15$ & $+0.05$ & $+0.21$ & $+0.05$ & $+0.04$ & $-0.05$ & $-0.01$ & $+0.15$ & $-0.01$ \\
ETTm1 & $+1.85$ & $+4.68$ & $+1.45$ & $-0.01$ & $+5.70$ & $-0.01$ & $+1.38$ & $-0.07$ & $+0.02$ & $+2.36$ & $+0.02$ \\
weather & $+0.18$\,[3/9] & $+0.58$ & $-0.33$\,[4/9] & $-0.02$ & $+0.42$\,[8/9] & $+0.01$ & $+0.15$\,[7/9] & $+0.30$\,[7/9] & $+0.25$\,[7/9] & $+0.06$\,[7/9] & $+0.12$\,[7/9] \\
electricity & $-0.12$ & $-0.19$ & $+0.05$ & $+0.00$ & $-0.23$ & $+0.00$ & $+0.33$ & $-0.05$ & $-0.00$ & $+0.37$ & $-0.00$ \\
traffic & $+0.19$ & $-0.00$ & $+0.21$ & $+0.00$ & $+0.10$ & $+0.00$ & $+0.20$ & $+0.05$ & $-0.00$ & $+0.32$ & $-0.00$ \\
\bottomrule
\end{tabular}}
\end{table}

\begin{table}[t]
\centering
\caption{Spike at 5 percent, survivor-only paired degradation ratios from the completed forecast
cache, immediate versus delayed reveal. A pair enters a cell's mean only when both runs finish;
bracketed counts mark fewer than 9 of 9 surviving pairs.}
\label{app:tab:spike-full}
\resizebox{\textwidth}{!}{\footnotesize
\begin{tabular}{lrrrrrrrr}
\toprule
& \multicolumn{4}{c}{immediate} & \multicolumn{4}{c}{delayed} \\
\cmidrule(lr){2-5} \cmidrule(lr){6-9}
stream & ungated filter & Huber-only & full gate & fixed cutoff & ungated filter & Huber-only & full gate & fixed cutoff \\
\midrule
ETTh2 & $+3.32$ & $+0.15$ & $+0.23$ & $+0.05$ & $+2.41$ & $+0.05$ & $+0.18$ & $-0.01$ \\
ETTm1 & $+2.70$ & $+0.63$ & $+6.81$ & $-0.01$ & $+4.80$ & $-0.19$ & $+3.18$ & $+0.01$ \\
weather & $+0.01$ & $-0.07$ & $+0.03$\,[3/9] & $-0.00$ & $+0.02$\,[8/9] & $-0.16$\,[7/9] & $-0.20$\,[7/9] & $+0.02$\,[7/9] \\
electricity & $+0.04$ & $+0.01$ & $-0.31$ & $+0.00$ & $+0.47$ & $-0.01$ & $+0.54$ & $-0.00$ \\
traffic & $+3.73$ & $+0.10$ & $+0.35$ & $+0.00$ & $+2.64$ & $+0.26$ & $+0.43$ & $-0.00$ \\
\bottomrule
\end{tabular}}
\end{table}

\begin{table}[t]
\centering
\caption{Level shift at 5 percent, survivor-only paired degradation ratios from the completed
forecast cache, immediate versus delayed reveal. A pair enters a cell's mean only when both runs
finish; bracketed counts mark fewer than 9 of 9 surviving pairs.}
\label{app:tab:level-shift-full}
\resizebox{\textwidth}{!}{\footnotesize
\begin{tabular}{lrrrrrrrr}
\toprule
& \multicolumn{4}{c}{immediate} & \multicolumn{4}{c}{delayed} \\
\cmidrule(lr){2-5} \cmidrule(lr){6-9}
stream & ungated filter & Huber-only & full gate & fixed cutoff & ungated filter & Huber-only & full gate & fixed cutoff \\
\midrule
ETTh2 & $+7.79$ & $+16.98$ & $+19.03$ & $+10.04$ & $+1.41$ & $+3.76$ & $+1.18$ & $+2.20$ \\
ETTm1 & $+15.02$ & $+10.60$ & $+13.50$ & $+4.15$ & $+1.86$ & $+1.81$ & $+1.71$ & $+2.34$ \\
weather & $-0.00$ & $-0.00$ & $+0.04$\,[3/9] & $-0.05$ & $-0.00$ & $+0.20$ & $+0.09$ & $-0.01$ \\
electricity & $+1.16$ & $+1.13$ & $+1.20$ & $+2.20$ & $+0.73$ & $+0.78$ & $+0.82$ & $+1.47$ \\
traffic & $+9.54$ & $+18.35$ & $+10.57$ & $+0.97$ & $+5.94$ & $+7.63$ & $+5.84$ & $+0.65$ \\
\bottomrule
\end{tabular}}
\end{table}

\begin{table}[t]
\centering
\caption{Packet-freeze contamination at 5 percent, survivor-only paired degradation ratios from
the completed forecast cache, immediate versus delayed reveal. A pair enters a cell's mean only
when both runs finish; bracketed counts mark fewer than 9 of 9 surviving pairs.}
\label{app:tab:stuck-full}
\resizebox{\textwidth}{!}{\footnotesize
\begin{tabular}{lrrrrrrrr}
\toprule
& \multicolumn{4}{c}{immediate} & \multicolumn{4}{c}{delayed} \\
\cmidrule(lr){2-5} \cmidrule(lr){6-9}
stream & ungated filter & Huber-only & full gate & fixed cutoff & ungated filter & Huber-only & full gate & fixed cutoff \\
\midrule
ETTh2 & $+0.22$ & $+3.24$ & $+8.29$ & $+12.26$ & $+0.13$ & $+0.42$ & $+0.85$ & $+1.97$ \\
ETTm1 & $+0.26$ & $+11.60$ & $+16.98$ & $+13.76$ & $+0.11$ & $+3.80$\,[8/9] & $+3.97$ & $+7.27$ \\
weather & $+0.00$ & $+0.00$ & $-0.02$\,[3/9] & $-0.00$ & $+0.00$ & $+0.02$ & $+0.02$ & $+0.00$ \\
electricity & $-0.08$ & $+1.41$ & $+1.74$ & $+1.72$ & $+0.36$ & $+0.87$ & $+0.55$ & $+0.94$ \\
traffic & $+1.75$ & $+7.18$ & $+6.14$ & $+6.38$ & $+1.03$ & $+4.50$ & $+3.31$ & $+4.15$ \\
\bottomrule
\end{tabular}}
\end{table}

\begin{figure}[t]
\centering
\includegraphics[width=\linewidth]{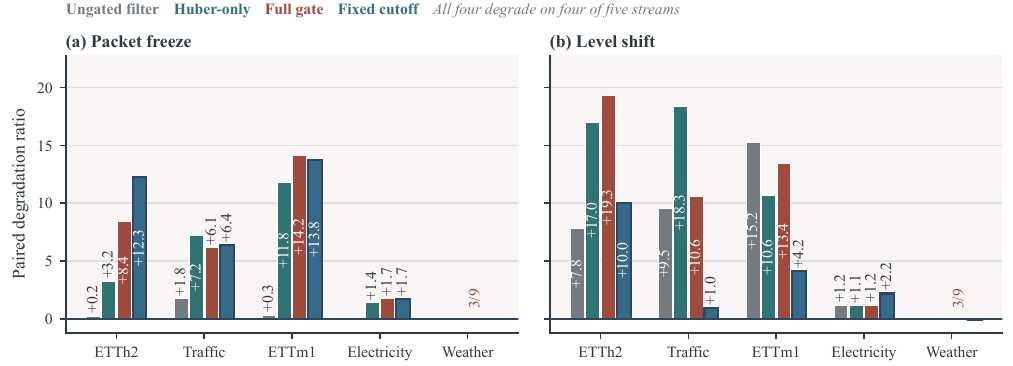}
\caption{The two negative results under immediate reveal, survivor-only ratios at 5 percent
contamination. (a) The packet-freeze injector degrades every gated variant, the fixed cutoff included,
beyond the ungated filter on the four non-weather streams. (b) The level-shift injector
degrades every variant on the four non-weather streams. Bracketed counts mark cells where fewer
than 9 of 9 pairs survived, the full gate's weather cells in both panels.}
\label{fig:negatives}
\end{figure}

The weather cells fall outside the sustained-contamination negatives of
Section~\ref{sec:cutoff}. Every variant stays near zero there under both injectors, the fixed cutoff
at $-0.0001$ under the packet-freeze injector. Outside the Kalman variants, the SGD surrogate's immediate
traffic freeze cell reads $-0.21$ over the 5 of its 9 pairs that survived. Under the
level-shift injector at 5 percent the SGD variants on traffic lose most of their pairs to
divergence.

Contamination fractions of 1 percent and 10 percent exist in the same result files and are
not tabulated here. The per-step transient trace of the component ablation at the weather event
is carried by Figure~\ref{fig:anatomy} in the main text.

Table~\ref{app:tab:sustained-all} extends the packet-freeze and level-shift tables to every variant under
immediate reveal. Section~\ref{sec:cutoff} reads the ETTh2 column of the level-shift half, from
the ungated filter's $+7.79$ to the frozen-scale variant's $+34.36$, as degrading all thirteen
tabulated variants, and quotes the fixed cutoff's freeze row, $+1.72$ to $+13.76$.

\begin{table}[t]
\centering
\caption{Survivor-only paired degradation ratios at 5 percent under immediate reveal for every
variant, packet-freeze and level-shift injectors, the four non-weather streams, 9 pairs per cell (three
backbones, three seeds). Bracketed counts mark cells
with fewer than 9 surviving pairs. The four primary variants repeat Tables~\ref{app:tab:stuck-full}
and~\ref{app:tab:level-shift-full}, and the detector-stride-96 gate uses its matched configuration.}
\label{app:tab:sustained-all}
\scriptsize
\setlength{\tabcolsep}{2.5pt}
\begin{tabular}{lrrrrrrrr}
\toprule
& \multicolumn{4}{c}{freeze} & \multicolumn{4}{c}{level shift} \\
\cmidrule(lr){2-5}\cmidrule(lr){6-9}
variant & electricity & ETTh2 & ETTm1 & traffic & electricity & ETTh2 & ETTm1 & traffic \\
\midrule
ungated filter & $-0.08$ & $+0.22$ & $+0.26$ & $+1.75$ & $+1.16$ & $+7.79$ & $+15.02$ & $+9.54$ \\
Huber-only & $+1.41$ & $+3.24$ & $+11.60$ & $+7.18$ & $+1.13$ & $+16.98$ & $+10.60$ & $+18.35$ \\
full gate & $+1.74$ & $+8.29$ & $+16.98$ & $+6.14$ & $+1.20$ & $+19.03$ & $+13.50$ & $+10.57$ \\
fixed cutoff & $+1.72$ & $+12.26$ & $+13.76$ & $+6.38$ & $+2.20$ & $+10.04$ & $+4.15$ & $+0.97$ \\
ungated SGD & $+0.19$ & $+0.52$ & $+0.44$ & $+0.02$\,[3/9] & $+1.64$ & $+12.84$ & $+15.04$ & $+0.12$\,[3/9] \\
SGD surrogate & $+0.15$ & $+0.49$ & $+0.42$ & $-0.21$\,[5/9] & $+1.70$ & $+11.31$ & $+13.72$ & $+0.15$\,[6/9] \\
soft-reset variant & $+1.60$ & $+7.88$ & $+8.08$ & $+4.35$ & $+1.27$ & $+8.84$ & $+11.17$ & $+10.84$ \\
boost variant & $+2.11$ & $+7.70$ & $+13.09$ & $+7.46$ & $+1.39$ & $+11.58$ & $+18.27$ & $+14.22$ \\
override variant & $+1.43$ & $+4.92$ & $+8.78$ & $+5.35$ & $+1.10$ & $+17.96$ & $+11.70$ & $+12.30$ \\
frozen scale & $+0.12$ & $+0.59$ & $+0.54$ & $+3.71$ & $+1.20$ & $+34.36$ & $+8.35$ & $+10.71$ \\
capped policy & $+2.15$ & $+7.86$ & $+15.83$ & $+5.80$ & $+2.49$ & $+29.07$ & $+13.46$ & $+10.40$ \\
capped policy, spike check live & $+1.99$ & $+7.98$ & $+12.36$ & $+7.24$ & $+1.29$ & $+13.93$ & $+11.78$ & $+1.81$ \\
detector-stride-96 gate & $+1.00$ & $+6.59$ & $+1.71$ & $+3.03$ & $+1.38$ & $+10.02$ & $+2.88$ & $+1.43$ \\
\bottomrule
\end{tabular}
\end{table}

\section{Adapter gain over the frozen backbone and its retention}
\label{app:gain}

Every degradation ratio in the paper is taken to the ungated filter, so a reference variant that
never adapts, the frozen backbone's own forecast through the identity calibration, states what
the adapter is worth and lets every other variant be read in units of that worth. The variant reads no
label, so it is immune to label contamination, and one clean cell per backbone and stream is its
reference for every cell. Gain is the quantity of~\eqref{eq:metrics}, taken for the ungated filter's clean run; retention reads every other variant's run $L$ in units of that gain, with $L_{\text{ungated}}$ the ungated filter's clean run,
\begin{equation}
\text{retention} = \frac{L_{\text{frozen}} - L}{L_{\text{frozen}} - L_{\text{ungated}}},
\label{eq:retention}
\end{equation}
so one keeps the whole clean gain, zero is no better than
never adapting, and a negative value is worse. Retention is undefined where the ungated filter has
no gain, which excludes weather and every delayed cell; no run entering
Table~\ref{app:tab:retention-full} diverged. Each variant is paired with the clean ungated run of the
same backbone and protocol.

\begin{table}[t]
\centering
\caption{Conditional gain retention for all thirteen variants over the 36 diagnostic
immediate-reveal runs with positive ungated clean gain, excluding weather. Entries are medians,
with the number below zero in parentheses; contamination is 5 percent. This selected cohort is a
conditional decision analysis, not evidence of general effectiveness.}
\label{app:tab:retention-full}
\small
\setlength{\tabcolsep}{6pt}
\begin{tabular}{lcccc}
\toprule
variant & clean & vector spike & vector shift & vector freeze \\
\midrule
ungated filter & 1.00 (0) & $-0.35$ (19) & $-2.96$ (36) & 0.89 (9) \\
Huber-only & 0.90 (0) & 0.78 (3) & $-3.46$ (35) & $-2.49$ (32) \\
full gate & 1.06 (0) & 0.71 (12) & $-2.82$ (35) & $-3.78$ (36) \\
fixed cutoff & 0.75 (3) & 0.74 (1) & $-2.87$ (26) & $-8.02$ (36) \\
soft-reset variant & 1.06 (0) & 0.82 (7) & $-3.02$ (35) & $-3.12$ (36) \\
boost variant & 1.05 (0) & 0.74 (13) & $-3.18$ (34) & $-3.67$ (36) \\
override variant & 0.94 (0) & 0.84 (0) & $-2.92$ (33) & $-3.09$ (34) \\
frozen scale & 0.98 (0) & 0.96 (4) & $-3.31$ (33) & 0.62 (9) \\
capped policy & 1.04 (0) & 0.96 (3) & $-2.82$ (32) & $-3.46$ (35) \\
capped, spike check live & 0.98 (0) & 0.85 (3) & $-2.61$ (26) & $-4.40$ (36) \\
gate, no soft reset & 1.02 (0) & 0.91 (8) & $-2.76$ (31) & $-3.33$ (35) \\
fallback, first version & 0.83 (0) & 0.46 (9) & $-2.90$ (36) & 0.82 (9) \\
fallback, second version & 0.98 (0) & 0.03 (9) & $-2.71$ (36) & 0.88 (9) \\
\bottomrule
\end{tabular}
\end{table}

Under immediate reveal the adapter removes 15 to 83 percent of the frozen backbone's error on 12
of 15 cells and multiplies it 2.9 to 22 times on weather on every backbone
(Table~\ref{app:tab:gain}); the weather transient of Section~\ref{sec:attribution} is therefore an
adapter-made cost that not adapting avoids entirely, and no variant of the main campaign reaches the
frozen backbone's 0.125 to 0.159 there; the variant of Appendix~\ref{app:causal} that pins the
forgetting factor and holds the noise scale does, at a gain of 0.47. Under causal delayed reveal the adapter is worse than the frozen
backbone on every cell, by 2 to 96 percent of the frozen error on the four non-weather streams, so
the protocol under which the divergences vanish is one under which this adapter has no value to
protect.

Over the 36 immediate runs with a gain (Table~\ref{app:tab:retention-full}), the gated variants keep
0.90 to 1.06 of the gain on clean streams and the fixed cutoff 0.75, its price for refusing a
third of the updates. Under 5 percent spikes the ungated filter keeps $-0.35$, worse than never
adapting in 19 of 36 runs, the capped policy 0.96, the full gate 0.71 with 12 runs below the
frozen backbone, and the fixed cutoff 0.74 with one such run. Under the feedback-vector shift every variant loses about
three times the gain. Under the packet-freeze intervention the ungated filter keeps 0.89 and the frozen-scale
gate 0.62. The full gate, fixed cutoff, capped policy, and active component-response variants lose
2.8 to 8 times the gain, so those response variants, and the fixed cutoff most, turn a mild cost into
a large one. Across the 15 immediate cells the adapter's gain rises with the residual lag-1
autocorrelation of the frozen backbone's own errors (Spearman 0.64 across 15 clustered cells; traffic is lowest at
0.72 to 0.74 and gain 0.15 to 0.21, ETTh2 and ETTm1 highest at 0.97 to 0.98 and 0.60 to 0.83,
weather the exception), the cells are clustered by stream, and the correlation is descriptive.

The first version of the fallback mixture, the one fixed before the runs, fails its own rule. Over the 36 immediate runs with a gain it keeps 0.83 of the gain on clean streams, no run at 0.9, and 0.46 under spikes with 9 runs below the frozen backbone against 19 for the ungated
filter (Table~\ref{app:tab:retention-full}). Under causal delayed reveal it is worse than the frozen
backbone on 36 of 45 runs, where the ungated filter is worse on all 45. On clean weather under
immediate reveal it finishes at 1.02, 1.80, and 0.19 on TiRex, Chronos-2, and TimesFM~2.5
against the frozen 0.140, 0.159, and 0.125. Its weight on
the adapter averages 0.36 to 0.69 over every cell and 0.49 on weather; as configured, its regret
bound does not bind, because the exp-concavity behind the rate $1/(8B^2)$ needs every forecast and
label inside the range the rate was set for, which the adapter's weather transient leaves, and a fixed share of $10^{-3}$ adds about $10^{-3}\,n/\eta$ to the static regret bound, 2.5 per step on weather
against a bound of 0.04 per step. The second version of Appendix~\ref{app:spec} sets the rate from
the frozen residual's training range, the scale the mixture's loss is measured on, clips to the moving band, and drops the fixed share.

The second version repairs the guarantee and confirms the delayed result. Its band binds on at most 0.4 percent of steps and the fed label leaves it on at most 0.14 percent, so the bound of
$8D^2\log 2$ on the cumulative excess over the better expert is in force, and the measured excess
stays within it on all 90 clean runs under both protocols, at most 0.54 of the bound, with the mixture
below both experts on 21 of them. Under immediate reveal it keeps 0.98 of the adapter's clean gain at the median and at least 0.90
on each of the 36 runs, and it finishes below the frozen backbone on all 45 runs. On clean
weather it finishes at 0.071, 0.077, and 0.066 on TiRex, Chronos-2, and TimesFM~2.5, against the
frozen 0.140, 0.159, and 0.125 and the ungated 2.03, 3.58, and 0.36. Under 5 percent spikes it
keeps 0.03 of the gain at the median with 9 runs below the frozen backbone, so it protects the
floor and keeps little of the ceiling. Under the level shift its weights follow the fed labels
and it loses 2.71 times the gain on every cell (Table~\ref{app:tab:retention-full}). Under
causal delayed reveal its clean error is within 0.3 percent of the frozen backbone's at the median
and within 17.7 percent at worst, on weather, against the first version's 863 percent; it is worse
than the frozen backbone on 33 of 45 runs by at most that margin, and its final weight on the
adapter is 0.000 on ETTh2 and ETTm1 on every backbone and 0.003 to 0.360 elsewhere. The mixture
learns what Table~\ref{app:tab:gain} states, that this adapter has no clean value under causal
reveal, and pays a bounded price for learning it (Table~\ref{app:tab:fallback2}); its regime tracking is
in Table~\ref{app:tab:regime-causal}.

\begin{table}[t]
\centering
\caption{Second-version fallback mixture on the 45 clean runs of each reveal. Gain is one minus the
mixture's clean MSE over the frozen backbone's; excess is the mixture's cumulative clipped fed-label
loss less the better expert's, over its bound $8D^2\log 2$; below both counts runs where the mixture
beats both experts; the last column is the range of the final weight on the adapter.}
\label{app:tab:fallback2}
\scriptsize
\setlength{\tabcolsep}{4pt}
\begin{tabular}{lrrrrrrr}
\toprule
& \multicolumn{3}{c}{gain over frozen} & & & & \\
\cmidrule(lr){2-4}
reveal & median & min & max & worse & excess / bound & below both & final weight \\
\midrule
immediate & 0.493 & 0.131 & 0.824 & 0 of 45 & 0.46 & 9 & 0.708--1.000 \\
delayed 96 & $-0.003$ & $-0.177$ & 0.055 & 33 of 45 & 0.54 & 12 & 0.000--0.360 \\
\bottomrule
\end{tabular}
\end{table}

\begin{table}[t]
\centering
\caption{Frozen backbone's clean MSE per backbone and stream, and the ungated adapter's gain over
it, one minus the ratio of clean MSEs, under immediate and delayed reveal (three seeds, identical
to three decimals because the clean ungated run is not replicated across seeds). All rows use
the completed forecast cache.}
\label{app:tab:gain}
\footnotesize
\begin{tabular}{llrrr}
\toprule
backbone & stream & frozen MSE & gain, immediate & gain, delayed \\
\midrule
TiRex & ETTh2 & 0.316 & 0.61 & $-0.73$ \\
TiRex & ETTm1 & 0.080 & 0.83 & $-0.93$ \\
TiRex & weather & 0.140 & $-13.51$ & $-1.05$ \\
TiRex & electricity & 0.502 & 0.23 & $-0.16$ \\
TiRex & traffic & 0.126 & 0.20 & $-0.30$ \\
Chronos-2 & ETTh2 & 0.377 & 0.68 & $-0.60$ \\
Chronos-2 & ETTm1 & 0.084 & 0.79 & $-0.30$ \\
Chronos-2 & weather & 0.159 & $-21.45$ & $-35.12$ \\
Chronos-2 & electricity & 0.667 & 0.33 & $-0.02$ \\
Chronos-2 & traffic & 0.110 & 0.21 & $-0.33$ \\
TimesFM~2.5 & ETTh2 & 0.321 & 0.63 & $-0.96$ \\
TimesFM~2.5 & ETTm1 & 0.082 & 0.77 & $-0.61$ \\
TimesFM~2.5 & weather & 0.125 & $-1.89$ & $-3.14$ \\
TimesFM~2.5 & electricity & 0.520 & 0.25 & $-0.22$ \\
TimesFM~2.5 & traffic & 0.113 & 0.15 & $-0.24$ \\
\bottomrule
\end{tabular}
\end{table}

\section{Causal reveal, residual predictability, and the adapter's forgetting}
\label{app:causal}

The composite coordinate reveal of \citet{kim2025tafas} and \citet{lau2025dsof} uses the packet in
Table~\ref{app:tab:feedback-packets}. At reveal step $r$, its horizon-$h$ block pairs the newly
observed label with the forecast issued at $r-h$, so that block's lag is exactly $h$. The first 96
steps are predict-only because the implementation forms a complete 96-block packet before an
update. Immediate reveal instead uses the 96 future labels of origin $r$, 95 of which also occur
in the next 95 origins, whereas delayed reveal uses the complete label of origin $r-96$. All three
updates use the state and predicted covariance at reveal step $r$. All 96 composite rows forecast
the same physical target time, so their forecast errors are strongly dependent. The implementation
nevertheless models the innovation covariance across horizon blocks as diagonal. The results below
therefore describe this composite diagonal-covariance update, not every coordinate-reveal estimator.
A second variant widens each horizon's block from $[f_{t,h}\;\;1]$ to $[f_{t,h}\;\;1\;\;r_{t-1}[1]\;\;r_{t-h}[h]\;\;\bar r_t]$, with $r_t = y_t - f_t$ the frozen residual vector of origin $t$ as in Appendix~\ref{app:valid}, so $r_{t-j}[j]$ is the residual of the forecast made $j$ origins earlier for the label observed at $t$, and $\bar r_t$ is the mean of $r_{t-j}[j]$ over $j = 1$ to $24$; each is divided by its
training-block root mean square and given a prior variance of 0.1. The features are therefore the
freshest observed errors, where \citet{dai2026orca} condition a residual predictor on the model's
input and output and wait for the full horizon. Both variants run as their own configurations, and the frozen backbone, the ungated filter, and the fallback mixture run
under each; the gated variants are not wired to the composite coordinate reveal. The rule for the widened variant, fixed before the runs, asked for a clean gain of at least 0.10 over the frozen backbone on at least 10 of 15
cells under the composite coordinate reveal with no diverged seed.

Under the composite coordinate reveal the adapter without features is worse than the frozen backbone on
14 of 15 cells, median gain $-0.34$, from $-0.58$ on TiRex ETTh2 to $+0.04$ on Chronos-2 ETTm1.
The loss holds in every horizon band, $-0.13$ at horizons 1 to 8, $-0.19$ at 9 to 24, $-0.26$ at
25 to 48, and $-0.34$ at 49 to 96 (Table~\ref{app:tab:causal}). Delayed reveal gives $-0.60$ on
15 of 15 and immediate reveal $+0.25$ on 12 of 15. The widened variant trips the divergence guard in 85 of 90 clean seeds, at or
before the guard's first check at step 200 in 42 of 45 under delayed reveal and after 124 to
13{,}789 steps under the composite coordinate reveal, with a non-finite hyperparameter state at step 0 on ETTm1 in 12
seeds, and the five surviving seeds finish 2.1 to 28.5 times the frozen error, so the rule fails at
0 of 15. Wrapped in the fallback mixture the widened variant forecasts as the frozen backbone, median
gain $-0.010$ and a regime ratio of 1.000.

An offline oracle separates information from estimation. On each cell it fits, on the training
block, a per-horizon affine correction of the frozen forecast, the same correction with the three
fresh residual features, the same with all 96 lagged residuals, and a ceiling that conditions on the
label adjacent to the target, the information immediate reveal supplies, and scores each on the test
block against the frozen backbone (Table~\ref{app:tab:oracle}). The adjacent label recovers 84 to
95 percent of the frozen error on four streams and 53 to 57 percent on traffic. Everything causally
available recovers at most 23 percent, reaches 0.10 on 4 of 15 cells with all 96 lags (TimesFM~2.5 traffic prints as 0.100 and
sits at 0.0997), and is
negative on weather with all lags; on ETTm1 the affine correction alone is worse than the frozen
backbone by 5 to 11 percent. A ridge recursive least-squares fit of the same corrections at
forgetting factor 0.999, run online under the composite coordinate reveal, stays within $-0.01$ to $+0.14$ of
the frozen backbone on ETTh2, electricity, and traffic and becomes unstable on ETTm1 and weather, further still with the
features, with no Kalman filter and no hyperparameter learning involved; its rule, fixed in advance as a gain of 0.10 on 10 of 15 cells, fails at 0 of 15. The adjacent label therefore carries enough
information to account for the scale of the immediate-reveal gain, and the causal residual summaries
tested here carry little on these streams; the oracle is a training-block fit and an offline
recursion, so it bounds information and does not reproduce the live adapter's estimation. \citet{dai2026orca} report single-digit to 16 percent error reductions for a residual
predictor on the same backbones and streams under the delayed full-horizon reveal, with a different
horizon average, context, and normalization; the negative results here are for this adapter family.

The adapter's own forgetting is part of the mechanism for its distance from the oracle, and only
part. Its learned forgetting factor sits at a median of 0.949 to 0.997 on the four non-weather
streams and at the floor of 0.90 on weather, under immediate and composite coordinate reveal alike, an effective memory of 20 to 330 origins against 1{,}000 at 0.999, and the hyperparameter learner moves it by
less than 0.02 between the two reveals on every cell (Table~\ref{app:tab:forgetting}). Under
immediate reveal the learned observation noise shrinks to 0.07 to 0.59 of its start on ETTm1,
electricity, and weather, and under the composite coordinate reveal it grows 1.5 to 1.9 times on weather.

A fixed-forgetting ablation separates the factor from the noise scale and from the reveal
(Table~\ref{app:tab:fixed-forgetting}). Six variants on the same ungated runner pin the forgetting factor at the oracle's 0.999 for the
whole run, hold the observation-noise scale or the process-noise scale at its warm-up value over
the test block, or combine these. The last variant holds the factor and both scales, so its test
block runs with no hyperparameter learning. A check of the saved trajectories confirms that no pinned or
held quantity moved in any run of either campaign. Pinning the factor raises the
clean gain on every one of the 12 non-weather cells, by 0.02 to 0.30 under the composite coordinate reveal and
0.10 to 0.73 under delayed reveal, lifts the pooled gain from $-0.34$ to $-0.15$ and from $-0.60$ to
$-0.05$, and removes every divergence. No run with the pinned factor diverges under any reveal, 270
runs, where the learned factor diverges on 9 of 9 weather seeds under the composite coordinate reveal. The rule fixed in advance, a gain of at least $-0.05$ on 10 of the 12 non-weather cells under the
composite coordinate reveal, still fails at 4 of 12. The pinned adapter loses 16 to 37 percent on ETTh2 and up
to 15 on electricity and traffic where the ridge recursion at the same factor sits within $-0.01$ to
$+0.14$; that comparator differs in its estimator, its state model, and its features at once, so it
locates a remainder without attributing it.

Three further variants hold the process-noise scale at its warm-up value, with the factor learned, with the factor pinned, and with the factor pinned and both
scales held, so the last runs the test block with no hyperparameter learning at the oracle's rate. That variant
reaches within $-0.05$ of the frozen backbone on 4 of the 12 non-weather cells under the composite coordinate reveal, at a pooled gain of $-0.13$ over the 15 cells with no diverged seed, and its rule, the same floor on 10 of 12, fails. Holding the process noise on its own moves the learned variant by $-0.06$ to $+0.11$ on the
non-weather cells and leaves its nine weather divergences in place. Against the variant with the
factor pinned alone, holding both scales adds $-0.02$ to $+0.11$ under the composite coordinate reveal and
$-0.02$ to $+0.12$ under delayed reveal, medians $+0.02$ and $+0.03$. On weather under the
composite coordinate reveal it loses 7.5 to 72 times the frozen error, against 2.9 to 21 with the factor
pinned alone. With the scales set at warm-up and no learning over the test block, the state model
itself loses under the reveal, 10 to 28 percent on ETTh2 and up to 17 percent on electricity and
traffic. Holding the observation-noise scale does not help on its own, a median change of
$-0.04$ and $-0.08$ under the two causal reveals, and together with the pinned factor it gives
back part of the factor's gain.

On weather the fast forgetting accounts for the divergence and leaves the failure. At 0.999 the adapter finishes within $-0.02$ to $+0.14$ of
the frozen backbone under delayed reveal and at 2.9 to 21 times its error under the composite coordinate reveal; under immediate reveal the factor alone brings weather to within $-0.33$ to $+0.08$ of the
frozen backbone, and the two pins together turn the learned variant's 2.9 to 22 times into a gain of
0.35 to 0.60, every one of the 15 cells positive.

\begin{table}[t]
\centering
\caption{The ungated adapter's learned hyperparameters on clean streams, median over steps then over
three seeds, range over the three backbones. Forgetting factor $\lambda$ (floor 0.90) and the ratio
of the learned observation-noise scale at the end of the run to its start, under immediate and
composite coordinate reveal.}
\label{app:tab:forgetting}
\footnotesize
\begin{tabular}{lcccc}
\toprule
& \multicolumn{2}{c}{$\lambda$ median} & \multicolumn{2}{c}{noise scale, end over start} \\
\cmidrule(lr){2-3}\cmidrule(lr){4-5}
stream & immediate & coordinate & immediate & coordinate \\
\midrule
ETTh2 & 0.956--0.981 & 0.949--0.979 & 1.27--1.80 & 1.29--1.55 \\
ETTm1 & 0.981--0.988 & 0.994--0.997 & 0.07--0.59 & 0.12--0.61 \\
weather & 0.902--0.905 & 0.900--0.901 & 0.19--0.19 & 1.52--1.85 \\
electricity & 0.994--0.996 & 0.993--0.994 & 0.29--0.59 & 0.30--0.61 \\
traffic & 0.983--0.985 & 0.978--0.982 & 2.20--2.37 & 2.15--2.49 \\
\bottomrule
\end{tabular}
\end{table}

\begin{table}[t]
\centering
\caption{Fixed-forgetting ablation of the ungated adapter. Clean gain over the frozen backbone, one
minus the ratio of clean MSEs, median over seeds then over the three backbones per stream and over
the 15 cells (pooled), by reveal; positive is the number of cells with a positive median gain and
diverged the number of the 45 clean seeds that diverged or reached a non-finite hyperparameter state, a
diverged seed counting as $-\infty$. The learned row is the ungated adapter of
Table~\ref{app:tab:causal}; $\lambda$ 0.999 pins the forgetting factor for the whole run, $R$ held
keeps the observation-noise scale at its warm-up value over the test block, both applies the two,
$Q$ held keeps the process-noise scale at its warm-up value over the test block with the factor and the
observation noise learned, $\lambda$ 0.999, $Q$ held adds the pinned factor, and all held pins the factor
and holds both scales, so the filter runs the test block with no hyperparameter learning; the process noise is
learned in the first four variants of each reveal.}
\label{app:tab:fixed-forgetting}
\scriptsize
\setlength{\tabcolsep}{3.5pt}
\begin{tabular}{llrrrrrrrr}
\toprule
reveal & variant & ETTh2 & ETTm1 & weather & electricity & traffic & pooled & positive & diverged \\
\midrule
immediate & learned & 0.63 & 0.79 & $-13.51$ & 0.25 & 0.20 & 0.25 & 12 & 0 \\
immediate & $\lambda$ 0.999 & 0.52 & 0.71 & $-0.19$ & 0.17 & 0.04 & 0.17 & 13 & 0 \\
immediate & $R$ held & 0.69 & 0.47 & $-13.54$ & 0.27 & 0.28 & 0.29 & 12 & 0 \\
immediate & both & 0.61 & 0.40 & 0.47 & 0.22 & 0.21 & 0.40 & 15 & 0 \\
immediate & $Q$ held & 0.62 & 0.74 & $-13.55$ & 0.27 & 0.21 & 0.27 & 12 & 0 \\
immediate & $\lambda$ 0.999, $Q$ held & 0.40 & 0.65 & $-0.19$ & 0.16 & 0.03 & 0.16 & 13 & 0 \\
immediate & all held & 0.47 & 0.59 & $-0.34$ & 0.17 & 0.05 & 0.17 & 12 & 0 \\
\midrule
coordinate & learned & $-0.54$ & $-0.00$ & $-\infty$ & $-0.24$ & $-0.34$ & $-0.34$ & 1 & 9 \\
coordinate & $\lambda$ 0.999 & $-0.36$ & 0.05 & $-11.78$ & $-0.15$ & $-0.11$ & $-0.15$ & 3 & 0 \\
coordinate & $R$ held & $-0.55$ & 0.03 & $-\infty$ & $-0.36$ & $-0.52$ & $-0.43$ & 3 & 9 \\
coordinate & both & $-0.45$ & 0.09 & $-11.38$ & $-0.23$ & $-0.41$ & $-0.30$ & 4 & 0 \\
coordinate & $Q$ held & $-0.47$ & 0.01 & $-\infty$ & $-0.25$ & $-0.36$ & $-0.36$ & 2 & 9 \\
coordinate & $\lambda$ 0.999, $Q$ held & $-0.20$ & 0.08 & $-9.02$ & $-0.14$ & $-0.09$ & $-0.11$ & 3 & 0 \\
coordinate & all held & $-0.27$ & 0.08 & $-46.73$ & $-0.15$ & $-0.12$ & $-0.13$ & 4 & 0 \\
\midrule
delayed 96 & learned & $-0.73$ & $-0.61$ & $-3.14$ & $-0.16$ & $-0.30$ & $-0.60$ & 0 & 0 \\
delayed 96 & $\lambda$ 0.999 & $-0.44$ & $-0.19$ & 0.04 & $-0.04$ & $-0.04$ & $-0.05$ & 3 & 0 \\
delayed 96 & $R$ held & $-0.86$ & $-0.10$ & $-3.49$ & $-0.25$ & $-0.45$ & $-0.45$ & 1 & 0 \\
delayed 96 & both & $-0.58$ & $-0.01$ & $-0.38$ & $-0.15$ & $-0.32$ & $-0.28$ & 2 & 0 \\
delayed 96 & $Q$ held & $-0.62$ & $-0.72$ & $-3.14$ & $-0.23$ & $-0.31$ & $-0.55$ & 0 & 0 \\
delayed 96 & $\lambda$ 0.999, $Q$ held & $-0.22$ & $-0.14$ & 0.04 & $-0.06$ & $-0.03$ & $-0.06$ & 3 & 0 \\
delayed 96 & all held & $-0.31$ & $-0.12$ & 0.00 & $-0.06$ & $-0.05$ & $-0.06$ & 3 & 0 \\
\bottomrule
\end{tabular}
\end{table}

Excluding weather, the adapter's causal value is regime tracking; on weather it fails. Under the regime shift of Appendix~\ref{app:power} at magnitude 1.0, the post-onset fed-label
error of the adapter without features is 0.16 to 0.73 of the frozen backbone's on the 12
non-weather cells under the composite coordinate reveal and 0.17 to 0.70 under delayed reveal. On weather
under the composite coordinate reveal it is 165 to 323 times the frozen backbone's. There the fallback
mixture caps the ratio at 1.05 to 1.07, and at 1.02 to 1.18 under delayed reveal, while keeping
0.17 to 0.71 elsewhere (Table~\ref{app:tab:regime-causal}).

\begin{table}[t]
\centering
\caption{Clean gain of the ungated adapter over the frozen backbone, one minus the ratio of clean
MSEs, by reveal and feature set, median over seeds then over the three backbones per stream and over
the 15 cells (pooled); positive is the number of cells with a positive median gain and diverged the
number of the 45 clean seeds that diverged or reached a non-finite hyperparameter state. A diverged seed
counts as $-\infty$.}
\label{app:tab:causal}
\scriptsize
\setlength{\tabcolsep}{3.5pt}
\begin{tabular}{llrrrrrrrr}
\toprule
reveal & features & ETTh2 & ETTm1 & weather & electricity & traffic & pooled & positive & diverged \\
\midrule
immediate & none & 0.63 & 0.79 & $-13.51$ & 0.25 & 0.20 & 0.25 & 12 & 0 \\
coordinate & none & $-0.54$ & $-0.00$ & $-\infty$ & $-0.24$ & $-0.34$ & $-0.34$ & 1 & 9 \\
delayed 96 & none & $-0.73$ & $-0.61$ & $-3.14$ & $-0.16$ & $-0.30$ & $-0.60$ & 0 & 0 \\
coordinate & fresh & $-\infty$ & $-\infty$ & $-\infty$ & $-\infty$ & $-\infty$ & $-\infty$ & 0 & 40 \\
delayed 96 & fresh & $-\infty$ & $-\infty$ & $-\infty$ & $-\infty$ & $-\infty$ & $-\infty$ & 0 & 45 \\
\bottomrule
\end{tabular}
\end{table}

\begin{table}[t]
\centering
\caption{Offline oracle of residual predictability, pooled gain over the frozen backbone on the test
block. The first three rows per backbone are least-squares fits on the training block, the fourth
conditions on the label adjacent to the target, the information immediate reveal supplies, and the
last two are ridge recursive least-squares fits at forgetting factor 0.999 run online under the
composite coordinate reveal; $<-1$ marks an unstable fit.}
\label{app:tab:oracle}
\scriptsize
\setlength{\tabcolsep}{4pt}
\begin{tabular}{llrrrrr}
\toprule
correction & backbone & ETTh2 & ETTm1 & weather & electricity & traffic \\
\midrule
affine, training-block fit & TiRex & 0.033 & $-0.070$ & 0.093 & 0.039 & 0.023 \\
 & Chronos-2 & 0.043 & $-0.105$ & 0.210 & 0.108 & 0.010 \\
 & TimesFM~2.5 & 0.027 & $-0.052$ & 0.007 & 0.037 & 0.004 \\
\addlinespace
affine, three fresh residuals & TiRex & 0.055 & 0.012 & 0.045 & 0.069 & 0.030 \\
 & Chronos-2 & 0.068 & $-0.041$ & 0.174 & 0.155 & 0.029 \\
 & TimesFM~2.5 & 0.041 & 0.023 & $-0.029$ & 0.061 & 0.007 \\
\addlinespace
affine, 96 fresh residuals & TiRex & 0.057 & 0.037 & $-0.006$ & 0.136 & 0.101 \\
 & Chronos-2 & 0.059 & $-0.040$ & $-0.053$ & 0.230 & 0.056 \\
 & TimesFM~2.5 & 0.059 & 0.079 & $-0.084$ & 0.159 & 0.100 \\
\addlinespace
affine on the adjacent label & TiRex & 0.945 & 0.944 & 0.871 & 0.838 & 0.571 \\
 & Chronos-2 & 0.933 & 0.924 & 0.873 & 0.858 & 0.557 \\
 & TimesFM~2.5 & 0.945 & 0.940 & 0.871 & 0.840 & 0.528 \\
\addlinespace
recursive affine & TiRex & 0.030 & $<-1$ & $<-1$ & 0.035 & 0.011 \\
 & Chronos-2 & 0.092 & $<-1$ & $<-1$ & 0.138 & $-0.002$ \\
 & TimesFM~2.5 & 0.006 & $<-1$ & $<-1$ & 0.036 & $-0.008$ \\
\addlinespace
recursive affine, three fresh & TiRex & 0.023 & $<-1$ & $<-1$ & 0.043 & 0.021 \\
 & Chronos-2 & 0.082 & $<-1$ & $<-1$ & $-0.116$ & 0.017 \\
 & TimesFM~2.5 & $-0.001$ & $<-1$ & $<-1$ & $-0.046$ & $-0.007$ \\
\bottomrule
\end{tabular}
\end{table}

\begin{table}[t]
\centering
\caption{Regime tracking under the causal reveals, post-onset fed-label MSE of the variant over the frozen
backbone's under the magnitude-1.0 regime shift, median over seeds, for the ungated adapter and the
second-version fallback mixture.}
\label{app:tab:regime-causal}
\footnotesize
\begin{tabular}{llrrrr}
\toprule
& & \multicolumn{2}{c}{coordinate reveal} & \multicolumn{2}{c}{delayed reveal} \\
\cmidrule(lr){3-4}\cmidrule(lr){5-6}
backbone & stream & ungated & fallback & ungated & fallback \\
\midrule
TiRex & ETTh2 & 0.44 & 0.43 & 0.49 & 0.47 \\
TiRex & ETTm1 & 0.28 & 0.27 & 0.53 & 0.46 \\
TiRex & weather & 322.52 & 1.06 & 2.04 & 1.18 \\
TiRex & electricity & 0.73 & 0.71 & 0.70 & 0.68 \\
TiRex & traffic & 0.19 & 0.20 & 0.19 & 0.20 \\
Chronos-2 & ETTh2 & 0.46 & 0.46 & 0.51 & 0.51 \\
Chronos-2 & ETTm1 & 0.26 & 0.25 & 0.36 & 0.35 \\
Chronos-2 & weather & 164.73 & 1.05 & 36.02 & 1.02 \\
Chronos-2 & electricity & 0.68 & 0.67 & 0.63 & 0.61 \\
Chronos-2 & traffic & 0.18 & 0.19 & 0.18 & 0.18 \\
TimesFM~2.5 & ETTh2 & 0.42 & 0.41 & 0.53 & 0.50 \\
TimesFM~2.5 & ETTm1 & 0.27 & 0.26 & 0.43 & 0.38 \\
TimesFM~2.5 & weather & 213.33 & 1.07 & 4.13 & 1.17 \\
TimesFM~2.5 & electricity & 0.68 & 0.66 & 0.67 & 0.65 \\
TimesFM~2.5 & traffic & 0.16 & 0.17 & 0.17 & 0.18 \\
\bottomrule
\end{tabular}
\end{table}

\section{ORCA diagnostic and simplified ports, not head-to-head method rankings}
\label{app:baselines}

The adapter conclusion of Section~\ref{sec:cutoff} is measured on one adapter family. This
appendix evaluates five released refiners on the same 15 cells, the same cached forecasts,
origins, train-block scale, divergence guard, seeds, and reveals, against the same frozen
reference. The five are the refiners evaluated in the released ORCA study~\citep{dai2026orca}: the
ORCA residual refiner, simplified ports of ELF~\citep{lee2025elf}, of the output
calibration of TAFAS \citep{kim2025tafas}, and of DSOF \citep{lau2025dsof}, and its ridge
baseline from the context and base forecast to the residual. Every hyperparameter uses its
released default. The collection phase buffers 3{,}000 training snapshots and, for the TAFAS and
DSOF ports, 300 validation snapshots before the first fit. The ports keep less than the papers.
The ELF port weights its learners by
mean absolute error and starts without the seasonal-naive learner, the TAFAS port keeps the output
calibration module alone and drops the input calibration and the periodicity-aware schedule, and
the DSOF port keeps the fast residual student and the slow replay stream; the optional
confidence router, off at its released default, stays off in every port. All five adapt on complete
horizon-length targets only.

Each cell is one stride-one stream of origins, the training block, the validation block, and the
test block, in order. At every origin the refiner first maps the cached point forecast and the
96-step context to its refined forecast, and then receives every snapshot whose closure step is
the current step; a snapshot is the context, the cached forecast, the refiner's own earlier
forecast, and the full 96-step label of one origin. Warm-up origins close 96 steps after issue
under every reveal, so the three reveals of a cell share their warm-up exactly. Test origins
close at issue under the immediate reveal, which exposes their complete labels noncausally, and
96 steps after issue under the composite coordinate and delayed reveals. The refiners consume complete labels only, so the
composite coordinate and delayed reveals carry identical information for them; the coordinate lane is a cross-process determinism check, and its refined forecasts reproduced the delayed lane's bit for
bit on 225 of 225 runs. Causality is checked from the closure delays every record carries; a second check fails any refiner whose
forecasts reproduce the frozen backbone exactly, one that never trained. The ELF, TAFAS, and ridge refiners draw no random number, so their three seeds
reproduce one forecast exactly on every cell under every reveal, and each of their cells is one run; the ORCA
and DSOF refiners differ across seeds everywhere. The refiners warm on the training and validation
blocks with the same 96-step closure, where the Kalman variants other than the warmed configuration use only the training block with
each label fed at once; the refiners therefore hold the validation block's labels before the test block, those
Kalman variants never see them, and no variant sees a test label before its closure. The refined test-block forecasts pass through the same divergence tracker and cumulative score as
every other variant, and a non-finite forecast is scored as it stands; none occurred in the 675 runs.

The rule was fixed before the campaign ran. It passes if the ORCA refiner's clean gain over the
frozen backbone is at least 0 on at least 10 of the 12 non-weather cells under the delayed reveal,
with three seeds on every cell and no diverged or non-finite seed. Under a pass, a published
black-box adapter recovers value under the causal reveal on these streams at this horizon and
context, so the loss of Section~\ref{sec:cutoff} is a property of the Kalman adapter family. Under
a fail, at released defaults, on these backbones, streams, horizon, and context, the strongest
cited black-box adapter also loses to the frozen backbone under the causal reveal, so the loss
extends beyond the Kalman adapter. The same count for the other four variants, the immediate-minus-delayed difference per variant, and each variant's per-cell delta against the Kalman variants are readings.

The rule passes. The ORCA refiner's median gain is at or above 0 on 12 of the 12 non-weather
cells under the delayed reveal, from 0.05 to 0.29, pooled 0.14 over the 15 cells, with no
diverged seed (Table~\ref{app:tab:baselines}). The gains are medians over three seeds; every seed of the ORCA refiner is at or above 0 on the 12 non-weather cells, the lowest at 0.05. The per-seed records are retained for release upon publication. The ELF, TAFAS, and ridge rows carry no seed spread because their three seeds reproduce one forecast.
By the reading fixed before the campaign, the pass scopes the loss of Section~\ref{sec:cutoff} to the Kalman adapter
family on these cells, at this horizon and context, under these reveals and defaults; the threats below name
what the evaluation does not separate from family.

The other four variants are readings. Under the delayed reveal the ELF port reaches at least 0 on 11
of the 12 non-weather cells, the TAFAS port on 8, the DSOF port on 8, and the ridge refiner on
12; no delayed-reveal seed of the 225 diverges or finishes non-finite, 135 of them distinct runs once each
deterministic variant's cell counts once. Under immediate reveal the
pooled gains are 0.21 for the ORCA refiner, 0.04 for the ELF port, 0.04 for the TAFAS port, 0.08
for the DSOF port, and 0.03 for the ridge refiner, so the immediate-minus-delayed difference of
the median non-weather gain runs from 0.00 to 0.05 across the
five variants against 0.78 for the Kalman adapter and 0.49 with its factor pinned. The contrast is
descriptive. For the three refiners that move the median by 0.03 to 0.04 under either reveal it cannot separate
the effect of immediate access from how little they adapt; the ORCA refiner moves it by 0.17 to 0.22 and gains 0.05 under immediate rather than delayed access. Against the Kalman adapter with its factor pinned at 0.999, the per-cell delta over the 12
non-weather cells under the delayed reveal runs from 0.17 to 0.58 for the ORCA refiner ($-0.06$
to 0.58 over the 15 cells) and from $-0.07$ to 0.54 across the other four. Under the composite coordinate reveal the Kalman variant consumes one label entry an origin after issue while a refiner sees nothing
until the horizon, so no delta is read there. On weather
the five variants reach $-6.88$ to 0.16 under the delayed reveal, the three lowest cells all the DSOF
port's ($-1.56$, $-0.85$, and $-6.88$ on TiRex, Chronos-2, and TimesFM~2.5) and the other four
variants within $-0.09$ to 0.16, against $-3.14$ for the learned Kalman adapter and 0.04 with its factor pinned. Read as a whole, the five variants say that the causal correction available on these streams at
context 96 is small, 0.02 to 0.04 pooled for four of the five. A plain ridge on the context and
base forecast clears the floor on 12 of 12 non-weather cells where the TAFAS and DSOF ports clear
it on 8. This campaign therefore bounds the Kalman adapter's loss and does not rank the published
methods.

Six limitations bound the interpretation. The context is 96 steps; ORCA reports its numbers at 520. The
target channel alone is forecast and adapted. The refiners read the 96-step context beside the cached forecast and warm on the validation block as well, where
the Kalman adapter updates on the stacked forecast alone and warms on the training block. A prespecified warm-up check uses both blocks and scores the Kalman adapter unchanged (the warmed rows of Table~\ref{app:tab:baselines}).
Under the delayed reveal it clears the floor on 0 of the 12 non-weather cells, as the reference variant does, moves a
cell by $-0.07$ to 0.04 with the median cell moving by less than 0.01, and moves the leaky gain by 0.03, so the
check separates warm-up from family and leaves inputs unseparated.

Every variant is scored on the train-block standard-deviation scale of its stream, whereas ORCA reports results using a train-split
mean-absolute scale, so its reported percentages are a different quantity; the released learning
rates, penalties, and combiner rate were chosen at that scale and run here at this one, a default value and a different behavior. ELF, TAFAS, and DSOF run as the released simplified ports
at their default settings. The ridge refiner reads the 96-step context beside the base forecast and
gains 0.13 on ETTm1 under the delayed reveal, where the recursive ridge on causal residual summaries of
Section~\ref{sec:cutoff} becomes unstable on that stream; the two fits differ in features, estimator, and penalty, and this
campaign does not separate them. \citet{dai2026orca} report up to 16 percent error reductions for the
residual refiner on the same backbones and streams under the delayed full-horizon reveal, with a
different horizon average, context, and normalization; here its per-stream medians run from
$-0.01$ on weather to 0.24 on ETTm1, a different profile from the paper's, and this campaign does
not explain the difference.

\begin{table}[t]
\centering
\caption{Released refiners and simplified method ports beside the Kalman adapter. Clean gain over the frozen
backbone, one minus the ratio of clean MSEs, median over seeds then over the three backbones per
stream and over the 15 cells (pooled), under immediate and delayed reveal; positive is the number
of the 15 cells with a positive median gain, within the number of the 12 non-weather cells at or above the floor of 0 fixed before the runs, and diverged the number of the 45 clean seeds that diverged
or finished non-finite, a diverged seed counting as $-\infty$. The learned and pinned Kalman rows repeat
Table~\ref{app:tab:fixed-forgetting} with the within count taken at this table's floor. The warmed rows are
the warm-up check, the learned adapter warm-started on the training and validation blocks. The
refiners see identical information under the composite coordinate and delayed reveals and reproduced their
delayed rows exactly, so they have no separate coordinate rows.}
\label{app:tab:baselines}
\scriptsize
\setlength{\tabcolsep}{3.5pt}
\begin{tabular}{llrrrrrrrrr}
\toprule
reveal & variant & ETTh2 & ETTm1 & weather & electricity & traffic & pooled & positive & within & diverged \\
\midrule
immediate & Kalman, learned & 0.63 & 0.79 & $-13.51$ & 0.25 & 0.20 & 0.25 & 12 & 12 & 0 \\
immediate & Kalman, warmed & 0.63 & 0.82 & $-14.24$ & 0.25 & 0.20 & 0.25 & 12 & 12 & 0 \\
immediate & Kalman, $\lambda$ 0.999 & 0.52 & 0.71 & $-0.19$ & 0.17 & 0.04 & 0.17 & 13 & 12 & 0 \\
immediate & ORCA & 0.08 & 0.31 & 0.00 & 0.22 & 0.21 & 0.21 & 14 & 12 & 0 \\
immediate & ELF port & 0.03 & 0.13 & 0.04 & 0.05 & 0.00 & 0.04 & 13 & 11 & 0 \\
immediate & TAFAS port & 0.01 & 0.09 & 0.04 & $-0.11$ & 0.02 & 0.04 & 11 & 9 & 0 \\
immediate & DSOF port & 0.01 & 0.08 & $-0.15$ & 0.20 & 0.11 & 0.08 & 11 & 11 & 0 \\
immediate & ridge & 0.03 & 0.17 & 0.10 & 0.03 & 0.00 & 0.03 & 15 & 12 & 0 \\
\midrule
delayed 96 & Kalman, learned & $-0.73$ & $-0.61$ & $-3.14$ & $-0.16$ & $-0.30$ & $-0.60$ & 0 & 0 & 0 \\
delayed 96 & Kalman, warmed & $-0.74$ & $-0.66$ & $-2.78$ & $-0.17$ & $-0.31$ & $-0.65$ & 0 & 0 & 0 \\
delayed 96 & Kalman, $\lambda$ 0.999 & $-0.44$ & $-0.19$ & 0.04 & $-0.04$ & $-0.04$ & $-0.05$ & 3 & 1 & 0 \\
delayed 96 & ORCA & 0.06 & 0.24 & $-0.01$ & 0.14 & 0.18 & 0.14 & 13 & 12 & 0 \\
delayed 96 & ELF port & 0.04 & 0.13 & 0.04 & 0.04 & 0.00 & 0.04 & 13 & 11 & 0 \\
delayed 96 & TAFAS port & 0.01 & 0.08 & 0.04 & $-0.11$ & 0.02 & 0.03 & 10 & 8 & 0 \\
delayed 96 & DSOF port & $-0.01$ & $-0.02$ & $-1.56$ & 0.21 & 0.17 & 0.02 & 8 & 8 & 0 \\
delayed 96 & ridge & 0.02 & 0.13 & 0.03 & 0.03 & 0.00 & 0.03 & 14 & 12 & 0 \\
\bottomrule
\end{tabular}
\end{table}

\FloatBarrier
\section{No-parameter-drift simulation, per-run}
\label{app:null}

The no-parameter-drift simulation (Section~\ref{sec:premise}) runs the full pipeline, the block
Kalman filter, its online hyperparameters, the calibration window, the martingale, and the
drift response, on a fixed-design regression stream with i.i.d.\ residuals and no parameter
drift. Each stream's per-step forecast pattern is a slowly varying sinusoid and its labels add
i.i.d.\ Gaussian noise. The design is time-varying, so the ordered forecast and label pairs are
neither identically distributed nor exchangeable. It uses 60 independent runs of
20{,}000 steps each, with a 500-step warm-up before the gate begins scoring.

Every one of the 60 runs fires at least once. Fires per run range from 33 to 42, with a mean
of 37.9. Ljung-Box at lag 10 rejects the conformal $p$-values at $p<0.01$ in 60 of 60 runs of
each learned-scale variant and in 0 of 60 of each held-scale variant; the warm-up 2{,}000 and
5{,}000 variants reject in 4 and 2 of 60, and the scale set to the truth in 55 while firing in 2.
The classical no-reset wealth process's first crossing coincides with the deployed gate's first fire in 60
of 60 runs of each learned pair, at steps 18 to 67 and 15 to 57 on the permuted stream. The first
fire after warm-up arrives between 18 and 67 steps into the run
(Figure~\ref{app:fig:nulldetail}).

The fixed-construction null runs the identical calibration window, randomized $p$-values, mixture, and
wealth-reset logic on exchangeable synthetic scores with no learner in the loop. It fires in at most
1 of 60 runs, under the same deployed initialization and the same horizon.

The element-level isolation of Section~\ref{sec:isolation} repeats this campaign under
twenty-two variants, all at the same 60 seeds and the same 20{,}000 steps, and all at the same
500-step warm-up except where the last three change it.
\begin{itemize}\setlength{\itemsep}{0pt}\setlength{\parskip}{0pt}
\item \textit{Full pipeline.} Every element active. Reproduces the saved simulation.
\item \textit{Filter frozen.} State and covariance held at the values they reached at the end of warm-up.
Learning of the forgetting factor and the process noise is inert with no time update and is
forced off. The observation-noise scale keeps learning.
\item \textit{Mean frozen.} The prediction and covariance updates run and only the mean
correction is skipped. All three hyperparameters learn and the full response is active.
\item \textit{Seven hyperparameter variants.} Online hyperparameter learning is restricted to
one subset of the forgetting factor, the process noise, and the observation-noise scale. The
filter and the full response are otherwise unchanged. The seven subsets are none, each of the
three alone, and each of the three pairs.
\item \textit{Response withheld.} The martingale computes and records every fire and its
response is never applied to the filter.
\item \textit{Huber-only response.} The full score, martingale, and spike classifier run. A fire
opens or restarts the 50-step recovery window, suppressing spike classification while the window
is open, but applies no covariance reset, process-noise boost, or forgetting-factor override.
\item \textit{Fixed construction.} The whole scorer is frozen and the test stream is randomly
permuted, so the frozen scorer faces an exchangeable sequence. This variant must reproduce the
fixed-construction null.
\item \textit{Full pipeline, permuted.} Every element active, on the randomly permuted test
stream the fixed-construction variant reads. This is the learning half of the matched pair.
\item \textit{No hyperparameters, permuted.} All three online hyperparameters held at their
warm-up values, the filter and the full response otherwise unchanged, on that same permuted
stream. This is the non-learning half of the pair, so the two differ only by hyperparameter
learning at one common input.
\item \textit{Four classical no-reset variants.} Exactly two settings change. The deployed sliding
calibration window is replaced by a calibration set that holds every score since monitoring
began, and the wealth reset on fire is turned off. The betting mixture, the level, and the spike
threshold are left as the campaign built them. The gate runs in shadow, so no response is applied and the filter follows its standard
learning path. The reported event is the first crossing of $\nicefrac{1}{\alpha}$, which
Ville's inequality bounds at probability $\alpha$. The four cells are the observation-noise
scale learned or held, each on the unpermuted and on the permuted stream.
\item \textit{Two longer warm-ups.} The deployed path with the warm-up before scoring extended
to 2{,}000 and to 5{,}000 steps. Each run is still 20{,}000 monitored steps. Both overrides
consume a different draw of each run's random sequence, so each monitors a different test stream
from the default one. The next setting changes only the initialization.
\item \textit{Scale started at truth.} The deployed path at the campaign's 500-step warm-up,
with $R$ started at its true value $0.0025$ before warm-up begins.
\end{itemize}
The thirteen variants of the earlier campaign reproduce exactly in the final campaign record.

\begin{table}[t]
\centering
\caption{Element-level isolation of the no-parameter-drift simulation by group, 60 runs of 20{,}000 steps each; the first row is the fixed-construction null with
no learner. Firing counts runs of 60 that fire at least once or, for the no-reset variants, ever cross
$\nicefrac{1}{\alpha}$; fires per run is a mean over the group's deployed-gate variants, and the
other columns are medians over 60 runs. At each step, $R$ is the mean of its 96 coordinates; the
two $R$ columns average that trajectory over the first and last tenth of each run (true coordinate
value $2.5\times10^{-3}$). Ljung-Box medians below machine precision print as $<10^{-15}$, and the learner-free
first row has no $R$. Variants per group in parentheses; the eight scale-held variants of the text are the
permuted-stream row and the seven of the next, and two variants in each of the next two groups also run on the
permuted stream. Table~\ref{tab:isolation} lists every variant.}
\label{tab:groups}
\scriptsize
\setlength{\tabcolsep}{3pt}
\begin{tabular}{lccccc}
\toprule
& & & & \multicolumn{2}{c}{$R$ ($10^{-3}$)} \\
\cmidrule(lr){5-6}
group & firing & per run & LB $p$ & first & last \\
\midrule
fixed-construction null & $\le$1/60 & -- & -- & -- & -- \\
frozen scorer on the permuted stream (1) & 1/60 & 0.02 & 0.47 & 7.9 & 7.9 \\
scale held fixed (7) & 1 to 3/60 & 0.02 & 0.41 to 0.47 & 7.9 & 7.9 \\
scale learned online (11) & 60/60 & 36.1 to 40.7 & $<10^{-15}$ & 2.6 to 2.7 & 2.3 to 2.5 \\
converged before monitoring (3) & 0 to 2/60 & 0.00 to 0.05 & $1.1{\times}10^{-5}$ to 0.16 & 2.2 to 2.5 & 2.2 to 2.5 \\
\bottomrule
\end{tabular}
\end{table}

\begin{table}[t]
\centering
\caption{Element-level isolation of the no-parameter-drift simulation. Twenty-two variants of the same campaign,
60 runs of 20{,}000 steps each, hold one element of the loop at its warm-up value, permute the
test stream, swap the deployed gate for the classical no-reset wealth process, or
change the warm-up and the initial noise scale. The firing column counts runs that fire at
least once out of 60, and for the no-reset group, whose wealth never resets, runs that ever cross
$\nicefrac{1}{\alpha}$. The next column gives the mean fires per run, and for that same group
the mean number of steps per run spent above $\nicefrac{1}{\alpha}$. The first-fire column is
the median over the runs that fired, blank where none did, and for the no-reset group it is
the first crossing. The Ljung-Box $p$ at lag 10 on the conformal $p$-values, the standardized
score $s_t$, $R$, and the peak log-wealth are medians over the
60 runs. The $s_t$ columns give each run's mean over the first and last tenth of the run. At each
step, $R$ is the mean of its 96 coordinates; the $R$ columns then average that trajectory over the
first and last tenth. The last column applies to the no-reset group only. Ljung-Box
medians below machine precision are reported as $<10^{-15}$. The warm-up is 500 steps except
where the last group changes it. The true $R$ on this stream is $0.0025$
and the correct value of $s_t$ under a well-specified filter is $m = 96$.
Table~\ref{tab:groups} groups these rows for Section~\ref{sec:isolation}.}
\label{tab:isolation}
\scriptsize
\setlength{\tabcolsep}{1.5pt}
\begin{tabular}{llrrrrccr}
\toprule
variant & what is held & firing & per run & first fire & LB $p$ & $s_t$ & $R$ & $\log M$ \\
\midrule
\multicolumn{9}{l}{\textit{observation-noise scale learned online}} \\
full pipeline & nothing & 60/60 & 37.93 & 46 & $<10^{-15}$ & $90.2 \to 96.1$ & $0.0027 \to 0.0025$ & \\
filter frozen & state, covariance & 60/60 & 38.63 & 45 & $<10^{-15}$ & $90.8 \to 96.1$ & $0.0026 \to 0.0023$ & \\
mean frozen & mean update & 60/60 & 40.73 & 45 & $<10^{-15}$ & $90.3 \to 96.1$ & $0.0026 \to 0.0025$ & \\
response withheld & response action & 60/60 & 37.87 & 46 & $<10^{-15}$ & $91.3 \to 96.1$ & $0.0027 \to 0.0025$ & \\
Huber-only & drift response & 60/60 & 37.87 & 46 & $<10^{-15}$ & $91.3 \to 96.1$ & $0.0027 \to 0.0025$ & \\
noise scale only & forgetting, process & 60/60 & 37.92 & 46 & $<10^{-15}$ & $90.1 \to 96.1$ & $0.0027 \to 0.0025$ & \\
forgetting, noise & process noise & 60/60 & 37.92 & 46 & $<10^{-15}$ & $90.1 \to 96.1$ & $0.0027 \to 0.0025$ & \\
process and noise & forgetting factor & 60/60 & 37.93 & 46 & $<10^{-15}$ & $90.2 \to 96.1$ & $0.0027 \to 0.0025$ & \\
\midrule
\multicolumn{9}{l}{\textit{observation-noise scale held fixed}} \\
no hyperparameters & all three & 1/60 & 0.02 & 79 & 0.46 & $30.3 \to 30.3$ & $0.0079 \to 0.0079$ & \\
forgetting only & process, scale & 1/60 & 0.02 & 79 & 0.46 & $30.3 \to 30.3$ & $0.0079 \to 0.0079$ & \\
process noise only & forgetting, scale & 1/60 & 0.02 & 79 & 0.41 & $30.5 \to 30.5$ & $0.0079 \to 0.0079$ & \\
forgetting, process & noise scale & 1/60 & 0.02 & 79 & 0.41 & $30.5 \to 30.5$ & $0.0079 \to 0.0079$ & \\
fixed construction & everything, permuted & 1/60 & 0.02 & 104 & 0.47 & $30.1 \to 30.1$ & $0.0079 \to 0.0079$ & \\
\midrule
\multicolumn{9}{l}{\textit{matched pair on the permuted stream}} \\
full pipeline & nothing & 60/60 & 36.10 & 41 & $<10^{-15}$ & $90.1 \to 96.1$ & $0.0027 \to 0.0025$ & \\
no hyperparameters & all three & 1/60 & 0.02 & 19 & 0.41 & $30.3 \to 30.3$ & $0.0079 \to 0.0079$ & \\
\midrule
\multicolumn{9}{l}{\textit{Classical no-reset wealth process, gate in shadow}} \\
scale learned & nothing & 60/60 & 19{,}953.8 & 46 & $<10^{-15}$ & $91.3 \to 96.1$ & $0.0027 \to 0.0025$ & 233 \\
scale held & noise scale & 1/60 & 1.20 & 79 & 0.47 & $30.5 \to 30.5$ & $0.0079 \to 0.0079$ & $-0.17$ \\
scale learned & nothing, permuted & 60/60 & 19{,}957.8 & 41 & $<10^{-15}$ & $91.3 \to 96.1$ & $0.0027 \to 0.0025$ & 237 \\
scale held & scale, permuted & 3/60 & 134.60 & 3{,}925 & 0.45 & $30.4 \to 30.5$ & $0.0079 \to 0.0079$ & 0.05 \\
\midrule
\multicolumn{9}{l}{\textit{convergence sensitivity checks, full deployed pipeline}} \\
warm-up 2{,}000 & nothing & 0/60 & 0.00 & & 0.16 & $96.1 \to 96.1$ & $0.0025 \to 0.0025$ & \\
warm-up 5{,}000 & nothing & 0/60 & 0.00 & & 0.11 & $96.1 \to 96.1$ & $0.0025 \to 0.0025$ & \\
scale set to truth & nothing & 2/60 & 0.05 & 17 & $1.1{\times}10^{-5}$ & $96.1 \to 96.1$ & $0.0022 \to 0.0022$ & \\
\bottomrule
\end{tabular}
\end{table}

The matched pair on the permuted stream reads at one common input. With everything learning the
campaign fires in 60 of 60 runs at 36.1 fires per run, its first fire arriving at steps 15, 41,
and 57 for the minimum, median, and maximum. With the forgetting factor, the process noise, and
the observation-noise scale all held fixed it fires in 1 of 60 runs, at 0.02 fires per run, a
median Kolmogorov-Smirnov $p$ of 0.999, a standardized score flat at $30.3$, and the noise scale
flat at $0.0079$. Among the no-reset variants, the learned member's peak log-wealth over a run falls between 206 and 256 with a median of 233, and each run spends 19{,}933 to
19{,}969 of its 20{,}000 steps above the threshold. The held member reaches a median peak log-wealth of $-0.17$ over a range of $-0.44$ to $4.64$. On the permuted stream the learning
member crosses at steps 15, 41, and 57 at a peak log-wealth of 206 to 254, median 237,
and the frozen member crosses at steps 19, 3{,}925, and 4{,}283, at a median peak log-wealth of $0.05$ over a range of $-0.44$ to $8.21$, a median Ljung-Box $p$ of 0.45, and a
median Kolmogorov-Smirnov $p$ of 0.29. After its first crossing the deployed gate's wealth
resets and it fires again, 37.9 times per run in total, and the classical no-reset wealth process stays above
$\nicefrac{1}{\alpha}$ for the rest of the run. The six deployed variants that hold the noise
scale fixed report a median Kolmogorov-Smirnov $p$ of 0.998 to 0.999.

\begin{figure}[t]
\centering
\includegraphics[width=\linewidth]{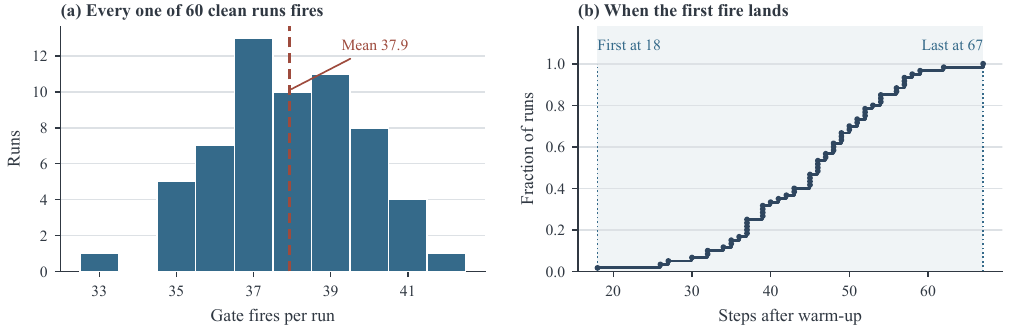}
\caption{Adaptive-null per-run detail over its 60 fixed-design-stream runs. (a) Fires per run,
33 to 42 with mean 37.9. (b) Empirical CDF of the first fire after warm-up, arriving 18 to
67 steps in. No single run drives the 60-of-60 result.}
\label{app:fig:nulldetail}
\end{figure}

\section{E-detector replay}
\label{app:edetector}

Sequential e-detectors offer change detection under an average-run-length
guarantee~\citep{shin2024edetectors}, so a natural question is whether swapping the martingale
for a detector carrying that guarantee repairs the deployed premise. We built
Shiryaev-Roberts~\citep{shiryaev1963quickest,roberts1966charts} and
CUSUM~\citep{page1954cusum} e-detectors from the deployed betting family and replayed them on
the saved $p$-value streams of the 405 saved clean gated runs. Those 405 streams cover six
variants under immediate reveal (the Huber-only variant, the full gate, the SGD surrogate, and the
three single-component variants) and three under delayed reveal (the Huber-only variant, the full
gate, and the SGD surrogate), each over five streams, three backbones, and three seeds.

The replay first reproduces the saved fire trajectory exactly in 405 of 405 runs, so the
detectors read the same evidence the deployed gate read. Both detectors then alarm in all 405
of those streams at every threshold $A \in \{10^{3}, 10^{4}, 10^{5}\}$. The main text's 135 immediate-reveal runs of the three martingale-carrying variants are 135 of 135 among them. Over
those same 135 runs at $A = 10^{4}$ the
Shiryaev-Roberts detector produces 193, 320, and 1{,}464 alarms per run at the minimum,
median, and maximum. For intuition only, the ratios $T/A$ range from 1.3 to 5.2 over these
stream lengths. They are heuristic reference scales: the average-run-length guarantee concerns
the first false alarm and does not bound alarm counts after restarts.

In the two synthetic null checks below, the observed alarm behavior remains within the stated
reference scales. On i.i.d.\ $p$-value streams the restarted
alarm counts stay at or below the corresponding $T/A$ reference at every threshold and stream
length. This is not a formal bound on restarted counts. On i.i.d.\ scores
passed through the deployed calibration window from an empty start, the Shiryaev-Roberts detector alarms
in 15 of 30 runs at $A = 10^{4}$ with at most 2 alarms per run, against the reference value
$T/A=2.0$. These e-detector instantiations therefore inherit the invalidity of the deployed
$p$-value stream. A remedy must either restore exchangeability for this conformal construction or
replace it with evidence valid under a dependence-aware null. E-processes that test exchangeability
directly~\citep{ramdas2022exchangeability} can measure whether a candidate remedy has done so.

A classical detector calibrated under dependence asks the same question from outside the
betting family. A two-sided Page CUSUM~\citep{page1954cusum} with drift 0.5 runs on each run's
standardized scores, its threshold set on the first 6{,}000 scored steps of the same run in two
ways, by a moving-block bootstrap~\citep{kunsch1989jackknife} of that reference at block lengths
1, 96, and 500, and by rescaling the increments with a Newey-West
estimate~\citep{newey1987simple} of the reference's long-run variance at
bandwidths 96, 250, and 500 against simulated i.i.d.\ paths. Both calibrations fail in finite samples in ways measured on i.i.d.\ data at this ratio of
reference to monitoring length. The bootstrap of a short reference is anti-conservative for a
maximum statistic (0.34 and 0.62 of 100 i.i.d.\ runs alarm at blocks 96 and 500), and the
Bartlett estimate runs low when the bandwidth is a large fraction of the reference (0.25 and 0.43
at bandwidths 250 and 500). Only block length 1, the i.i.d.\ calibration, and bandwidth 96 hold
within 0.15 there (0.04 and 0.11), and bandwidth 96 also stays at or below 0.01 on AR(1) and
96-term moving-sum synthetic streams. The reference segments carry a lag-1 autocorrelation of 0.88 to 0.98 and an integrated
autocorrelation time of 11 to 62 steps, and the long-run standard deviation at bandwidth 96 runs
at a median 6.1 times the reference's.

On the 180 saved streams the i.i.d.\ calibration alarms
in 180 of 180 at a median of 167 steps into monitoring, and the dependence-calibrated detector in
144 of 180 at a median of 5{,}076 steps, 36 of 36 on electricity, ETTm1, traffic, and weather and
0 of 36 on ETTh2, identically under both protocols and both scale variants. On the SGD adapter's 90
streams it alarms in 78 of 90, 6 of 18 on ETTh2. A detector built to absorb dependence of the
reuse length still alarms on four of five streams, and the stream it holds is one of the two on
which the adapter's horizon-1 term never crossed (Table~\ref{app:tab:horizons}); the settings
whose i.i.d.\ calibration fails are retained for release upon publication as a sensitivity.

A closely related published monitor is WATCH~\citep{prinster2025watch}, a weighted conformal test
martingale that freezes its calibration set once a secondary martingale on the covariates alone
detects a shift and thereafter weights every calibration $p$-value by an online density ratio
between calibration and target covariates. Replaying it needs the covariate each score was computed under, so twelve summary features of every origin's context are rebuilt from the data on the split the streams were saved on.
Four monitors then run on every saved clean run at $\alpha=0.05$ with the composite jumper and no
wealth reset. They are the plain conformal test martingale on the scores, the covariate
martingale, WATCH with adaptation at the covariate process's threshold of 20, and the same
weighted procedure with adaptation forced at a 250-step burn-in (Table~\ref{app:tab:watch}). On the 180 saved shadow runs, the learned-scale and
frozen-scale shadow variants under both protocols, every monitor crosses $\nicefrac{1}{\alpha}$ in
180 of 180, whereas Ville's inequality bounds the expected number of crossing runs under the null
by 9. The plain wealth process crosses at a median step of
25; WATCH crosses at the same median step with a higher peak, 7{,}339 against 3{,}096 in log-wealth.
The covariate wealth process crosses in every run at a median step of 44, so adaptation triggers in
180 of 180 runs between steps 20 and 78. WATCH had already crossed before reweighting could act in
141 runs; with adaptation forced at step 250, it had crossed in all 180. After adaptation the
density-ratio classifier separates the calibration and target covariates almost perfectly on three
streams, the weights' effective sample size falling to a median of 1.0 to 1.2 calibration points on
electricity, ETTh2, and traffic against 57 on ETTm1 and 16 on weather, and the weighted martingale
keeps growing. On the SGD adapter's 78 aligned shadow streams, every monitor crosses in 78 of 78.
WATCH crosses at a median step of 40 and before adaptation in 37 of 78 runs. The 12 traffic streams
of that bundle's TiRex and Chronos-2 runs, saved with 12{,}912 to 13{,}009 test steps against the
covariates' 13{,}014, fail the alignment check and are excluded from every count. A marginal
covariate change with the conditional law fixed, which weighted conformal monitoring is built to
absorb, is not what breaks these streams. The covariates do shift at once and almost completely,
but reweighting the calibration set leaves the score wealth process crossing everywhere. The covariate representation, the adaptation trigger, and the
burn-in are this paper's choices, so the replay tests whether reweighting this representation
rescues these streams and leaves WATCH under every configuration untested.

\begin{table}[t]
\centering
\caption{WATCH replayed on the saved clean score streams at $\alpha=0.05$ with the composite
jumper and no wealth reset, 180 deployed shadow runs (two shadow variants, both protocols) and the 78
aligned SGD shadow runs (90 less the 12 traffic streams that fail the alignment check). The covariate martingale is WATCH's secondary monitor on nearest-neighbour
distances of the covariates alone. Crossings count runs in which the wealth ever crosses
$\nicefrac{1}{\alpha}$, first is the median first-crossing step, and peak the median maximum
log-wealth. Under the null, Ville's inequality bounds the expected numbers of crossing runs by
9 and 3.9, respectively.}
\label{app:tab:watch}
\footnotesize
\setlength{\tabcolsep}{4pt}
\begin{tabular}{lcccccc}
\toprule
& \multicolumn{3}{c}{deployed shadow streams} & \multicolumn{3}{c}{SGD shadow streams} \\
\cmidrule(lr){2-4}\cmidrule(lr){5-7}
monitor & crossings & first & peak & crossings & first & peak \\
\midrule
conformal test martingale & 180/180 & 25 & 3{,}096 & 78/78 & 37 & 3{,}239 \\
covariate martingale & 180/180 & 44 & 2{,}135 & 78/78 & 42 & 2{,}135 \\
WATCH & 180/180 & 25 & 7{,}339 & 78/78 & 40 & 7{,}057 \\
WATCH, adaptation forced at 250 & 180/180 & 25 & 6{,}355 & 78/78 & 37 & 6{,}595 \\
\bottomrule
\end{tabular}
\end{table}

\section{Firing and response under a sustained regime change}
\label{app:power}

Reducing detector exposure and changing the response alter clean-stream firing in different ways,
so this section tests whether the resulting variants remain responsive to a change the gate should
find. A regime-shift injector moves every label
from a deterministic onset at 60 percent of the test block to its end by a fixed sign and magnitude
of 1.0 or 3.0 in robust-scale units. This follows the level-shift convention and defines a separate
configuration for each magnitude and protocol. The every-origin variants are the ungated filter, full
gate, capped policy, fixed cutoff, shadow variant, and full gate without its soft reset. Together with
the detector-stride-96 pair, they run over five streams, three seeds, and three backbones: 45 runs
per variant, magnitude, and protocol, with a clean cell beside each as background. Of the 1{,}440
regime records, 1{,}428 enter; the 12
excluded are the full gate's immediate-reveal weather runs, which diverge at the transient of
step 18{,}537 before the onset at 23{,}627. We fixed two summaries before running the study
(Table~\ref{app:tab:power}). First, per variant and
magnitude, the fraction of runs whose first post-onset fire lands within 96 steps of the onset
against the fraction of background gaps within 96 steps, the background being the run's own
pre-onset inter-fire gaps pooled with its clean twin's; a variant that fires after the onset no more
often than its background does shows no onset-specific response under this heuristic. Second, the post-onset fed-label MSE against
the ungated filter on the same cell, with the pre-onset ratio as the check that the variants agreed
before the change, which they do to within 0.88 to 1.16. The firing and background fractions use
different sampling units, so this comparison is descriptive and is not a power test.

Under this heuristic analysis, the only deployed fraction above its background reference
after rounding is 0.54 against 0.53 for the immediate-reveal full gate at magnitude 1.0. At the
larger magnitude, its first fire lands within 96 steps in 0.51 of runs against a background of 0.53.
The capped policy and the gate without its soft reset fire within 96 steps in 0.51 to 0.62 of runs
against backgrounds of 0.73 to 0.75; their clean-stream gaps are 6 to 8 steps
at the median. Under delayed reveal, every-origin gates fire within 96 steps in 0.18 to 0.20 of
runs against backgrounds of 0.60 to 0.76, an excess of $-0.40$ to $-0.56$. The two
detector-stride-96 variants never fire within 96 steps except for the capped variant in one of 45
immediate-reveal runs at magnitude 3.0. Their median latencies range from 640 to 8{,}437 steps,
against a background of 0.

The cost tells the variants apart. Under immediate reveal the three every-origin gates adapt as the ungated filter does, with
post-onset ratios of 0.91 to 0.98, because their fires, background or not, reopen the response.
The variants that keep the per-step spike check and have no fire in reach classify the shifted
residuals as spikes and hold the update back; the fixed cutoff pays 1.92 and 8.92 times the
ungated filter's post-onset error at magnitudes 1.0 and 3.0, and the detector-stride-96 pair 1.88 and 5.51
to 5.53.
Under delayed reveal every gated variant pays more than the ungated filter, the every-origin gates
1.10 to 1.38, the cutoff 3.32 at magnitude 3.0, and the detector-stride-96 pair 2.59. On the shadow stream the classical no-reset wealth process has crossed $\nicefrac{1}{\alpha}$ at some step
before the onset in 180 of 180 runs, so a first crossing can report nothing about the change. The deployed gate's first post-onset fire lands at a median of 1 and 20 steps on ETTh2, 192 and 288 on weather, 518
and 615 on ETTm1, 1{,}008 and 1{,}104 on electricity, and 2{,}393 and 2{,}489 on traffic, under
immediate and delayed reveal. The offline detector-stride-96 wealth process crosses after the onset in none of
its 1{,}440 offset replays. The two fractions compared are of different kinds, a run-level first fire against gap-level
background, so the table only describes each variant against its own firing pattern. The
mitigations that removed the clean-stream fires also left the gate largely unresponsive within
one horizon of this change. The spike check provides protection against isolated contamination
by suppressing updates, but the same mechanism also prevents prompt adaptation to a persistent
level shift.

\begin{table}[t]
\centering
\caption{Firing and response under a sustained regime change at magnitudes 1.0 and 3.0, 45 runs per
variant, magnitude, and protocol (the full gate's 6 immediate-reveal weather runs excluded, diverging
before the onset). Fire $\leq 96$ is the fraction of runs whose first post-onset fire lands within
96 steps, background is the fraction of pre-onset and clean-twin inter-fire gaps within 96 steps,
latency is the median first post-onset fire in steps, and cost is the median post-onset fed-label
MSE against the ungated filter on the same cell. Fire $\leq 96$ and background use different sampling units;
their comparison is descriptive and is not a power test. The fixed cutoff carries no martingale.}
\label{app:tab:power}
\scriptsize
\setlength{\tabcolsep}{3pt}
\begin{tabular}{llrcccccccc}
\toprule
& & & \multicolumn{4}{c}{magnitude 1.0} & \multicolumn{4}{c}{magnitude 3.0} \\
\cmidrule(lr){4-7}\cmidrule(lr){8-11}
protocol & variant & $n$ & fire$\leq$96 & background & latency & cost & fire$\leq$96 & background & latency & cost \\
\midrule
immediate & full gate & 39 & 0.54 & 0.53 & 72 & 0.98 & 0.51 & 0.53 & 93 & 0.98 \\
immediate & capped policy & 45 & 0.58 & 0.75 & 45 & 0.94 & 0.62 & 0.75 & 30 & 0.97 \\
immediate & full gate, no soft reset & 45 & 0.51 & 0.73 & 59 & 0.93 & 0.62 & 0.73 & 32 & 0.91 \\
immediate & detector stride 96 & 45 & 0.00 & 0.00 & 8{,}341 & 1.88 & 0.00 & 0.00 & 1{,}562 & 5.51 \\
immediate & detector stride 96, capped & 45 & 0.00 & 0.00 & 640 & 1.88 & 0.02 & 0.00 & 1{,}354 & 5.53 \\
immediate & fixed cutoff & 45 & --- & --- & --- & 1.92 & --- & --- & --- & 8.92 \\
delayed & full gate & 45 & 0.20 & 0.60 & 141 & 1.15 & 0.20 & 0.60 & 137 & 1.38 \\
delayed & capped policy & 45 & 0.20 & 0.76 & 141 & 1.10 & 0.20 & 0.76 & 127 & 1.22 \\
delayed & full gate, no soft reset & 45 & 0.18 & 0.74 & 144 & 1.11 & 0.18 & 0.74 & 128 & 1.30 \\
delayed & detector stride 96 & 45 & 0.00 & 0.00 & 8{,}149 & 1.13 & 0.00 & 0.00 & 1{,}546 & 2.59 \\
delayed & detector stride 96, capped & 45 & 0.00 & 0.00 & 8{,}437 & 1.13 & 0.00 & 0.00 & 2{,}010 & 2.59 \\
delayed & fixed cutoff & 45 & --- & --- & --- & 1.13 & --- & --- & --- & 3.32 \\
\bottomrule
\end{tabular}
\end{table}

Two further replays ask whether alternative sequential constructions respond to this change,
first on the learner's own stream and then on a stream the learner does not touch. Both are
compared with the clean run of the same cell read at the same pseudo-onset, the first revealed
post-onset label, so that the event and twin fractions have the same sampling unit
(Tables~\ref{app:tab:power-learner} and~\ref{app:tab:power-frozen}).

For each run, the one-offset construction draws one of the 96 offsets from the run seed before
reading the scores, then applies the growing-calibration, no-reset wealth process to that thinned
stream. The union construction runs all 96 offset paths and crosses when their maximum wealth
reaches $96/\alpha$; a union bound gives anytime validity if every offset stream satisfies the
classical null. The offset average compares the arithmetic mean of the 96 wealths with
$1/\alpha$. It is a fixed-time e-value under the corresponding nulls, but we do not establish
anytime validity because the paths are adapted to different filtrations. The restarting
e-detector of \citet{shin2024edetectors} uses the preselected offset and alarms at
$A\in\{20,400,4000\}$. When its null conditions hold, the reference guarantee is an
average-run-length bound for the first alarm, not Ville's crossing-probability bound or a bound
on the number of alarms after restarting. A post-onset event is the first crossing or alarm at or
after the pseudo-onset and within the stated window. A no-reset path already above its threshold
at the onset is not counted as a new event.

On the shadow adapter's saved scores, no theoretical guarantee is established for any
construction because the required null is not established. Descriptively, the three stride-96
wealth constructions do not cross within 960 steps in any of the 45 runs at either magnitude
under either protocol. The e-detector at threshold 20 alarms in 1 of 45 runs at magnitude 3.0
under each protocol and never in the other settings. The every-origin
classical no-reset wealth process crosses within 96 steps in 0.20 to 0.22 of runs against a twin of 0.09 to
0.13, and within 960 steps in 0.20 to 0.22 against 0.20 to 0.22 (Table~\ref{app:tab:power-learner}),
with its wealth above the threshold at the onset step in 8 and 15 of 45 runs, immediate and
delayed; the 180 of 180 above counts runs that had crossed at any earlier step, since the no-reset
wealth falls back below the threshold between crossings.

The reason is in the scores. At magnitude 3.0 the learner's score over the first 96 post-onset steps runs 6.9 to 42 times its pre-onset mean
on the four non-weather streams and over the next 96 steps 0.11 to 0.33 of it; at magnitude 1.0,
1.6 to 5.5 times and then 0.20 to 1.18. The adapter absorbs the shift within one horizon and
overshoots, so a monitor that sees one score every 96 steps sees at most one elevated score. The
frozen backbone's residual against the same fed labels, built from the cache as in
Appendix~\ref{app:valid}'s no-learner analysis and scored under the training-block scale, stays 14 to
17 times its pre-onset mean at magnitude 1.0 in every 96-step block to 2{,}000 steps and about 130 times
at 3.0. On that stream the classical no-reset wealth processes carry no established bound, one offset crossing on 26 of 45 clean
runs and the offset average on 45, the dependence Appendix~\ref{app:valid} places in the
forecast-error process. The e-detector at 400 alarms 0.76 times per clean run, small relative to
the $T/A$ reference but not covered by a formal restarted-count bound. It alarms after the
magnitude-1.0 change within 480 steps in 27 of 45 runs
against 0 of 45, all 9 runs on ETTh2, ETTm1, and traffic and none on electricity and weather
within 480 steps, and the magnitude-3.0 change in 44 of 45, at median latencies of 405 and 255
steps. At 4{,}000 it alarms 0.29 times per clean run, above the $T/A$ reference, and alarms within
480 steps in 19 and 35 of
45. The
96-parallel e-detector alarms within 15 steps but 260 to 280 times per clean run and is noise.

\begin{table}[t]
\centering
\caption{Alternative sequential constructions replayed on the shadow adapter's saved scores
under the sustained change. Each entry is the fraction of 45 runs whose first post-onset crossing
or alarm occurs within 960 steps of the first revealed post-onset label, followed by the fraction
for the clean twin at the same pseudo-onset (event/twin), by protocol and magnitude. No guarantee
is established on these adaptive score streams. A classical no-reset path already above its
threshold at the onset is not counted as a new event.}
\label{app:tab:power-learner}
\scriptsize
\setlength{\tabcolsep}{4pt}
\begin{tabular}{lcccc}
\toprule
& \multicolumn{2}{c}{immediate} & \multicolumn{2}{c}{delayed} \\
\cmidrule(lr){2-3}\cmidrule(lr){4-5}
construction & 1.0 & 3.0 & 1.0 & 3.0 \\
\midrule
every origin & 0.22/0.22 & 0.22/0.22 & 0.20/0.20 & 0.22/0.20 \\
one offset & 0.00/0.00 & 0.00/0.00 & 0.00/0.00 & 0.00/0.00 \\
union over 96 offsets & 0.00/0.00 & 0.00/0.00 & 0.00/0.00 & 0.00/0.00 \\
offset average & 0.00/0.00 & 0.00/0.00 & 0.00/0.00 & 0.00/0.00 \\
e-detector, $A=20$ & 0.00/0.00 & 0.02/0.00 & 0.00/0.00 & 0.02/0.00 \\
e-detector, $A=400$ & 0.00/0.00 & 0.00/0.00 & 0.00/0.00 & 0.00/0.00 \\
e-detector, $A=4000$ & 0.00/0.00 & 0.00/0.00 & 0.00/0.00 & 0.00/0.00 \\
\bottomrule
\end{tabular}
\end{table}

\begin{table}[t]
\centering
\caption{The same constructions on the frozen backbone's residual against the regime-shifted fed
labels, 45 runs per magnitude (five streams, three backbones, three seeds). The table reports
post-onset events within 480 and 960 steps against the clean stream at the same pseudo-onset
(event/twin), median latency in steps, and the clean background as crossings of 45 runs or alarms
per run.}
\label{app:tab:power-frozen}
\scriptsize
\setlength{\tabcolsep}{3pt}
\begin{tabular}{lccccccc}
\toprule
& \multicolumn{3}{c}{magnitude 1.0} & \multicolumn{3}{c}{magnitude 3.0} & \\
\cmidrule(lr){2-4}\cmidrule(lr){5-7}
construction & 480 & 960 & latency & 480 & 960 & latency & clean \\
\midrule
every origin & 0/0 & 9/0 & 618 & 5/0 & 9/0 & 460 & 45/45 cross \\
one offset & 3/1 & 19/1 & 1{,}044 & 8/1 & 29/1 & 639 & 26/45 cross \\
union over 96 offsets & 4/0 & 23/0 & 742 & 13/0 & 34/0 & 578 & 41/45 cross \\
offset average & 16/2 & 28/2 & 509 & 24/2 & 36/2 & 392 & 45/45 cross \\
e-detector, $A=20$ & 35/5 & 39/7 & 171 & 45/5 & 45/7 & 152 & 2.62 alarms \\
e-detector, $A=400$ & 27/0 & 28/0 & 405 & 44/0 & 45/0 & 255 & 0.76 alarms \\
e-detector, $A=4000$ & 19/0 & 28/0 & 555 & 35/0 & 45/0 & 345 & 0.29 alarms \\
\bottomrule
\end{tabular}
\end{table}

The monitor of Table~\ref{app:tab:power-frozen} runs live in two variants (Table~\ref{app:tab:twostream}).
Both read the frozen backbone's residual against the fed labels at every 96th origin from a per-seed
random offset, at the training-block scale, with a growing calibration set and an e-detector
restarting at wealth one and alarming at 400; an alarm scored at origin $s$ is actionable at $s$ plus
the feedback delay, and from that step the adapter's state, covariance, and hyperparameter state are held for
480 steps. The hold variant emits the held adapter's forecast and the handover variant the frozen backbone's,
with the guard evaluated on the emitted forecast. On clean streams under immediate reveal the variants
alarm 0.667, 0.600, and 1.000 times per run on TiRex, Chronos-2, and TimesFM~2.5, equal to the
replay's 0.667, 0.600, and 1.000 on each backbone, the identity check of the live monitor;
under delayed reveal an alarm becomes actionable 96 steps after its origin and alarms past the end of
a run are not counted, so Chronos-2 reads 0.467. The hold covers 0.4 percent of clean steps or fewer,
and both variants reproduce the ungated filter on clean, spike, and packet-freeze cells to within 0.003. Under
the regime shift at magnitude 1.0 the monitor alarms after the onset in 43 of 45 runs under both
reveals against 8 of 45 on the clean twin at the same pseudo-onset, at median latencies of 405 and
501 steps. The responses split. Holding the adapter costs 1.14 and 1.10 times the ungated error under the 5 percent level shift
and 1.42 and 1.03 times its post-onset error under the true change, immediate and delayed reveal
in that order. Handing over to the frozen backbone cuts the level-shift cost to 0.667 and 0.704
and raises the regime cost to 6.66 and 1.99. The alarm cannot tell the contamination from the
change, so the two responses trade one against the other.

\begin{table}[t]
\centering
\caption{Live two-stream variants, 15 cells per row. The ungated adapter's gain over the frozen backbone
(cells worse than the frozen backbone in parentheses), the MSE of the hold variant and of the handover variant
over the ungated filter's on the same cells, and the fraction of steps held, medians over cells; on the
regime rows the ratios are post-onset fed-label costs.}
\label{app:tab:twostream}
\footnotesize
\begin{tabular}{llrrrr}
\toprule
reveal & cell & ungated gain & hold & handover & held \\
\midrule
immediate & clean & 0.25 (3) & 1.000 & 1.000 & 0.004 \\
immediate & spike, 5 percent & $-0.55$ (9) & 1.000 & 1.000 & 0.000 \\
immediate & freeze, 5 percent & 0.33 (6) & 1.000 & 0.997 & 0.024 \\
immediate & level shift, 5 percent & $-2.28$ (15) & 1.138 & 0.667 & 0.037 \\
immediate & regime shift, post-onset & $-1.93$ (15) & 1.42 & 6.66 & 0.148 \\
\addlinespace
delayed 96 & clean & $-0.60$ (15) & 1.000 & 1.000 & 0.000 \\
delayed 96 & spike, 5 percent & $-3.51$ (15) & 1.000 & 1.000 & 0.000 \\
delayed 96 & freeze, 5 percent & $-0.90$ (15) & 1.001 & 1.000 & 0.024 \\
delayed 96 & level shift, 5 percent & $-3.50$ (15) & 1.101 & 0.704 & 0.037 \\
delayed 96 & regime shift, post-onset & $-3.25$ (15) & 1.03 & 1.99 & 0.148 \\
\bottomrule
\end{tabular}
\end{table}

\section{Synthetic drift latency}
\label{app:latency}

Table~\ref{app:tab:latency} reports the synthetic-drift latency study over three seeds, with
declared break points and a per-break bounded search. Mean adaptation latency is the number
of steps after each declared break until a trailing-error recovery rule is met, with
never-recovered breaks capped at the bounded search horizon. The full gate is slower than the
ungated filter on three of three seeds at the task the drift response exists to serve, and
its test-block MSE is roughly four times worse. The detector-stride-96 gate of
Section~\ref{sec:cutoff}, whose martingale sees every 96th score, is slower still, at mean
latency 162 to 198 against the full gate's 128 to 155, with test-block MSE 0.73 to 1.06.
It still declares drift, later, and the response still costs.

\begin{table}[ht]
\centering
\caption{Mean latency in steps and test-block MSE on the synthetic drift benchmark of
Appendix~\ref{app:latency}, one fixed-design stream with declared break points, three seeds per
variant, latency to a trailing-error recovery rule with never-recovered breaks capped at the
bounded search horizon.}
\label{app:tab:latency}
\footnotesize
\begin{tabular}{lrrrrrr}
\toprule
 & \multicolumn{3}{c}{mean latency} & \multicolumn{3}{c}{final MSE} \\
\cmidrule(lr){2-4} \cmidrule(lr){5-7}
variant & seed 0 & seed 1 & seed 2 & seed 0 & seed 1 & seed 2 \\
\midrule
ungated filter & 126.75 & 128.0 & 95.25 & 0.0915 & 0.1008 & 0.0983 \\
full gate & 154.75 & 137.5 & 128.0 & 0.3418 & 0.4471 & 0.3653 \\
detector-stride-96 gate & 182.5 & 198 & 161.5 & 0.7345 & 1.0615 & 0.9102 \\
ungated SGD & 226.0 & 250.0 & 216.75 & 0.3270 & 0.3329 & 0.3302 \\
\bottomrule
\end{tabular}
\end{table}
 
\end{document}